%% file: main_arXiv.tex
\documentclass{article}

\usepackage{fullpage}
\usepackage[round]{natbib}
\usepackage{algorithm}
\usepackage[noend]{algorithmic}
\usepackage{amsmath,amsthm,amsfonts,amssymb}
\usepackage{amsmath}
\usepackage{color}
\usepackage{mathrsfs}
\usepackage{enumitem}
\usepackage{bm}
\usepackage{multirow}
\usepackage{booktabs}
\usepackage{makecell}
\usepackage{graphicx}
\usepackage{subfigure}
\usepackage{caption}
\usepackage{thmtools}
\usepackage{thm-restate}
\usepackage{setspace}
\usepackage{makecell}
\definecolor{light-gray}{gray}{0.85}
\usepackage{colortbl}
\usepackage{xcolor}
\usepackage{hhline}
\usepackage{xspace} 
\usepackage{mathtools} 
\usepackage{cancel}
\usepackage{amsfonts}

\input{notations_arxiv}
\input{notations_shared}

\newtheorem{remark}{Remark}
\newtheorem{example}[theorem]{Example}

\usepackage{times}

\def\shownotes{1}  
\ifnum\shownotes=1
\newcommand{\authnote}[2]{{\scriptsize $\ll$\textsf{#1 notes: #2}$\gg$}}
\else
\newcommand{\authnote}[2]{}
\fi
\newcommand{\yub}[1]{{\color{red}\authnote{Yu}{#1}}}
\newcommand{\qinghua}[1]{{\color{red}\authnote{QL}{#1}}}

\newcommand{\yw}[1]{{\color{red}\authnote{YW}{#1}}}

 \title{\fontsize{15pt}{15pt}\textbf{Breaking the Curse of Multiagency: Provably Efficient Decentralized Multi-Agent RL with Function Approximation}}

\author{
Yuanhao Wang\footnotemark[2]\hspace{.35em}\footnotemark[1]
\and
Qinghua Liu\footnotemark[2]\hspace{.35em}\footnotemark[1]
\and
Yu Bai\footnotemark[3]\hspace{.35em}\footnotemark[4]
\and
Chi Jin\footnotemark[2]\hspace{.35em}\footnotemark[4]
}

\date{\today}

\begin{document}

\maketitle

\input{Sections/abstract.tex}
 \input{Sections/intro.tex}

 \input{Sections/related_work.tex}

 \input{Sections/prelim.tex}
\input{Sections/abstract_alg.tex}

\input{Sections/linear.tex}
 \input{Sections/tabular.tex}

\input{Sections/pi_cce.tex}

\input{Sections/conclusion.tex}

\bibliographystyle{plainnat}
\bibliography{ref}


\newpage
\appendix


\input{Sections/tools.tex}

\input{Sections/v_type.tex}

 \input{Sections/proof_abstract_alg.tex}

\input{Sections/linear_properties.tex}

\input{Sections/proof_linearVLPR.tex}

\input{Sections/lems_linearVLPR.tex}

\input{Sections/proof_tabularVLPR.tex}

 \input{Sections/discussion_pi_cce.tex}

 \input{Sections/proof_pi_cce.tex}

\end{document}

%% file: notations_arxiv.tex
\usepackage{parskip}
\usepackage{enumitem}
\usepackage{amsmath,amssymb}
\usepackage{bbm}
\usepackage{mathtools}

\usepackage{color}
\usepackage{amsmath}
\usepackage{amsfonts}
\allowdisplaybreaks

\usepackage{algorithm}

 \usepackage[noend]{algorithmic}

\usepackage{url}
\usepackage[breaklinks=true]{hyperref}
\hypersetup{
  colorlinks,
  citecolor=blue!70!black,
  linkcolor=blue!70!black,
  urlcolor=blue!70!black
}

\renewcommand{\qed}{\hfill$\blacksquare$}

\newcommand{\piout}{\pi^{\rm out}}
\renewcommand{\P}{\mathbb{P}}
\newcommand{\E}{\mathbb{E}}

\newcommand{\R}{\mathbb{R}}

\newcommand{\N}{\mathbb{N}}
\renewcommand{\hat}{\widehat}
\renewcommand{\tilde}{\widetilde}
\newcommand{\eps}{\varepsilon}
\renewcommand{\epsilon}{\varepsilon}
\renewcommand{\bar}{\overline}

\newcommand{\set}[1]{{\left\{ #1 \right\}}}
\newcommand{\sets}[1]{{\{ #1 \}}}
\newcommand{\abs}[1]{{\left| #1 \right|}}

\newcommand{\paren}[1]{{\left( #1 \right)}}

\newcommand{\brac}[1]{{\left[ #1 \right]}}
\newcommand{\bracs}[1]{{[ #1 ]}}

\newcommand{\cD}{\mathcal{D}}
\newcommand{\cO}{\mathcal{O}}
\newcommand{\cT}{\mathcal{T}}
\newcommand{\cS}{\mathcal{S}}
\newcommand{\cA}{\mathcal{A}}
\newcommand{\cV}{\mathcal{V}}
\newcommand{\cF}{\mathcal{F}}
\newcommand{\cG}{\mathcal{G}}
\newcommand{\cB}{\mathcal{B}}
\newcommand{\cL}{\mathcal{L}}
\newcommand{\cI}{\mathcal{I}}
\newcommand{\cX}{\mathcal{X}}

\newcommand{\up}[1]{\overline{ #1}}
\newcommand{\low}[1]{\underline{ #1}}

\newcommand{\poly}{{\mathrm{poly}}}

\newcommand{\trans}{^\top}
\newcommand{\trace}{{\rm Tr}}

\newcommand{\defeq}{\mathrel{\mathop:}=}
\newcommand{\eqdef}{=\mathrel{\mathop:}}

\newtheorem{theorem}{Theorem}
\newtheorem{condition}[theorem]{Condition}
\newtheorem{corollary}[theorem]{Corollary}
\newtheorem{lemma}[theorem]{Lemma}
\newtheorem{proposition}[theorem]{Proposition}
\newtheorem{definition}[theorem]{Definition}
\newtheorem{assumption}[theorem]{Assumption}

\newcounter{parentnumber}

\newcommand{\Vb}{\up{V}}
\renewcommand{\a}{\mathbf{a}}

\newcommand{\fu}{\up{f}}
\newcommand{\fl}{\low{f}}
\renewcommand{\th}{^{\rm th}}

\newcommand{\norm}[1]{\left\|{#1}\right\|} 
\newcommand{\ltwo}[1]{\norm{#1}_2} 
\newcommand{\lfro}[1]{\left\|{#1}\right\|_{\sf Fr}} 
\newcommand{\norms}[1]{\|{#1}\|} 
\newcommand{\ltwos}[1]{\norms{#1}_2} 

\newcommand{\CCEGap}{{\rm CCEGap}}
\newcommand{\CCEReg}{{\rm CCEReg}}

\newcommand{\lin}{{\rm lin}}

\newcommand{\Mar}{{\sf Mar}}

\DeclareMathOperator*{\argmax}{arg\,max}
\newcommand{\Btheta}{{B_\theta}}

\newcommand{\mc}[1]{\mathcal{#1}}

\newcommand{\wt}{\widetilde}
\newcommand{\wb}{\overline}
\newcommand{\what}{\widehat}

\newcommand{\tO}{\wt{\mc{O}}}

\newcommand{\indics}[1]{1\{#1\}} 
\newcommand{\<}{\left\langle}
\renewcommand{\>}{\right\rangle}

\newcommand{\Reg}{{\rm Reg}}

\newcommand{\Unif}{{\rm Unif}}
\newcommand{\regress}{{\rm reg}}

\newcommand{\setto}{\leftarrow}

\newcommand{\oV}{\wb{V}}
\newcommand{\oQ}{\wb{Q}}

\newcommand{\DD}{\mathbb{D}}
\newcommand{\D}{\mathbb{D}}

\newcommand{\betareg}{\beta_{{\rm regress},h}}
\newcommand{\thetas}{\theta^{\star}}

\newcommand{\cC}{\mathcal{C}}
\newcommand{\dE}{d_{\rm E}}
\newcommand{\dBE}{d^{\rm BE}}
\newcommand{\dBEi}{\dBE_{\pi_{-i}}(\cF_i, \Pi_i, \epsilon)}

\newcommand{\FTU}{{\sf APE}}

\newcommand{\APE}{{\sf APE}}
\newcommand{\OPMD}{{\sf DOPMD}}
\newcommand{\DOPMD}{{\sf DOPMD}}

\newcommand{\cceaprx}{\textsc{CCE-approx}}
\newcommand{\vapprox}{\textsc{V-approx}}
\newcommand{\noreg}{\textsc{No-Regret-Alg}}
\newcommand{\optreg}{\textsc{Optimistic-Regress}}
\newcommand{\Piexp}{\Gamma_{\rm explore}}
\newcommand{\Piexpup}{\bar{\Gamma}}
\newcommand{\out}{{\rm out}}

\newcommand{\dreplay}{d_{{\rm replay}}}

\newcommand{\vlpr}{\textsf{VLPR}}
\newcommand{\VLPR}{\textsf{VLPR}}
\newcommand{\avlpr}{\textsf{AVLPR}}
\newcommand{\AVLPR}{\textsf{AVLPR}}

\newcommand{\ba}{{\mathbf a}}
\newcommand{\br}{{\mathbf r}}

\newcommand{\act}{\mathsf{a}}
\renewcommand{\Mar}{\mathsf{Mar}}

\newenvironment{talign}
 {\align}
 {\endalign}
\newenvironment{talign*}
 {\csname align*\endcsname}
 {\endalign}

\newcommand{\Sighat}{\hat\Sigma_{i,h}^{\bar\pi}}
\newcommand{\Sig}{\Sigma_{i,h}^{\bar\pi}}

%% file: notations_shared.tex
\makeatletter
\def\blfootnote{\gdef\@thefnmark{}\@footnotetext}
\makeatother

%% file: Sections/abstract.tex
\begin{abstract}
A unique challenge in Multi-Agent Reinforcement Learning (MARL) is the \emph{curse of multiagency}, where the description length of the game as well as the complexity of many existing learning algorithms scale \emph{exponentially} with the number of agents. While recent works successfully address this challenge under the model of tabular Markov Games, their mechanisms critically rely on the number of states being finite and small, and do not extend to practical scenarios with enormous state spaces where \emph{function approximation} must be used to approximate value functions or policies. 

This paper presents the first line of MARL algorithms that provably resolve the curse of multiagency under function approximation. We design a new decentralized algorithm---\emph{V-Learning with Policy Replay}, which gives the first \emph{polynomial} sample complexity results for learning approximate Coarse Correlated Equilibria (CCEs) of Markov Games under decentralized linear function approximation. Our algorithm always outputs \emph{Markov} CCEs, and achieves an optimal rate of $\tO(\epsilon^{-2})$ for finding $\epsilon$-optimal solutions. Also, when restricted to the tabular case, our result improves over the current best decentralized result $\tO(\epsilon^{-3})$ for finding Markov CCEs. We further present an alternative algorithm---\emph{Decentralized Optimistic Policy Mirror Descent}, which finds policy-class-restricted CCEs using a polynomial number of samples. In exchange for learning a weaker version of CCEs, this algorithm applies to a wider range of problems under generic function approximation, such as linear quadratic games and MARL problems with low ``marginal'' Eluder dimension.

\end{abstract}


%% file: Sections/intro.tex
\section{Introduction}

\blfootnote{$^\dagger$Princeton University. Email: 
   \texttt{\{yuanhao,qinghual,chij\}@princeton.edu} }
\blfootnote{$^\ddagger$Salesforce Research. Email: 
   \texttt{yu.bai@salesforce.com} }
\blfootnote{$^*$ and $^\mathsection$ denote equal contribution. }

Multi-agent reinforcement learning (MARL) concerns problems in which agents learn to maximize their own utility via interacting with unknown environments as well as other agents, who may be strategic and adaptive. 
Modern MARL systems have achieved significant success on a wide range of challenging tasks, including the game of Go \citep{silver2016mastering,silver2017mastering}, Poker \citep{brown2018superhuman, brown2019superhuman}, strategic games \citep{vinyals2019grandmaster, openaidota, bakhtin2022human, wurman2022outracing}, decentralized controls \citep{brambilla2013swarm}, autonomous driving \citep{shalev2016safe}, as well as complex social scenarios such as hide-and-seek~\citep{baker2020emergent}. Compared to the single-agent RL with a rich literature of theoretical understandings, MARL brings a set of new game-theoretic challenges, many of which remain open.

One unique challenge in MARL is the \emph{curse of multiagency}, where the description length of the game (in particular, the size of the joint action space) scales exponentially with the number of agents. As a result, any learning algorithm that attempts to model the entire game (such as the transition probabilities or joint Q-values) suffers from exponentially large sample or computational complexities \citep{bai2020near,liu2021sharp}. These algorithms are prohibitive to run in practice even for fairly small multi-agent applications. To handle this challenge, practitioners promote the design of \emph{decentralized} algorithms (see, e.g.,~\citet{zhang2021multi} for a review), where agents only aim to learn the relevant pieces of the games from their own local perspectives, such as individual policies, V-values or marginal Q-values (cf. definitions in Section \ref{sec:prelim}). Decentralized algorithms further allow each agent to learn almost independently, with minimal or even no communication between the agents, which gives versatility and advantages to their implementation.

The curse of multiagency in MARL has been \emph{provably} addressed by a recent line of theoretical works \citep{song2021can,jin2021v,mao2022provably} using the V-Learning algorithm~\citep{bai2020near}. However, their results only work for the basic setting of \emph{tabular} Markov games \citep{shapley1953stochastic} where the numbers of states and actions are finite and small. 
Further, their mechanisms rely critically on the tabular setting that permits the synergy of (1) per-state no-regret algorithms, (2) incremental value updates, and (3) optimism; this prohibits a direct extension to practical scenarios with large state spaces.
This is in contrast to modern MARL practice which commonly engages problems with an enormous number of states, where \emph{function approximation}---typically in the form of deep neural networks---must be used to approximate either value functions or policies~\citep{sutton2018reinforcement}. This naturally raises the following open question: 
\begin{center}
\textbf{Can we design decentralized MARL algorithms that breaks the curse of multiagency \\ even with function approximation?}
\end{center}
In this paper, we answer the above question affirmatively by designing algorithms that finds approximate Coarse Correlated Equilibria (CCEs) in the presence of general function approximation, with polynomial sample complexity in problem parameters (including the number of agents). Concretely,

\begin{itemize}[leftmargin=1.5em,topsep=0pt,itemsep=0pt]
\item We design a new decentralized meta-algorithm for MARL---\emph{V-Learning with Policy Replay} (\vlpr{}), and its accelerated version \avlpr{} (Section~\ref{sec:avlpr_main}). Both algorithms integrate the standard V-Learning algorithm~\citep{bai2020near, jin2021v} with new policy replay mechanisms to output Markov CCEs and facilitate learning under function approximation. \vlpr{} is fully decentralized (assuming shared randomness among players), and \avlpr{} requires minimal communication (See Section \ref{sec:dec_main}). Both run in polynomial time given efficient subroutines. 

\item Our meta-algorithms \vlpr{} and \avlpr{} calls for abstract subroutines to (1) estimate V-values for each agent (instead of joint Q-values); (2) compute stage-wise CCE policies by no-regret algorithms. We prove that under mild conditions on the subroutines, both meta-algorithms efficiently find approximate CCEs within a polynomial number of samples (Section~\ref{sec:vlpr}). These mild conditions hold in both linear and tabular settings (Section \ref{sec:vlpr_examples}).

\item We instantiate \avlpr{} in  the setting of \emph{decentralized linear function approximation}, which gives the first decentralized MARL algorithm that provably breaks the curse of multiagency in this setting (Section~\ref{sec:linear-vlpr}). Our algorithm achieves an optimal rate of $\tO(\epsilon^{-2})$ for finding $\epsilon$-optimal solutions. 
For tabular Markov Games, the current best decentralized algorithm for finding Markov CCEs requires $\tO(\epsilon^{-3})$ samples \citep{daskalakis2022complexity}. \avlpr{} improves over this result on the dependency of $1/\eps$, the number of states, as well as the horizon (Section~\ref{sec:tabular-avlpr}).

\item We provide an alternative algorithm \emph{Decentralized Optimistic Policy Mirror Descent} (\DOPMD{}), which finds policy-class-restricted CCEs---a weaker notion of CCEs than standard definition---with sample complexity breaking the curse of multiagency (Section~\ref{sec:restricted}). In exchange for the weaker CCE notion, \DOPMD{} applies to a wider range of problems with general function approximation that has bounded Bellman-Eluder dimension. These problems include linear quadratic games, and  games with low ``marginal'' Eluder dimension or Bellman rank.

\end{itemize}

%% file: Sections/related_work.tex
\subsection{Related work}

In this section, we review previous theoretical works on MARL under the model of Markov Games \citep{shapley1953stochastic,littman1994markov}. We acknowledge the abundant recent work on empirical MARL or under alternative mathematical models, which are beyond the scope of this paper.





\paragraph{Centralized MARL}
Sample-efficient learning of Markov Games has been studied extensively in a recent surge of work~\citep{brafman2002r,wei2017online, jia2019feature,sidford2020solving,bai2020provable,xie2020learning, bai2020near,zhang2020model,tian2020provably,liu2021sharp,bai2021sample,huang2021towards,jin2022power,chen2022almost}. Most of those approaches are centralized in nature, in that they estimate quantities (such as transition models or joint $Q$ functions) whose number of parameters scales exponentially with respect to the number of players, and thus suffer from the curse of multiagency in their sample complexities.

\paragraph{Decentralized MARL}
Decentralized approaches to break the curse of multiagency in Markov Games are pioneered by the V-Learning algorithm, which is initially proposed in the zero-sum setting by~\citet{bai2020near}, and subsequently extended to the general-sum setting~\citep{song2021can,jin2021v,mao2022provably,mao2022improving,cui2022provably,zhang2022policy}, in which it can learn an approximate Correlated Equilibria (CE) or CCE of the game with sample complexity that scales polynomially with respect to the number of agents.
Later, the SPoCMAR algorithm by~\citet{daskalakis2022complexity} further learns approximate CCEs that are guaranteed to be \emph{Markov}\footnote{By contrast, the policies learned by V-Learning are non-Markov, history-dependent policies in general.}, with a slightly worse polynomial sample complexity. Both algorithms only work for tabular Markov Games and do not handle function approximation. Our algorithms \vlpr{} and \avlpr{} can be seen as extensions of the V-Learning algorithm to the function approximation setting and can further output a Markov policy. Furthermore, the specialization of our algorithm to the tabular setting achieves improved sample complexity over~\citet{daskalakis2022complexity} for learning Markov CCEs.

Decentralized algorithms for learning CE/CCEs have also been well-established in other games such as Normal-Form Games (NFGs)~\citep{stoltz2005incomplete,cesa2006prediction} and Extensive-Form Games (EFGs)~\citep{kozuno2021model,bai2022near,bai2022efficient,song2022sample,fiegel2022adapting}, by letting each agent run a no-regret algorithm that works even against adversarial opponents. However, this success does not extend to Markov Games due to the fundamental hardness of learning against adversarial opponents in Markov Games: there is a worst-case exponential-in-horizon regret lower bound~\citep{liu2022learning}. Finally, decentralized algorithms have also been established in Markov Potential Games~\citep{zhang2021gradient,leonardos2021global,song2021can,ding2022independent}---a subclass of Markov Games---which however relies critically on its special potential structure.



\paragraph{MARL with function approximation}
A few recent works consider learning Markov Games with linear~\citep{xie2020learning,chen2022almost} and general~\citep{jin2022power, huang2021towards,zhan2022decentralized,xiong2022self,chen2022unified,ni2022representation} function approximation, by adapting techniques from the single-agent setting~\citep{jiang2017contextual,jin2020provably,zhou2021nearly,du2021bilinear,jin2021bellman,foster2021statistical}. All these works require \emph{centralized} function classes and suffer from the curse of multiagency when specializing to the tabular setting. Our \DOPMD{} algorithm differs from the related algorithms of~\citet{liu2022learning,zhan2022decentralized} where our new inner loop admits decentralized function classes, which could be applied in much broader scenarios.

Technically, the policy replay mechanism used in our algorithms (in particular the one in~\AVLPR{} via doubling tricks) is similar to that of~\citet{zanette2022stabilizing}, which is used there for designing a Q-Learning style algorithm for linear function approximation in the single-agent setting. However, our approaches are otherwise quite different, in particular in the way of updating values, where they use Q-Learning style incremental updates, whereas our algorithms use stage-wise learning with batch updates (similar in spirit to Value Iteration).

\paragraph{Comparison with independent work~\citep{cui2023breaking}} Concurrent to this work, \citet{cui2023breaking} also consider the problem of breaking the curse of multiagency in the context of Markov games under \emph{linear} function approximation, and in addition achieves the same improved sample complexity for finding Markov CCE in the basic \emph{tabular} setting. Here we highlight a few key differences between the two works in the linear setting besides the apparent differences in the algorithm design: (1) \emph{In terms of assumptions}, both works assume Bellman completeness with respect to certain policy classes (see, e.g., Assumption \ref{assumption:q-v-completeness}). We point out that it is crucial to restrict the expressiveness of the policy class, otherwise the game becomes ``essentially tabular'' (see Appendix \ref{app:essentially-tabular}). We only require completeness with respect to linear argmax policies, while \citet{cui2023breaking} require completeness with respect to a policy class $\Pi^\text{estimate}$ that is implicitly defined by their algorithm and the no-regret learning oracle being used, which generally consists of policies that are more complex than linear argmax policies.\footnote{We remark that due to the several key differences between two papers in algorithm design and underlying mechanism, this statement about their $\Pi^{\rm estimate}$ holds true regardless of choosing the full-information no-regret learning oracle in their algorithm as either the Exponential Weights algorithm (the choice in~\citet{cui2023breaking}) or the Expected Follow-The-Perturbed-Leader algorithm~\citep{hazan2020faster} (the choice in our paper). Please see Appendix \ref{app:cui} for more details.}
(2) \emph{In terms of sample complexity}, this paper achieves $\tO(\epsilon^{-2})$ rate which has the optimal statistical dependency on error $\epsilon$, while \citet{cui2023breaking} achieve a rate of $\tO(\epsilon^{-4})$.
We remark that \citet{cui2023breaking} have better dependency in the number of actions $A$, while our results have better dependency in dimension $d$ and horizon $H$. The differences in $A, d, H$ dependency come from the differences in both algorithmic techniques and assumptions (where the minimax-optimal rates can be potentially different).

In addition to the above differences, \citet{cui2023breaking} further provide results for learning under certain amount of model misspecification, learning approximate Correlated Equilibria (CEs), and learning linear Markov Potential Games, all of which have not been considered in this paper. 
On the other hand, this paper presents results \emph{beyond the linear function approximation setting}: both \vlpr{} and \avlpr{} are generic meta-algorithms that provide guarantees for any function class as long as the required conditions for the subroutines are fulfilled. We further design a new algorithm for general function approximation that learns policy-class-restricted CCEs under weaker conditions (Section~\ref{sec:restricted}).

%% file: Sections/prelim.tex
\section{Preliminaries} \label{sec:prelim}

\paragraph{Markov Games}
We consider episodic general-sum Markov Games with $m$ players, which can be described as a tuple ${\rm MG}(H,\cS,\sets{\cA_i}_{i\in[m]},\P,\sets{r_i}_{i\in[m]})$. Here $H$ is the horizon length, $\cS$ is the state space, $\cA_i$ is the action space of the $i$-th player with $\abs{\cA_i}=A_i$\footnote{Our results in Section~\ref{sec:vlpr}, \ref{sec:restricted} do not require $A_i$ to be finite. 
}; we use $\ba=(a_1,\dots,a_m)\in\prod_{i=1}^m \cA_i\eqdef \cA$ to denote a joint action for all players, $\P=(\P_h)_{h\in[H]}$ are the transition probabilities, where $\P_h(\cdot|s,\ba)\in\Delta(\cS)$ is the probability distribution of the next state at current state-action $(s,\ba)$ at step $h$; $r_i=(r_{i,h})_{h\in[H]}$ are the reward functions for player $i$, where each each $r_{i,h}:\cS\times\cA\to[0,1]$ is a function that maps any state-action $(s,\ba)$ to a deterministic\footnote{Our results can generalize directly to the case of stochastic rewards.} reward. In each episode, the game starts at a fixed initial state $s_1$. At step $h\in[H]$ and state $s_h$, each player takes their own action $a_{i,h}\in\cA_i$, receives their own reward $r_{i,h}=r_{i,h}(s_h,\ba_h)$ where $\ba_h=(a_{1,h},\dots,a_{m,h})$, and the game transits to the next state $s_{h+1}\sim \P_h(\cdot|s_h,\ba_h)$ in a Markov fashion.

A Markov policy for the $i$-th player is denoted by $\pi_i=\sets{\pi_{i,h}(\cdot|s)\in\Delta(\cA_i)}_{(s,h)\in\cS\times[H]}$, which prescribes a distribution $\pi_{i,h}(\cdot|s)\in\Delta(\cA_i)$ over the $i$-th player's actions at any $(s, h)$. Here, we use $\Delta(\cA_i)$ to denote the probability simplex over the action set $\cA_i$. 
A Markov \emph{joint} policy $\pi=\sets{\pi_h(\cdot|s)\in\Delta(\cA)}_{(s,h)\in\cS\times[H]}$ is a joint policy over all players that prescribes a distribution over the joint actions, where the randomness of different players can be \emph{correlated} in general.
A special case of Markov joint policy is \emph{product policy} $\pi = \sets{\pi_i}_{i\in[m]}$ where each agent plays $\pi_i$ \emph{independently}.
For any joint policy $\pi$, we define its V-value function and (joint) Q-value function for any $(i,h)\in[m]\times[H]$ as
$V^\pi_{i,h}(s) \defeq \E_{\pi}\bracs{ \sum_{h'=h}^H r_{i,h'}(s_{h'}, \ba_{h'}) \mid s_h=s  }$ and $Q^\pi_{i,h}(s, \ba) \defeq \E_{\pi}\bracs{ \sum_{h'=h}^H r_{i,h'}(s_{h'}, \ba_{h'}) \mid (s_h, \ba_h) = (s, \ba)  }$ respectively. 
Additionally, with a slight overload in notations, we define the \emph{marginal Q-function} for player $i\in[m]$ and any $h\in[H]$ as
\begin{talign*}
    Q_{i,h}^\pi(s, a_i) \defeq \E_{\pi}\brac{ \sum_{h'=h}^H r_{i,h'}(s_{h'}, \ba_{h'}) \mid (s_h, a_{i,h}) = (s, a_i)  },
\end{talign*}
which measures the Q-value of player $i$ conditioned at a state and their own action, while marginalizing over the opponents' actions according policy $\pi$.
For notational simplicity, we define operator $\P_h$ and $\D_{\pi}$ as $[\P_hV](s,\ba)\defeq \E_{s'\sim\P_h(\cdot|s,\ba)}[V(s')]$ and $\D_{\pi}[Q](s)\defeq \E_{\ba\sim\pi(\cdot|s)}[Q(s, \ba)]$. We also use $\pi_{-i}$ to denote the joint policy of all but the $i$-th player specified by $\pi$. For any Markov product policy $\pi=\pi_i\times\pi_{-i}$, the Bellman operator $\cT^{\pi}_{i,h}$ for player $i$ at step $h$ is a self-map over the $i$-th player's marginal Q-function space $(\cS\times\cA_i \rightarrow\R)$, defined as
\begin{talign*}
    [\cT^{\pi}_{i,h} f](s,a_i) \defeq \E_{\a_{-i}\sim\pi_{-i,h}(\cdot|s), s'\sim\P_h(\cdot\mid s,\a), a_i'\sim\pi_{i,h+1}(s')} \brac{ r_{i,h}(s,\a) + f(s',a_i') }.
\end{talign*}

\paragraph{Coarse Correlated Equilibrium}
Our goal is to find an approximate equilibrium of the Markov Game, i.e., a joint policy such that each player's own policy is near-optimal against their opponents in a certain sense. In our multi-player general-sum setting, the standard notion of Nash Equilibrium is both computationally PPAD-hard~\citep{daskalakis2013complexity} and statistically intractable, requiring $\exp(\Omega(m))$ samples~\citep{rubinstein2017settling}. 
We focus on learning Coarse Correlated Equilibrium (CCE), a common relaxed notion of equilibrium for general-sum Markov Games~\citep{liu2021sharp}, which does not exhibit such hardness and can indeed be learned with polynomial time and samples in the basic tabular setting~\citep{song2021can,jin2021v,mao2022provably}.

For any $\eps>0$, we say that a joint policy $\pi$ is an $\eps$-approximate CCE of the game if
\begin{talign*}
    \textstyle{  \CCEGap(\pi) \defeq \max_{i\in[m]} (\max_{\pi_i^\dagger} V_{1,i}^{\pi_i^\dagger, \pi_{-i}}(s_1) - V_{1,i}^{\pi}(s_1)) \le \eps,}
\end{talign*}
Here, the maximizer $\pi_i^\dagger$ is also known as the best response. We denote $V_{i,h}^{\dagger,\pi_{-i}}:=\max_{\pi_i^\dagger}V_{i,h}^{\pi_i^\dagger,\pi_{-i}}$.

We consider the standard setting of PAC learning from bandit feedback, where the agents repeatedly interact with the underlying Markov Game for many episodes, and observe the trajectory $(s_1,\ba_1,\br_1,\dots,s_H,\ba_H,\br_H)$ (where $\br_h\defeq (r_{i,h})_{i\in[m]}$) within each episode. The goal is to find an $\eps$-approximate CCE $\what{\pi}$ of the game within as few episodes of play as possible.

\subsection{Decentralized MARL with function approximation}

To allow decentralized MARL with large state spaces, this paper considers \emph{function approximation}, where each player $i\in[m]$ has her own marginal Q-value function class $\cF_i$.
Formally, we let each player $i\in[m]$ be equipped with finite\footnote{Our results extend directly to the case of infinite function classes via standard covering arguments.} function class $\cF_i=\cF_{i,1}\times\dots\times \cF_{i,H}$, where
each $f_{i,h}\in\cF_{i,h}\subset( \cS\times\cA_i\to \R)$ models a marginal Q-function at step $h\in[H]$.\footnote{While we focus on Q-type function approximation, our meta-algorithms can also extend to V-type function approximation, though the two types may encompass fairly different problem structures; see Appendix~\ref{app:v_type} for a discussion.}

With suitable assumptions about $\cF_i$ and the game (presented in the sequel), we are interested in finding an approximate CCE with sample complexity avoiding the \emph{curse-of-multiagent}~\citep{jin2021v,song2021can}, i.e. scaling polynomially in $\max_{i\in[m]} \log |\cF_i|$, the number of players $m$, as well as all other problem parameters.

%% file: Sections/abstract_alg.tex
\section{Decentralized MARL via policy replay: meta-algorithms and guarantees}
\label{sec:vlpr}

\begin{algorithm}[t]
\small
\caption{V-Learning with Policy Replay (\vlpr)}
\label{alg:vlpr}
\begin{algorithmic}[1]
\STATE \textbf{Initialize} $\pi^1$ to be the uniform policy: $\pi^1_{i,h}(\cdot|s)\setto \Unif(\cA_i)$ for all $(i,s,h)$.
\FOR{iteration $t=1,\ldots,T$}
\STATE Set replay policy  $\bar\pi^t\setto \Unif(\{\pi^{\tau}\}_{\tau\in [t]})$ and $\oV_{i,H+1}^{t+1}\setto 0$. \label{line:policy-replay}
\FOR{$h=H,\ldots,1$}
\STATE Compute $\pi^{t+1}_h\setto\cceaprx_h(\bar\pi^t,\{\oV_{i,h+1}^{t+1}\}_{i\in[m]},t)$.
\STATE Compute $\{\oV_{i,h}^{t+1}\}_{i\in[m]}\setto\vapprox_h(\bar\pi^t,\pi^{t+1}_h,\{\oV_{i,h+1}^{t+1}\}_{i\in[m]},t)$.
\ENDFOR
\ENDFOR
\ENSURE $\pi^{\rm out}$ sampled uniformly at random from $\{\pi^t\}_{t\in[T]}$.
\end{algorithmic}
\end{algorithm}

\paragraph{Algorithm}
Our first main algorithm, V-Learning with Policy Replay (\VLPR{}; Algorithm~\ref{alg:vlpr}), is a meta-algorithm for decentralized MARL with function approximation. At a high level, \VLPR{} adopts a \emph{policy replay} mechanism (Line~\ref{line:policy-replay}), which in the $t$-th iteration sets the \emph{roll-in policy} $\bar\pi^t=\Unif(\set{\pi^\tau}_{\tau\in[t]})$ to be the uniform mixture of all previously learned policies. Using this roll-in policy, it then learns a new approximate CCE-policy $\pi^{t+1}$ by \emph{stage-wise learning} which recursively computes the approximate CCE policies and V-values from  $h=H$ to $1$ using two subroutines:
\begin{itemize}[leftmargin=1.5em,topsep=0pt,itemsep=0pt]
\item $\cceaprx_h$ (Algorithm~\ref{alg:abs-cce}) takes in value estimates $\{\oV_{i,h+1}^{t+1}\}_{i\in[m]}$, and computes an approximate CCE $\pi_h^{t+1}$ for the $h$-th step. It requires two ingredients: (1) An ordered set of \emph{exploration policies} and \emph{active players}  $(\wt{\pi},P)\in \Piexp(\bar\pi, \mu_h^k)$ ($P\subseteq [m]$ is an index set), where each round executes each such $\wt{\pi}$ to observe a trajectory, and adds the observation $(s_h, a_{i,h},r_{i,h}+\oV_{i,h+1}(s_{h+1}))$ into the $i$-th player's dataset $\cD^{k,i}_{\rm sample}$ iff $i\in P$. (2) Each player then runs a no-regret algorithm~$\noreg$ using the collected data. We require relatively strong $\noreg$, which achieves small \emph{per-state} regret in the face of large state spaces (in a proper sense) under bandit feedback (cf. Condition~\ref{cond:cce-regret}), which will be discussed momentarily.
\item $\vapprox_h$ (Algorithm~\ref{alg:abs-v}) takes in the new policy $\pi_h^{t+1}$ and value estimates $\{\oV_{i,h+1}^{t+1}\}_{i\in[m]}$, and produces estimates $\{\oV_{i,h}^{t+1}\}_{i\in[m]}$ for the $h$-th step by regression algorithm $\optreg$, which is required to achieve optimistic estimation with small errors (cf. Condition~\ref{cond:optv}). 
\end{itemize}

Notably, \VLPR{} combines the policy replay mechanism and the $\vapprox$ subroutine which re-learns a new value function at each iteration in a \emph{batch} fashion.
This mechanism is different from the standard V-Learning algorithm which directly plays a newly learned policy in each iteration without replay, but uses incremental value updates. That mechanism effectively learns the value of an implicit ``output policy'' (the ``certified policy'') which is different from the previously played policies~\citep{bai2020near,jin2021v,song2021can,mao2022provably}. However, in the presence of function approximation, the batch learning in \VLPR{} is preferred and precisely enabled by the policy replay mechanism, as it is otherwise unclear how to generalize the incremental value update approach to the case with general function classes.

\begin{algorithm}[t]
\caption{$\cceaprx_h(\bar\pi,\{\oV_{i,h+1}\}_{i\in[m]},K)$}
\label{alg:abs-cce}
\small
\begin{algorithmic}[1]
\REQUIRE Exploration policy mapping $\Piexp$; subroutine $\noreg$.
\STATE Execute $\bar\pi$ for $K$ episodes to collect $\{\cD^i_{\rm init}\}_{i \in [m]}$. Initialize $\cD_{\rm sample}^{k,i}\setto \{\}$ for all $(i,k)\in[m]\times[K]$.
\label{line:collect-dinit}
\FOR{$k=1,\ldots,K$}
\FOR{$(\tilde\pi,P)\in\Piexp(\bar\pi,\mu_h^k)$}
\STATE Execute $\tilde\pi$ to collect a trajectory $(s_1, \ba_1, \br_1, \dots, s_H, \ba_H, \br_H)$.
\STATE Update $\cD_{\rm sample}^{k,i}\setto \cD_{\rm sample}^{k,i}\cup \sets{(s_h,a_{i,h},r_{i,h}+\oV_{i,h+1}(s_{h+1}))}$ for all $i\in P$.
\ENDFOR
\STATE Update $\mu_{i,h}^{k+1}\setto\noreg(\mu_{i,h}^{k},\cD_{\rm sample}^{k,i},\cD_{\rm init}^i)$ for all $i\in[m]$.
\ENDFOR
\ENSURE $\piout_h :=\frac{1}{K}\sum_{k=1}^K\mu_h^k$, where $\mu_h^k=\mu^{k}_{1,h}\times\cdots\times \mu^{k}_{m,h}$.
\end{algorithmic}
\end{algorithm}

\begin{algorithm}[t]
\caption{$\vapprox_h(\bar\pi, \pi_h, \{\oV_{i,h+1}\}_{i\in[m]}, K)$}
\label{alg:abs-v}
\small
\begin{algorithmic}[1]
\REQUIRE Exploration policy mapping $\Piexp$; subroutine $\optreg$.
\STATE Initialize $\cD_{\rm reg}^i\setto \{\}$ for all $i\in[m]$.
\FOR{$k=1,\ldots,K$}
\FOR{$(\tilde\pi,P)\in\Piexp(\bar\pi,\pi_h)$}
\STATE Execute $\tilde\pi$ to collect a trajectory $(s_1, \ba_1, \br_1, \dots, s_H, \ba_H, \br_H)$.
\STATE Add $(s_h,a_{i,h},r_{i,h}+\oV_{i,h+1}(s_{h+1}))$ into $\cD_{\rm reg}^{i}$ for all $i\in P$.
\ENDFOR
\ENDFOR
\STATE $\oV_{i,h}\setto\optreg(\pi_{i,h},\cD_{\rm reg}^i)$ for all $i \in [m]$.
\ENSURE $\{\oV_{i,h}\}_{i\in[m]}$.
\end{algorithmic}
\end{algorithm}

\paragraph{Conditions and guarantee}
\VLPR{} is a generic meta-algorithm. Once the subroutines satisfy specific requirements, the meta-algorithm will be guaranteed to learn an approximate CCE of the game.

\begin{condition}[Required conditions for \VLPR{}]
\label{cond:G-condition}
There exists \emph{bonus function} $G_{i,h}(s,\bar\pi, K,\delta)$ for every $(i,h)\in[m]\times[H]$
such that the followings hold when executing Algorithm \ref{alg:vlpr}.
\begin{enumerate}[label=(1\Alph*), topsep=0pt, itemsep=0pt]
\item {\bf Per-state no-regret}: \label{cond:cce-regret}
Subroutine $\pi = \cceaprx_h(\bar\pi,\{\oV_{i,h+1}\}_{i\in[m]},K)$ (Algorithm~\ref{alg:abs-cce}) satisfies that with probability at least $1-\delta$, for all $(i,s)\in[m]\times\cS$: 
\begin{talign*}
\max_{\mu_{i,h}\in\Delta(\cA_i)} 
\left(\D_{\mu_{i,h}\times \pi_{-i,h}}-\D_{ \pi_{h}}\right)\left[r_{i,h}+\P_{h+1}\oV_{i,h+1} \right](s)
\le G_{i,h}(s,\bar\pi,K,\delta).
\end{talign*}
\item {\bf Optimistic V-estimate}:\label{cond:optv} Subroutine $\oV_{i, h} = \vapprox_h(\bar\pi, \pi_h, \{\oV_{i,h+1}\}_{i\in[m]}, K)$ (Algorithm~\ref{alg:abs-v}) satisfies that with probability at least $1-\delta$, for all $(i,s)\in[m]\times\cS$:
$$
\begin{cases}
\oV_{i,h}(s) \ge \min\left\{\D_{ \pi_{h}}\left[ r_{i,h}+\P_{h+1}\oV_{i,h+1} \right](s)  + G_{i,h}(s,\bar\pi,K,\delta),H-h+1\right\}, \\
\oV_{i,h}(s) \le \qquad~ \D_{ \pi_{h}}\left[ r_{i,h}+\P_{h+1}\oV_{i,h+1} \right](s)+2G_{i,h}(s,\bar\pi,K,\delta).
\end{cases}
$$
\item {\bf Pigeon-hole condition}:\label{cond:pigeon}
There exists an absolute complexity measure $L\in\R^+$ such that for any $(i,h)\in[m]\times[H]$, $(T,\delta)\in\N\times(0,1)$, and any policy sequence $\sets{\pi^1,\ldots,\pi^T}$,
\begin{talign*}
    \sum_{t=1}^T  \E_{s_h\sim\pi^{t+1}}\left[G_{i,h}(s_h,\Unif(\{\pi^\tau\}_{\tau\in[t]}), t,\delta)  \right] \le \sqrt{LT\log^2(T/\delta)}.
\end{talign*}
\end{enumerate}
\end{condition}

Condition~\ref{cond:cce-regret} requires that the subroutine $\cceaprx$ (which calls $\noreg$) achieves \emph{per-state} low-regret (recall in Algorithm \ref{alg:abs-cce} the output policy is a uniform mixture of polices that are played). This is more stringent than regret bounds w.r.t. a fixed state distribution as in standard \emph{contextual} bandit problems~\citep{lattimore2020bandit}, but is crucial for learning CCEs which require the learned policies to extrapolate well to multiple roll-in distributions. 

Condition~\ref{cond:optv} requires the subroutine $\vapprox$ (which calls $\optreg$) to produce \emph{optimistic} and \emph{accurate} value estimates for policy $\pi_h$, in a precise sense that the difference between the estimate $\oV_{i,h}$ and the ground truth $\D_{ \pi_{h}}\left[ r_{i,h}+\P_{h+1}\oV_{i,h+1} \right]$ is sandwiched (modulo truncation) within $[1,2]$ times the bonus function $G_{i,h}$.

Condition~\ref{cond:pigeon} has a similar flavor to the pigeon-hole principle, and is used to ensure the expected bonuses sum up to $\tO(\sqrt{T})$ as in UCB-style algorithms, e.g.,~\citet{azar2017minimax,jin2020provably}.

We are now ready to state our main guarantee for \VLPR{}.
\begin{theorem}[``Regret'' guarantee for VLPR]
\label{thm:vlpr}
Suppose Condition \ref{cond:G-condition} holds for Algorithm \ref{alg:vlpr}. Then with probability at least $1-3\delta$, we have that
\begin{talign}
\label{eqn:ccereg_vlpr}
\CCEReg(T) \defeq \max_{i\in[m]} \sum_{t=1}^T \left[V_{i,1}^{\dagger,\pi^t_{-i}}(s_1) - V_{i,1}^{\pi^t}(s_1)\right]_+ \le \tO(\sqrt{H^2LT}).
\end{talign}
\end{theorem}
\begin{corollary}[Sample complexity] \label{remark:sample}
Choosing $T= \tO(H^2L/\eps^2)$ ensures that the output policy $\pi^{\out}$ of Algorithm~\ref{alg:vlpr} satisfies $\CCEGap(\pi^{\out})\le\eps$ further with probability at least\footnote{The success probability can be further improved to $1-\delta$ for any small $\delta>0$ with at most an additional $\log(1/\delta)$ factor in the sample complexity, using an optimistic evaluation of the $\CCEGap$ combined with boosting.} $0.99$, and the total number of episodes played is at most  $\tO\paren{ H^5L^2\Piexpup / \eps^4 }$, where $\Piexpup \defeq \max_{\bar\pi,\pi'}\abs{\Piexp(\bar\pi, \pi')}$.
\end{corollary}

Theorem~\ref{thm:vlpr} and Corollary~\ref{remark:sample} assert that an $\eps$-approximate CCE can be found within ${\rm poly}(H,L,\Piexpup,1/\eps)$ samples, as long as all the subroutines in Algorithm~\ref{alg:vlpr} satisfy Condition \ref{cond:G-condition}. The proof (given in Appendix~\ref{app:proof-vlpr}) is relatively straightforward given the conditions, which uses performance difference arguments and combine Condition~\ref{cond:cce-regret} \&~\ref{cond:optv} to upper bound $\CCEReg(T)$ by the bonuses, and uses Condition~\ref{cond:pigeon} to further bound the summation of the bonuses over $t\in[T]$.

\subsection{Accelerated $\tO(1/\eps^2)$ algorithm via infrequent policy updates} \label{sec:avlpr_main}
The $\tO(1/\eps^4)$ rate obtained in Theorem~\ref{thm:vlpr} is slower than the standard $1/\eps^2$ rate. This happens as \VLPR{} adopts the replay mechanism and updates the policy \emph{at every iteration} $t\in[T]$, which causes the $T\times T=\tO(1/\eps^4)$ rate. However, such a frequent policy update may be unnecessary if the roll-in distributions induced by the replay policies $\sets{\bar\pi^t}_{t\ge 1}$ do not change quickly over $t$. 

To address this, we design an accelerated algorithm called \AVLPR{} (Algorithm~\ref{alg:AVLPR}) that improves this rate to $1/\eps^2$ under an additional condition (Condition~\ref{cond:switch}) that allows the algorithm to perform  well with \emph{infrequent policy updates}---more precisely $\cO(\log T)$ updates---within $T$ iterations (Theorem~\ref{thm:avlpr}). We will realize this condition by \emph{doubling tricks}. See Appendix~\ref{app:avlpr} for details.

\subsection{Decentralized execution} \label{sec:dec_main}
Our algorithms \VLPR{} and \AVLPR{} are thus far described in terms of all players jointly. Nevertheless, both algorithms can be implemented in a decentralized fashion. Rigorously, we consider the setting that each player is only able to see the shared state and their own action and reward. That is, they do not know other players' actions or rewards if without communication. We show that using certain simple protocols, \VLPR{} can be executed in a fully decentralized fashion without any communication (assuming shared randomness among players), and \AVLPR{} can be executed with $\cO(\log T)$ rounds of extremely small communication only for the checking the triggering condition (Line~\ref{line:avlpr-trigger} in Algorithm~\ref{alg:AVLPR}). We defer the detailed arguments to Appendix~\ref{app:decentralized-execution}.

%% file: Sections/linear.tex
\section{Instantiation in linear and tabular settings} \label{sec:vlpr_examples}

We now instantiate \AVLPR{} concretely in two settings: decentralized linear function approximation (a new setting), and learning Markov CCEs for tabular Markov Games. We focus on the sample complexity here; both instantiations are also computationally efficient (cf. Appendix~\ref{app:details-linear-avlpr} \&~\ref{app:details-tabular-avlpr}).

\subsection{Decentralized linear function approximation}
\label{sec:linear-vlpr}

We consider Markov Games with \emph{decentralized linear function approximation}, where each $\cF_{i,h}=\{f_{i,h}(\cdot,\cdot)=\phi_i(\cdot,\cdot)^\top \theta_h: \Vert \theta_h\Vert_2\le B_\theta := H\sqrt{d}\}$ is a linear function class
with respect to a known $d$-dimensional \emph{feature map}\footnote{Without loss of generality, we assume bounded features: $\sup_{(s,a_i)\in\cS\times\cA_i} \ltwo{\phi_i(s,a_i)}\le B_\phi \defeq 1$ for all $i\in[m]$.} $\phi_i:\cS\times \cA_i\to \R^{d}$. We consider the class of linear argmax policies 
\begin{talign}
\label{eqn:linear_argmax_policies}
    \Pi_{i,h}^{\lin} := \set{\pi_{i,h}(\cdot|s) = \arg\max_{a_i\in\cA_i} \phi_i(s,a_i)^\top w_{i,h},~\forall~s\in\cS \mid w_{i,h}\in \R^{d} }.
\end{talign}
induced by the feature map $\phi_i$, and denote $\Pi^{\rm lin}_i=\bigtimes_{h\in[H]} \Pi_{i,h}^{\lin}$ and $\Pi^{\lin}\defeq \bigtimes_{i\in[m]}\Pi_i^{\lin}$.
To ensure that the feature map is informative enough, we make the following assumption. 

\begin{assumption}[$\Pi^{\lin}$-completeness]
\label{assumption:q-v-completeness}
For any $(i,h)\in [m]\times [H]$, any $f_{i,h+1}:\cS\times \cA_i\to [0,H]$, any $\pi\in\Pi^{\lin}$, we have $\cT_{i,h}^\pi f_{i,h+1}\in\cF_{i,h}$.
\end{assumption}

At $m=1$ (the single-agent setting), Assumption~\ref{assumption:q-v-completeness} is strictly weaker than the linear MDP assumption~\citep{jin2020provably} but stronger than the linear completeness assumption~\citep{zanette2020learning}, both common assumptions for RL with linear function approximation. 
For $m\ge 2$, Assumption~\ref{assumption:q-v-completeness} can be seen as a decentralized multi-agent generalization of the linear MDP assumption, which requires that for every player $i\in[m]$ the Bellman backup of any $\oV_{i,h+1}$ with respect to any linear argmax policy $\pi_{-i}$ is contained in $\cF_{i,h}$ (thus is linear in $\phi_i(s,a_i)$).

We remark that in Assumption~\ref{assumption:q-v-completeness}, requiring completeness only for the \emph{restricted} policy class $\Pi^{\lin}$ is crucial: if completeness is required for all Markov policies,
then the game is ``essentially tabular'' in the sense that the number of \emph{non-trivial states} must be small (cf. Appendix~\ref{app:essentially-tabular}).

\paragraph{Main result}  For decentralized linear function approximation, we instantiate \AVLPR{} to obtain the following guarantee. The algorithmic details and the proof can be found in Appendix~\ref{app:proof-linear}.
\begin{theorem}[\AVLPR{} for decentralized linear function approximation]
\label{thm:linear-rate}
Suppose the decentralized linear function approximation satisfies Assumption~\ref{assumption:q-v-completeness}. Then a suitable instantiation of \avlpr{} finds an $\epsilon$-CCE within
$\tO\left(d^4H^6 m^2 (\max_{i\in[m]} A_i^5) / \epsilon^2\right)$
episodes of play.
\end{theorem}
Theorem~\ref{thm:linear-rate} achieves a $\tO(1/\eps^2)$ sample complexity with polynomial dependence on $(d,H,m,\max_{i\in[m]}A_i)$, avoiding the curse of multiagency. To our best knowledge, this is the first such result for learning Markov Games with decentralized linear function approximation.

\paragraph{Overview of techniques}
Establishing Theorem~\ref{thm:linear-rate} requires instantiating the $\noreg$ and $\optreg$ subroutines in \AVLPR{} for the linear function approximation setting such that Conditions~\ref{cond:cce-regret}-\ref{cond:pigeon} \&~\ref{cond:switch} are satisfied. We choose $\optreg$ to be the standard ridge regression, which ensures Condition~\ref{cond:optv} by Assumption~\ref{assumption:q-v-completeness}. 

The more challenging task is to choose $\noreg$ that satisfies Condition~\ref{cond:cce-regret}, which, roughly speaking, requires (1) per-state regret guarantees at all $s\in\cS$; (2) the policies $\set{\mu^k_h}_{k\in[K]}$ to lie in $\Pi^{\lin}$. Perhaps counter-intuitively, this rules out either running a separate linear adversarial bandit algorithm at each state, which violates (2), or adversarial \emph{contextual} linear bandit algorithms such as LINEXP3~\citep{neu2020efficient}, which violates (1). We resolve this by converting the problem into $\cS$ parallel online linear optimization problems using the special structure of $\Pi^{\lin}$, and applying the Expected Follow-the-Perturbed-Leader algorithm~\citep{hazan2020faster} to produce a single set of iterates within $\Pi^{\lin}$ that solves all $\cS$ problems simultaneously (without any $\abs{\cS}$ dependence in rate), thereby fulfilling both requirements.

With these subroutines chosen, we show that Condition~\ref{cond:G-condition}
is satisfied with bonus function
\begin{talign}
\label{equ:g-linear}
    & \textstyle G_{i,h}(s,\bar\pi,K,\delta) := \wt{\Theta}\left(\max_{a_i\in\cA_i}\|\phi_i(s,a_i)\|_{ 
(\Sigma^{\bar\pi}_{i,h}+\lambda I)^{-1} } \times d (\max_i A_i^{1.5}) H / \sqrt{K}+K^{-1}\right), \\
    & ~~~{\rm where} ~~~\Sigma^{\bar\pi}_{i,h}:=\E_{s_h\sim\bar\pi}~\E_{a_{i,h}\sim \Unif(\cA_i)}\left[\phi_i(s_h,a_{i,h})\phi_i(s_h,a_{i,h})^\top\right], \quad \lambda = \wt\Theta(d(\max_i A_i)/K) \nonumber.
\end{talign}

%% file: Sections/tabular.tex
\subsection{Learning Markov CCE in tabular Markov Games}
\label{sec:tabular-avlpr}




We also instantiate \AVLPR{} on tabular Markov Games (where $\cF_i$ is the class of all possible marginal Q functions), and obtain the following result (algorithm details and proof in Appendix~\ref{sec:proof-tabular-appendix}).

\begin{theorem}[Tabular Markov Games]
\label{thm:tabular-rate}
For tabular Markov Games with $S$ states, a suitable instantiation of \avlpr{} finds a \emph{Markov} $\epsilon$-CCE within $\tO\paren{ H^6S^2(\max_{i\in[m]} A_i) / \eps^2 }$ episodes of play.
\end{theorem}
The only existing algorithm for learning Markov CCEs avoiding the curse of multiagency is the SPoCMAR algorithm of~\citet{daskalakis2022complexity}, which achieves a $\tO\left(H^{10}S^3(\max_iA_i)/\epsilon^3\right)$ sample complexity. Theorem~\ref{thm:tabular-rate} achieves both an improved $(H,S)$ dependence and a near-optimal $\tO(1/\eps^2)$ rate. To establish Theorem~\ref{thm:tabular-rate}, we instantiate $\noreg$ to be a separate EXP3 algorithm at every state $s\in\cS$, and $\optreg$ to be simply a state-wise optimistic value estimate. We show that these ensure Conditions \ref{cond:G-condition} with following bonus function:
\begin{talign*}
 \textstyle G_{i,h}(s,\bar\pi,K,\delta) := \wt\Theta\paren{ \eta_i^{-1} (J_h(s)+\iota)^{-1}+\eta_i H^2A_i },
\end{talign*}
where $\eta_i$ is the learning rate for the $i$-th player's $\noreg$,  $J_h(s)$ is the expected visitation count of state $s$ at step $h$ when running roll-in policy $\wb\pi$ for $K$ episodes, and $\iota=\tO(1)$.

%% file: Sections/pi_cce.tex
\section{Learning CCE within restricted policy classes}
\label{sec:restricted}

In this section, we present an alternative approach for learning a \emph{CCE within a restricted policy class} $\Pi$ (henceforth $\Pi$-CCE) under potentially much more relaxed assumptions on the function class.

\paragraph{Restricted policy class}
We let each player $i\in[m]$ be equipped with a class $\Pi_i$ of Markov policies (in addition to their marginal Q class $\cF_i$), and let $\Pi\defeq \prod_{i\in[m]}\Pi_i$ be the set of product policies over $\sets{\Pi_i}_{i\in[m]}$. For any joint policy $\Lambda$, we say $\Lambda$ is an $\eps$-approximate $\Pi$-CCE if
\begin{talign*}
    \CCEGap^\Pi(\Lambda) \defeq \max_{i\in[m]} \paren{\max_{\pi_i^\dagger\in\Pi_i} V_{1,i}^{\pi_i^\dagger\times \Lambda_{-i}}(s_1) - V_{1,i}^{\Lambda}(s_1)} \le \eps.
\end{talign*}
In words, $\Lambda$ is an approximate $\Pi$-CCE as long as no player gains much by deviating to some other policy within $\Pi_i$. Note that we always have $\CCEGap^\Pi(\Lambda)\le\CCEGap(\Lambda)$, and the inequality is in general strict even when $\Pi_i$ is the set of all possible Markov policies for player $i$ (the largest class allowed here)\footnote{Concretely, there exists a Markov Game in which there exists a $\Lambda\in\Delta(\Pi^{\Mar})$ such that $\CCEGap^{\Pi^{\Mar}}(\Lambda)=0$ but $\CCEGap(\Lambda)\ge H/4$ for any $H\ge 2$; see Appendix~\ref{app:separation_pi_cce} for the construction.}, so that the $\Pi$-CCE is in general a more restricted notion.

\paragraph{Assumptions}
Our first assumption requires each function class $\cF_i$ to be complete with respect to Bellman operators $\sets{\cT^\pi_{i,h}}_{\pi, h}$, a standard assumption to ensure accurate value estimation via square-loss regression~\citep{jin2021bellman}.   This assumption relaxes Assumption~\ref{assumption:q-v-completeness} since this assumption only holds for $f_{i,h+1}\in\cF_{i,h+1}$ (while Assumption~\ref{assumption:q-v-completeness} holds for arbitrary $f_{i,h+1}$).
\begin{assumption}[$\Pi$-completeness]
\label{asp:completeQ}
For every $i\in[m]$, the function class $\cF_i$ satisfies completeness with respect to $\Pi$, that is, for any $h\in[H]$ and $(f_{i,h+1},\pi)\in \cF_{i,h+1}\times\Pi $, we have $\cT^\pi_{i,h} f_{i,h+1}\in\cF_{i,h}$.
\end{assumption}

We also require each $\cF_i\subset((\cS\times\cA_i)\to [0,H])$ to have bounded Bellman-Eluder (BE) dimension~\citep{jin2021bellman} to ensure sample-efficient RL. For any $i\in[m]$, we define
\begin{align}
\label{eqn:def_eluder}
d_i(\cF_i, \Pi, \epsilon) \defeq \max_{\pi_{-i}\in\Pi_{-i}} \dBEi,
\end{align}
where $\dBEi$ denotes the Bellman-Eluder dimension of $\cF_i$ with respect to the Bellman operators $\sets{\cT_{i,h}^{\pi_i\times \pi_{-i}}}_{\pi_i\in\Pi_i}$ (cf. Definition~\ref{def:bellman-eluder}). 
The Bellman-Eluder dimension is a standard complexity measure in single-agent RL for controlling the complexity of exploration.
We assume such Bellman-Eluder dimension of the marginal value functions to be bounded for all players $i\in[m]$.
\begin{assumption}[Bounded BE dimension]
\label{asp:bounded_eluder}
There exist scalars $\sets{d_i}_{i\in[m]}$ such that for all $i\in[m]$ and $\eps\in(0,1)$, we have $d_i(\cF_i,\Pi,\epsilon)\le d_i\log(1/\eps)$.
\end{assumption}
Note that Assumption~\ref{asp:completeQ} \&~\ref{asp:bounded_eluder} are both \emph{decentralized} in nature, as they only require properties about $(\cF_i,\Pi_i)$ in the single-agent MDP induced by a fixed $\pi_{-i}\in\Pi_{-i}$. These are in contrast to previous approaches for learning Markov Games with general function approximation, which require similar structural conditions on their \emph{centralized} function classes~\citep{jin2022power,huang2021towards,chen2022unified}.

\subsection{Algorithm and guarantee}

\begin{algorithm}[t]
\small
\caption{\OPMD{}: Decentralized Optimistic Policy Mirror Descent}
\label{alg:opmd}
\begin{algorithmic}[1]
\REQUIRE Learning rate $\sets{\eta_i}_{i\in[m]}$, function class $\sets{\cF_i}_{i\in[m]}$, policy class $\sets{\Pi_i}_{i\in[m]}$, $\sets{(K_i,\beta_i)}_{i\in[m]}$.
\STATE \textbf{Initialize} $\Lambda^{1}_i\setto \Unif(\Pi_i)$ for all $i\in[m]$.
\FOR{round $t=1,\dots,T$}
\STATE Sample a policy $\pi^t_i\sim \Lambda^t_i$ for each $i\in[m]$, and set $\pi^{t} = \pi^{t}_1\times \ldots\times \pi^{t}_m$. \label{line:sample-policy}
\FOR{$i\in[m]$}
\STATE Obtain $i$-th player's optimistic estimates $\sets{\Vb_{i}^{(t),\pi_i\times \pi^{t}_{-i}}}_{\pi_i\in\Pi_i}\setto \APE_i(\cF_i,\Pi_i,\pi_{-i}^t,K_i,\beta_i)$. \label{line:optimistic-value}
\STATE Update $\Lambda^{t+1}_i(\pi_i)  \propto_{\pi_i} \Lambda^{t}_i(\pi_i) \cdot \exp(\eta_i \cdot  \Vb_{i}^{(t),\pi_i\times \pi^{t}_{-i}})$.~\label{line:hedge}
\ENDFOR
\ENDFOR
\ENSURE Average policy $\wb{\Lambda}\defeq \frac{1}{T}\sum_{t\in[T]} \Lambda^t_1\times\dots\times\Lambda^t_m$.
\end{algorithmic}
\end{algorithm}

Our algorithm Decentralized Optimistic Policy Mirror Descent (\DOPMD{}, Algorithm~\ref{alg:opmd}) is a double-loop algorithm. Its outer loop is similar to the policy mirror descent algorithms of~\citep{liu2022learning,zhan2022decentralized}, where each player $i\in[m]$ maintains $\Lambda_i^t$---a distribution over polices in $\Pi_i$. The player then samples a policy $\pi_i^t\sim \Lambda_i^t$ (Line~\ref{line:sample-policy}), obtains optimistic value estimates 
(Line~\ref{line:optimistic-value}), and performs Mirror Descent/Hedge (Line~\ref{line:hedge}) in the policy space with these optimistic value estimates to obtain the update $\Lambda_i^{t+1}\in\Delta(\Pi_i)$. 

The key new ingredient in our algorithm is the subroutine \APE{} (Explorative All-Policy Evaluation; full description in Algorithm~\ref{alg:ftu}) for obtaining optimistic value estimates. For each player $i\in[m]$, subroutine $\APE_i(\cF_i,\Pi_i,\pi_{-i}^t,K_i,\beta_i)$ plays $K_i$ episodes and obtains accurate value estimations for \emph{all} $\pi_i\in\Pi_i$, in the MDP induced by the (fixed) opponent's policy $\pi_{-i}^t$. At a high level, \APE{} modifies the GOLF algorithm of~\citet{jin2021bellman} by playing the policy that maximizes the \emph{uncertainty}:
\begin{talign*}
    \pi_i^k \defeq \argmax_{\pi_i\in\Pi_i} 
    \set{ \max_{f:(f,\pi_i)\in\cB^k} f_1(s_1, \pi_{i,1}(s_1)) -
        \min_{f:(f,\pi_i)\in\cB^k} f_1(s_1, \pi_{i,1}(s_1)) },
\end{talign*}
specified by the square-loss confidence set $\cB^k$, instead of maximizing the optimistic value estimate as in GOLF.

\paragraph{Theoretical guarantee}
We are now ready to state the guarantee for the \DOPMD{} algorithm. The proof can be found in Appendix~\ref{app:proof-restricted-cce}.
\begin{theorem}[Guarantee for \DOPMD{}]
\label{thm:restricted-cce}
Under Assumption~\ref{asp:completeQ} \&~\ref{asp:bounded_eluder}, for any $\eps>0$, Algorithm~\ref{alg:opmd} with $\eta_i=\sqrt{\log\abs{\Pi_i}/(H^2T)}$, $K_i=\tO(H^4d_i\log(\sum_{i\in[m]} \abs{\Pi_i}\abs{\cF_i}/\eps^2)$,  $\beta_i=\tO(H^2\log(\sum_{i\in[m]}\abs{\Pi_i}\abs{\cF_i}))$ outputs an $\epsilon$-approximate $\Pi$-CCE within 
at most $T\le \tO(H^2\log(\sum_{i\in[m]} |\Pi_i| )/\eps^2)$ rounds. 

The total number of episodes played is at most
\begin{talign*}
    T \times \paren{\sum_{i\in[m]} K_i} = \tO\paren{ H^6\paren{\sum_{i\in[m]} d_i} \log^2(\sum_i \abs{\Pi_i}\abs{\cF_i})/ \eps^4 }.
\end{talign*}
where $\tO(\cdot)$ hides polylogarithmic factors in $H, d_i, \epsilon, \delta, \log\abs{\cF_i}, \log\abs{\Pi_i}, m$.
\end{theorem}

The sample complexity asserted in Theorem~\ref{thm:restricted-cce} for learning an $\eps$-approximate $\Pi$-CCE is polynomial in the (summation of the) BE dimensions, the log-cardinality of the function classes and policy classes, as well as $1/\eps$. While the $\Pi$-CCE guarantee is weaker than the \VLPR{} or \AVLPR{} algorithm (Theorem~\ref{thm:vlpr} \&~\ref{thm:avlpr}), in return, Theorem~\ref{thm:restricted-cce} only requires BE dimension and completeness assumptions, which are standard for general function approximation and potentially much more relaxed than Condition \ref{cond:G-condition} required in Section \ref{sec:vlpr}.

\paragraph{Decentralized execution} Note that the $i$-th player's \APE{} only uses their own marginal Q class $\cF_i$ and local observations for estimating the values for all $\pi_i\in\Pi_i$, and thus Algorithm~\ref{alg:opmd} can be executed in a decentralized fashion by letting each player execute \APE{} in lexicographic order in each round. As a result, neither communication nor shared randomness is required among players. 
This is different from the centralized algorithms of~\citet{liu2022learning,zhan2022decentralized} that operate with joint Q classes.

\subsection{Examples}

We first show that Assumption~\ref{asp:completeQ} \&~\ref{asp:bounded_eluder} hold for learning $\Pi$-CCE in linear quadratic games~\citep{zhang2019policy}---a special type of Markov Games with continuous states/actions and linear transitions---with linear policy classes and linear value classes.
\begin{example}[Linear quadratic games (LQGs)]
\label{example:lqg}
We consider $m$-player finite-horizon LQGs specified by a state space $\cS\subset \R^{d_S}$ and action spaces $\sets{\cA_i\subset \R^{d_{A,i}}}_{i\in[m]}$. The initial state $s_1\in\R^{d_s}$ is fixed, and the state transition at the $h$-th step is given by
\begin{talign}
\label{eqn:lqg_transition}
    s_{h+1}=A_hs_h+\sum_{i=1}^m B_{i,h}a_{i,h}+z_h,
\end{talign}
where $A_h\in\R^{d_S\times d_S}$, $B_{i,h}\in\R^{d_{S}\times d_{A,i}}$ are parameters of the game, and $z_h$ are independent mean-zero noises. The reward is given by $r_{i,h}(s,\ba)=s_h^\top K_h^i s_h + \sum_{j=1}^m a_{j,h}^\top K^i_{j,h} a_{j,h}$ for all $(i,h)\in[m]\times[H]$, where $K_h^i\in\R^{d_S\times d_S}$, $K_{j,h}^i\in\R^{d_{A,j}\times d_{A,j}}$ are parameters of the game.
\end{example}

An important policy class for LQGs is the class of \emph{linear policies} (denoted as $\Pi$) of the form $\pi_{i,h}(s)=M_{i,h}s$, which for instance contains the CCE of the game under standard assumptions~\citep{bacsar2008h}. In Appendix~\ref{app:lqg}, we show that such LQGs with properly chosen linear policy classes and linear value classes satisfy Assumption~\ref{asp:completeQ} and \ref{asp:bounded_eluder} with $d_i=\cO((d_s + d_{A,i})^2)$, and admits sample-efficient learning of a $\Pi$-CCE with $\tO({\rm poly}(H, \sum_{i\in[m]} d_i)/\eps^4)$ samples by \DOPMD{}.

By contrast, \VLPR{}/\AVLPR{} are unlikely to be instantiated on Example~\ref{example:lqg}---Condition~\ref{cond:optv} there typically requires $\Pi$-completeness of \emph{optimistic values} (i.e., linear function plus bonus); a sufficient condition is $\Pi$-completeness of all values at step $h+1$ as in Assumption~\ref{assumption:q-v-completeness}. Such optimistic values are no longer linear here and thus unlikely to be contained in our linear function class at step $h$.

Next and more generally, as Assumption~\ref{asp:bounded_eluder} only requires bounded Bellman-Eluder dimension (cf. Definition~\ref{def:bellman-eluder}) in a decentralized sense for each player, this contains rich subclasses such as low Eluder dimension or low Bellman rank for each player's induced \emph{marginal} MDPs, by similar arguments as~\citep[Proposition 11 \& 12]{jin2021bellman}.

\begin{example}[Low Eluder dimension]
Suppose that for all $i\in[m]$, $\cF_i$ has \emph{low Eluder dimension} (\citet{wang2020reinforcement}; cf. Definition~\ref{def:eluder}) in the sense that $\max_{h\in[H]}\dE(\cF_{i,h}, \eps)\le d_i\log(1/\eps)$, and satisfies $\Pi$-completeness (Assumption~\ref{asp:completeQ}). Then, Assumption~\ref{asp:bounded_eluder} also holds with the same $\sets{d_i}_{i\in[m]}$.
\end{example}

In particular, the class of functions with low Eluder dimension subsumes certain \emph{non-linear} function classes such as generalized linear models~\citep{russo2013eluder}
, which are of the form $\cF_{i,h} = \sets{Q_{i,h}(\cdot, \cdot) = \sigma(\phi_i(\cdot, \cdot)^\top \theta_{i,h}): \theta_{i,h}\in\R^{d_i}}$, where $\phi_i:\cS\times\cA_i\to\R^{d_i}$ is a feature map, and $\sigma:\R\to\R$ is a link function with $\sigma'(\cdot)\in[c_1,c_2]$ for some $0<c_1<c_2$.

\begin{example}[Low Bellman rank]
Suppose for all $i\in[m]$, the single-agent MDP induced by any $\pi_{-i}\in\Pi_{-i}$ has \emph{low Bellman rank}~\citep{jiang2017contextual} in the following sense: For any fixed $\pi_{-i}\in\Pi_{-i}$, there exist maps $\psi_{i,h}^{\pi_{-i}}:\Pi_i\to\R^{d_i}$, $\phi_{i,h}^{\pi_{-i}}:\cF_i\times\Pi_i\to\R^{d_i}$ such that for any $f\in\cF_i$, $\pi_i,\pi_i'\in\Pi_i$
\begin{talign*}
    \textstyle \E_{(s_h, a_{i,h})\sim\pi_i'\times\pi_{-i}}\brac{ (f_h - \cT_h^{\pi_i\times\pi_{-i}}f_{h+1})(s_h, a_{i,h}) } = \<\phi_{i,h}^{\pi_{-i}}(f, \pi_i), \psi_{i,h}^{\pi_{-i}}(\pi_i')\>.
\end{talign*}
Then, Assumption~\ref{asp:bounded_eluder} holds with the same $\sets{d_i}_{i\in[m]}$.
\end{example}

%% file: Sections/conclusion.tex
\section{Conclusion}
This paper provides the first line of results for provably efficient decentralized MARL under function approximation which avoids the curse of multiagency. We present two complementary approaches: The first one via policy replay and stage-wise no-regret learning, which we instantiate concretely in the linear and tabular setting and achieve a near-optimal $\tO(\eps^{-2})$ rate for learning an $\eps$-approximate CCE in both settings; The second one via policy mirror descent with decentralized exploration, which learns a  restricted version of CCE but applies to broader classes of problems. We believe our work opens up many interesting directions for future works, such as (1) deriving sharper sample complexities for both approaches,  in particular improving the $(d, \max_{i\in[m]} A_i)$ dependence for \AVLPR{} in the linear setting and the $S$ dependence in the tabular setting; (2) improving the computational efficiency for the policy mirror descent approach; and  (3) identifying new problem classes amenable to the policy replay approach.

%% file: Sections/tools.tex
\section{Technical tools}

\subsection{Concentration}

The following Freedman's inequality can be found in~\citep[Lemma 9]{agarwal2014taming}.
\begin{lemma}[Freedman's inequality]
  \label{lemma:freedman}
  Suppose random variables $\set{X_t}_{t=1}^T$ is a martingale difference sequence, i.e. $X_t\in\cF_t$ where $\set{\cF_t}_{t\ge 1}$ is a filtration, and $\E[X_t|\cF_{t-1}]=0$. Suppose $X_t\le R$ almost surely for some (non-random) $R>0$. Then for any $\lambda\in(0, 1/R]$, we have with probability at least $1-\delta$ that
  \begin{align*}
    \sum_{t=1}^T X_t \le \lambda \cdot \sum_{t=1}^T \E\brac{X_t^2 | \cF_{t-1}} + \frac{\log(1/\delta)}{\lambda}.
  \end{align*}
\end{lemma}

\subsection{Eluder \& Bellman-Eluder dimension}
\label{app:eluder}

We begin by presenting the standard definition of the Eluder dimension of a function class~\citep{russo2013eluder,wang2020reinforcement}. 
\begin{definition}[Eluder dimension]
\label{def:eluder}
For any function class $\cF\subset(\cX\to\R)$, its Eluder dimension $\dE(\cF,\eps)$ is defined as the length of the longest sequence $\sets{x_1,x_2,\dots,x_n}\subset \cD$ such that there exists $\eps'\ge \eps$ so that for all $i\in[n]$, $x_i$ is \emph{$\eps'$-independent} of its prefix sequence $\sets{x_1,\dots,x_{i-1}}$, in the sense that there exists some $f_i,g_i\in\cF$ such that
\begin{align*}
    \sqrt{\sum_{j=1}^{i-1} \brac{(f_i - g_i)(x_j)}^2 } \le \eps'~~~{\rm but}~~~\abs{ (f_i-g_i)(x_i) } \ge \eps'.
\end{align*}
\end{definition}

\begin{definition}[Distributional Eluder dimension]
\label{def:distributional-eluder}
For any function class $\cF\subset(\cX\to\R)$, its distributional Eluder dimension $\dE(\cF,\cD,\eps)$ with respect to a class of distributions $\Pi\subset \Delta(\cX)$ and $\eps>0$ is defined as the length of the longest sequence $\sets{\mu_1,\mu_2,\dots,\mu_n}\subset \cD$ such that there exists $\eps'\ge \eps$ so that for all $i\in[n]$, $\mu_i$ is \emph{$\eps'$-independent} of its prefix sequence $\sets{\mu_1,\dots,\mu_{i-1}}$, in the sense that there exists some $f_i\in\cF$ such that
\begin{align*}
    \sqrt{\sum_{j=1}^{i-1} \paren{\E_{X\sim \mu_j}\brac{f_i(X)}}^2 } \le \eps'~~~{\rm but}~~~\abs{ \E_{X\sim \mu_i}\brac{f_i(X)} } \ge \eps'.
\end{align*}
\end{definition}

For decentralized MARL, we consider the following definition of the Bellman-Eluder dimension, which is similar to the original definition of~\citet{jin2021bellman} applied to the single-agent MDPs for player $i$ when facing a fixed Markov opponent $\pi_{-i}$, except that here we consider Bellman operators with respect to all policies $\pi_i\in\Pi_i$ instead of the Bellman optimality operator.
\begin{definition}[Bellman-Eluder dimension]
\label{def:bellman-eluder}
    For any player $i\in[m]$, any Markov policy class $\Pi_i$ for the $i$-th player, any Markov policy $\pi_{-i}$ for all but the $i$-th player, and any $\eps>0$, define 
    \begin{align*}
        \dBEi \defeq \min_{\cD\in\sets{\cD_{\Pi_i\times\pi_{-i}}, \cD_\Delta}} \max_{h\in[H]} \dE\paren{\set{f_h - \cT_{i,h}^{\pi_{i}\times\pi_{-i}}f_{h+1} : (f, \pi_i)\in\cF\times\Pi_i}, \cD, \eps},
    \end{align*}
    where $d_E(\cdot,\cdot,\eps)$ denotes the distributional Eluder dimension (Definition~\ref{def:distributional-eluder}), and
    \begin{align*}
        & \cD_{\Pi_i\times\pi_{-i}} \defeq \set{d^{\pi_i\times\pi_{-i}}_{h}(\cdot,\cdot): \pi_i\in\Pi_i} \subset \Delta(\cS\times\cA_i), \\
        & \cD_{\Delta} \defeq \set{\delta_{(s,a_i)}: (s,a_i)\in\cS\times\cA_i}\subset \Delta(\cS\times\cA_i),
    \end{align*}
    where $d^{\pi_i\times\pi_{-i}}_h(\cdot,\cdot)\in\Delta(\cS\times\cA_i)$ denotes the distribution of $(s_h, a_{i,h})$ when playing policy $\pi_i\times\pi_{-i}$ in the game, and $\delta_{(s, a_i)}\in\Delta(\cS\times\cA_i)$ denotes the point mass at $(s,a_i)$.
\end{definition}

%% file: Sections/v_type.tex
\section{Discussions about V-type function approximation}
\label{app:v_type}

Our meta-algorithms \VLPR{} and \AVLPR{} and their guarantees can extend directly to V-type function approximation. Indeed, at their meta-algorithm level (Algorithm~\ref{alg:vlpr}-\ref{alg:abs-v}), \VLPR{} and \AVLPR{} do not strictly speaking require $\cF_i$ to be marginal Q classes---They directly apply as-is if $\sets{\cF_i}_{i\in[m]}$ are instead V classes, so long as the subroutines $\noreg$ and $\optreg$ can be designed Conditions~\ref{cond:cce-regret}-\ref{cond:pigeon} (and Condition~\ref{cond:switch}) can still be satisfied with some bonus functions $\sets{G_{i,h}}_{(i,h)\in[m]\times[H]}$. 

However, we remark that when instantiated concretely, V-type function approximation may encompass problems with fairly different structures from Q-type function approximation. For instance, imagine adapting the linear function approximation results in Section~\ref{sec:linear-vlpr} to linear V classes. A sensible choice of the V class would be $\cF_{i,h}\subset \set{f_{i,h}(\cdot)=\phi_i(\cdot)^\top\theta_{i,h}:\theta_{i,h}\in\R^{d_i}}$, where $\phi_i:\cS\to\R^{d_i}$ are feature maps for the state. In this case, a suitable choice of the policy class is linear policies of the form $\pi_{i,h}(\cdot|s)=\argmax_{a_i\in\cA_i} \phi_{i,h}(s)^\top \theta_{i,h}^{a_i}$ where $\sets{\theta_{i,h}^{a_i}}_{a_i\in\cA_i}\subset \R^{d_i}$ is a collection of vectors. However, such a policy class can be interpreted as requiring any action $a_i\in\cA_i$ to ``have the same meaning'' across all states, which could be rather unnatural compared with the Q-type feature map $\phi_i(s,a_i)$ which allows $a_i\in\cA_i$ to be a general action \emph{index} that could mean different things at different states.

%% file: Sections/proof_abstract_alg.tex
\section{Proofs and additional details for Section~\ref{sec:vlpr}}

\subsection{Proof of Theorem~\ref{thm:vlpr}}
\label{app:proof-vlpr}

By the Bellman optimality equation, we have that for all $(t,i,h,s)\in[T]\times[m]\times[H]\times\cS$
\begin{equation} \max_{\mu_{i,h}\in\Delta(\cA_i)} \D_{\mu_{i,h}\times \pi_{-i,h}^t}\left[ r_{i,h}+\P_{h+1}V_{i,h+1}^{\dagger,\pi^t_{-i}} \right](s)
= V_{i,h}^{\dagger,\pi^t_{-i}}(s). 
\end{equation}
On the other hand, by using Condition \ref{cond:cce-regret} and the first inequality in Condition \ref{cond:optv}, we have that with probability $1-2TH\delta$, for all $(t,i,h,s)\in[T]\times[m]\times[H]\times\cS$
\begin{equation}
 \begin{aligned} &\quad \max_{\mu_{i,h}\in\Delta(\cA_i)} \D_{\mu_{i,h}\times \pi_{-i,h}^t}\left[ r_{i,h}+\P_{h+1}\oV_{i,h+1}^{t} \right](s)\\
 &\le  \D_{ \pi_{h}^t}\left[ r_{i,h}+\P_{h+1}\oV_{i,h+1}^{t} \right](s)+ {G_{i,h}(s,\bar\pi^{t-1},t-1,\delta)}\le \oV_{i,h}^t(s). 
\end{aligned}   
\end{equation}
Therefore, by backward induction with the above two relations, we have  that for all $(t,i,h,s)\in[T]\times[m]\times[H]\times\cS$
\begin{equation}
   \oV_{i,h}^t(s)  \ge V_{i,h}^{\dagger,\pi^t_{-i}}(s). 
\end{equation}
Similarly, by backward induction with the second inequality in Condition \ref{cond:optv}, we can show that for all $(t,i,h,s)\in[T]\times[m]\times[H]\times\cS$
\begin{equation}\label{eq:meta-2}
    \oV_{i,h}^t(s)  \le V_{i,h}^{\pi^t}(s)+ 2\sum_{h'=h}^{H}\E_{\pi^t}\left[{G_{i,h'}(s_{h'},\bar\pi^{t-1},t-1,\delta)}\right].
\end{equation}
As a result, we can upper bound the CCE-regret by 
\begin{align*}
\sum_{t=1}^T [V_{i,1}^{\dagger,\pi^t_{-i}} - V_{i,1}^{\pi^t}] &\le
2\sum_{t=1}^T \sum_{h=1}^{H}\E_{\pi^t}\left[{G_{i,h}(s_{h},\bar\pi^{t-1},t-1,\delta)}\right]\le \tO\left(\sqrt{H^2LT}\right),
\end{align*}
where the final inequality follows from  Condition \ref{cond:pigeon}.

Finally the $\CCEGap$ of  the output policy $\pi^{\rm out}$ can be bounded with Markov's inequality and the choice of $T=\tO\left(H^2L/\epsilon^2\right)$.
\qed

\subsection{Accelerated algorithm}
\label{app:avlpr}

\begin{algorithm}[h]
\small
\caption{Accelerated V-Learning with Policy Replay (\avlpr{})}
\label{alg:AVLPR}
\begin{algorithmic}[1]
\STATE \textbf{Initialize} $\pi^1$ to be the uniform policy: $\pi^1_{i,h}(\cdot|s)\setto \Unif(\cA_i)$  for all $(i,s,h)$, $\cB^0_h\setto\emptyset$, $I_1\setto 0$.
\FOR{iteration $t=1,\ldots,T$}
\STATE Execute $\pi^t$ to sample an episode, and update $\cB^t_h=\cB^{t-1}_h\bigcup\{s_h\}$. \label{line:avlpr-run}
\IF{$\exists(i,h)\in[m]\times[H]$ s.t. $\Psi_{i,h}(\cB^t_h) \ge \Psi_{i,h}(\cB^{I_t}_h)+1$ or $t=1$} \label{line:avlpr-trigger}
\STATE Set replay policy $\bar\pi^t\setto \Unif(\set{\pi^{\tau}}_{\tau\in [t]})$ and $\oV_{i,H+1}^{t+1}\setto 0$.
\FOR{$h=H,\ldots,1$}
\STATE Compute $\pi^{t+1}_h\setto\cceaprx(\bar\pi^t,\{\oV_{i,h+1}^{t+1}\}_{i\in[m]},t)$.
\STATE Compute $\{\oV_{i,h}^{t+1}\}_{i\in[m]}\setto\vapprox(\bar\pi^t,\{\oV_{i,h+1}^{t+1}\}_{i\in[m]},\pi^{t+1}_h,\sets{\cF_i}_{i\in[m]},t)$.
\ENDFOR
\STATE set $I_{t+1}\setto t$
\ELSE
\STATE set $I_{t+1}\setto I_t$ and $\pi^{t+1}\setto \pi^t$
\ENDIF
\ENDFOR
\ENSURE $\pi^{\rm out}$ sampled uniformly at random from $\{\pi^t\}_{t\in[T]}$.
\end{algorithmic}
\end{algorithm}


\paragraph{Algorithm}
We present our accelerated algorithm \avlpr{} in Algorithm~\ref{alg:AVLPR}. The main new ingredient in \avlpr{} is an \emph{infrequent update} mechanism: The algorithm only performs the policy replay and learns a new policy $\pi^{t+1}$ if a certain \emph{triggering condition} (Line~\ref{line:avlpr-trigger}) is satisfied, in which case the learning procedure is the same as in \VLPR{}. Otherwise, it simply executes the current policy $\pi^t$ for one episode, adds the state $s_h$ into dataset $\cB_h^t$, and sets $\pi^{t+1}\setto \pi^t$ (Line~\ref{line:avlpr-run}).

Intuitively, the triggering condition requires that the dataset to have accumulated significantly since the last replay iteration $I_t<t$. This design is motivated by a \emph{doubling-trick} type of observation: The state visitation induced by $\pi^t$ (and thus the sample complexity) does not differ significantly regardless of whether $\pi^t$ are updated or not, until some summary statistic (for example the visitation count of any state in the tabular case) is found to have increased to at least two times (or any constant factor $>1$) since the last replay. We use $\Psi_{i,h}(\cdot)$ denote the logarithm of such a summary statistic, so that a new replay is triggered only if $\Psi_{i,h}(\cB_h^t)\ge \Psi_{i,h}(\cB_h^{I_t})+1$.

\paragraph{Condition and guarantee}
Concretely, \AVLPR{} requires the following additional condition to ensure the validity of the infrequent update mechanism, 
which intuitively requires the bonus function can increase at most by a constant factor between consecutive policy updates.

\begin{condition}[Validity of infrequent policy update] 
\label{cond:switch}
The triggering criterion $\sets{\Psi_{i,h}}_{(i,h)\in[m]\times[H]}$ in Algorithm~\ref{alg:AVLPR} satisfies the following:
\begin{enumerate}[label=(\alph*),topsep=0pt,itemsep=0pt]
\item With probability at least $1-\delta$, for all $(t,i,h)$, if $\Psi_{i,h}(\cB^t_h) \le \Psi_{i,h}(\cB_{h}^{I_t})+1$, then we must have $G_{i,h}(s,\bar\pi^{I_t},I_t,\delta) \le 8\times G_{i,h}(s,\bar\pi^t,t,\delta)$ for all $s\in\cS$;
\item The number of replays triggered (i.e. Line~\ref{line:avlpr-trigger}) in Algorithm~\ref{alg:AVLPR} within $T$ iterations is upper bounded by $\dreplay\log T$ iterations with probability one, for some constant $\dreplay>0$.
\end{enumerate}
\end{condition}

We now state our meta-guarantee for \AVLPR{}; the proof can be found in Appendix~\ref{app:proof-avlpr}.
\begin{theorem}[Meta-guarantee for AVLPR]
\label{thm:avlpr}
Suppose the subroutines in Algorithm~\ref{alg:AVLPR} can be instantiated such that Condition \ref{cond:cce-regret}-\ref{cond:pigeon} \& \ref{cond:switch} holds with the same bonus functions $\set{G_{i,h}}_{(i,h)\in[m]\times[H]}$ and the deployed triggering functions $\sets{\Psi_{i,h}}_{(i,h)\in[m]\times[H]}$. Then we have with probability at least $1-\delta$ that
\begin{talign*}
\CCEReg(T) \defeq \max_{i\in[m]} \sum_{t=1}^T \left[V_{i,1}^{\dagger,\pi^t_{-i}}(s_1) - V_{i,1}^{\pi^t}(s_1)\right]_+ \le \tO(\sqrt{H^2LT}).
\end{talign*}
As a corollary, choosing $T= \tO(H^2L/\eps^2)$ ensures that the output policy $\pi^{\out}$ of Algorithm~\ref{alg:vlpr} satisfies $\CCEGap(\pi^{\out})\le\eps$ further with probability at least\footnote{The success probability can be further boosted to any $1-\delta$ by a similar argument as in Theorem~\ref{thm:vlpr}.} $0.99$, and the total number of episodes played is at most (with $\Piexpup \defeq \max_{\bar\pi,\pi'}\abs{\Piexp(\bar\pi, \pi')}$)
\begin{talign*}
    \cO\paren{ T + HT\times \dreplay\log T \times \Piexpup } = \tO\paren{ H^3L\Piexpup \cdot \dreplay/ \eps^2 }.
\end{talign*}
\end{theorem}

\subsection{Proof of Theorem~\ref{thm:avlpr}}
\label{app:proof-avlpr}
Let $\cI$ denote the subset of $[T]$ where Line \ref{line:avlpr-trigger} is triggered. 
By Condition~\ref{cond:switch}, $|\cI|\le \dreplay\log T$.

By the Bellman optimality equation, we have that for all $(t,i,h,s)\in[T]\times[m]\times[H]\times\cS$
\begin{equation} \max_{\mu_{i,h}\in\Delta(\cA_i)} \D_{\mu_{i,h}\times \pi_{-i,h}^t}\left[ r_{i,h}+\P_{h+1}V_{i,h+1}^{\dagger,\pi^t_{-i}} \right](s)
= V_{i,h}^{\dagger,\pi^t_{-i}}(s). 
\end{equation}
On the other hand, by using Condition \ref{cond:cce-regret} and the first inequality in Condition \ref{cond:optv}, we have that with probability $1-2TH\delta$, for all $(t,i,h,s)\in \cI\times[m]\times[H]\times\cS$
\begin{equation}
 \begin{aligned} &\quad \max_{\mu_{i,h}\in\Delta(\cA_i)} \D_{\mu_{i,h}\times \pi_{-i,h}^{t+1}}\left[ r_{i,h}+\P_{h+1}\oV_{i,h+1}^{t+1} \right](s)\\
 &\le  \D_{ \pi_{h}^{t+1}}\left[ r_{i,h}+\P_{h+1}\oV_{i,h+1}^{t+1} \right](s)+ {G_{i,h}(s,\bar\pi^{t},t,\delta)}\le \oV_{i,h}^{t+1}(s). 
\end{aligned}   
\end{equation}

Therefore, by backward induction with the above two relations, we have  that for all $(t,i,h,s)\in\cI\times[m]\times[H]\times\cS$
\begin{equation}
   \oV_{i,h}^{t+1}(s)  \ge V_{i,h}^{\dagger,\pi^{t+1}_{-i}}(s),
\end{equation}
which implies that for all for all $(t,i,h,s)\in[T]\times[m]\times[H]\times\cS$: 
$$
\oV_{i,h}^{I_t+1}(s)  \ge V_{i,h}^{\dagger,\pi^{I_t+1}_{-i}}(s).
$$
Similarly, by backward induction with the second inequality in Condition \ref{cond:optv}, we can show that for all $(t,i,h,s)\in\cI\times[m]\times[H]\times\cS$
\begin{align*}
    \oV_{i,h}^{t+1}(s)  \le V_{i,h}^{\pi^{t+1}}(s)+ 2\sum_{h'=h}^{H}\E_{\pi^{t+1}}\left[{G_{i,h'}(s_{h'},\bar\pi^{t},t,\delta)}\right],
\end{align*}
which implies that for all $(t,i,h,s)\in[T]\times[m]\times[H]\times\cS$:
$$
    \oV_{i,h}^{I_t+1}(s)  \le V_{i,h}^{\pi^{I_t+1}}(s)+ 2\sum_{h'=h}^{H}\E_{\pi^{I_t+1}}\left[{G_{i,h'}(s_{h'},\bar\pi^{I_t},I_t,\delta)}\right].
$$
As a result, we can upper bound the CCE-regret by 
\begin{align*}
\sum_{t=1}^T [V_{i,1}^{\dagger,\pi^t_{-i}} - V_{i,1}^{\pi^t}] &  \overset{(i)}{=}  \sum_{t=1}^T [V_{i,1}^{\dagger,\pi^{I_t+1}_{-i}} - V_{i,1}^{\pi^{I_t+1}}]  \\
&\le 
2\sum_{t=1}^T \sum_{h=1}^{H}\E_{\pi^{I_t+1}}\left[{G_{i,h}(s_{h},\bar\pi^{I_t},I_t,\delta)}\right]\\
&\overset{(ii)}{\le}
16\sum_{t=1}^T \sum_{h=1}^{H}\E_{\pi^{I_t+1}}\left[{G_{i,h}(s_{h},\bar\pi^{t-1},t-1,\delta)}\right]\\
&\overset{(iii)}{=}
16\sum_{t=1}^T \sum_{h=1}^{H}\E_{\pi^{t}}\left[{G_{i,h}(s_{h},\bar\pi^{t-1},t-1,\delta)}\right] \\
& \overset{(iv)}{\le} \tO\left(\sqrt{H^2LT}\right),
\end{align*}
where (i) and (iii) uses the fact that $\pi^{t}=\pi^{I_t+1}$, (iv) follows from  Condition \ref{cond:pigeon}, and (ii) follows from Lemma \ref{lem:trigger-bonus}.

Finally the $\CCEGap$ of  the output policy $\pi^{\rm out}$ can be bounded with Markov's inequality and the choice of $T=\tO\left(H^2L/\epsilon^2\right)$. The total sample complexity would be bounded by
\begin{align*}
T+|\cI|\times H \times\left(\underbrace{\cO\left(\bar\Gamma \cdot T\right)}_{\text{cost of \cceaprx}} + 
\underbrace{\cO\left(\bar\Gamma \cdot T\right)}_{\text{cost of \vapprox}}
\right) = \tO\left(H^3L\bar\Gamma \dreplay/\epsilon^2\right).
\end{align*}

\begin{lemma}\label{lem:trigger-bonus}
Suppose Condition \ref{cond:switch} holds, then with probability at least $1-\delta$, 
for all $(t,i,h)$,  $G_{i,h}(s,\bar\pi^{I_t},I_t,\delta) \le 8\times {G_{i,h}(s_{h},\bar\pi^{t-1},t-1,\delta)}$ for all $s\in\cS$.
\end{lemma}
\begin{proof}
If Line \ref{line:avlpr-trigger} is triggered in the $(t-1)\th$ iteration, then $I_t=t-1$ and the result holds. 
Otherwise, $I_t=I_{t-1}$ and for all $(i,h)$, $\Psi_{i,h}(\cB^{t-1}_h) \le \Psi_{i,h}(\cB_{h}^{I_{t-1}})+1$, which, by Condition \ref{cond:switch}, implies
$$
G_{i,h}(s,\bar\pi^{I_t},I_t,\delta) 
=
G_{i,h}(s,\bar\pi^{I_{t-1}},I_{t-1},\delta)
\le 8\times {G_{i,h}(s_{h},\bar\pi^{t-1},t-1,\delta)} 
$$
for all $s\in\cS$.
\end{proof}

\subsection{Decentralized execution protocol for \VLPR{} and \AVLPR{}}
\label{app:decentralized-execution}

In this section, we first describe our protocols, then argue that both \VLPR{} and \AVLPR{} can be made decentralized (with minimal communication for \AVLPR{}) under these protocols.

We consider the following protocol: Before the game starts, the players sample a sequence of random bits with length polynomial in the number of episodes played, and all players can observe this (shared) sequence of random bits. Using this sequence, the players can then implement shared randomness in a decentralized fashion. For example, executing $\Unif(\set{\pi^\tau}_{\tau\in[T]})$ where each $\pi^\tau$ is a product policy can be done by using the shared random bits (with the same pre-determined protocol) to sample a shared $\tau\sim \Unif(T)$, then executing $\pi^\tau=\pi_1^\tau\times\dots\times\pi_m^\tau$, which can be done in a fully decentralized fashion.

We further assume that exploration policy mapping $\Piexp(\wb\pi, \pi')$ (which we recall is an ordered set of tuples $(\wt{\pi}, P)$) is \emph{marginally executable} in the following sense: The ordering is known to all the players, and for each $(\wt{\pi}, P)\in\Piexp(\wb\pi, \pi')$ in an ordered fashion, $P$ is known to all players, and the marginal policy $\wt{\pi}_i$ (conditioning on the shared random bits) is known to the $i$-th player as long as the marginal policies $\wb\pi_i$ and $\pi'_i$ (conditioning on the shared random bits) are known to the $i$-th player. 

We remark that this assumption is satisfied with typical choices of $\Piexp$, such as our instantiations in both the tabular case and the linear case. In particular, our tabular setting chooses $\Piexp(\wb\pi, \pi')= \brac{(\wb\pi_{1:h-1}\circ\pi'_{h}, [m])}$, which directly satisfies marginal executability. For the linear setting, recall by~\eqref{eqn:pi-explore-linear} that we have chosen 
\begin{align*}
    \Piexp(\wb\pi, \pi') = \brac{ (\wt{\pi}, P)=\paren{ \wb\pi_{1:h-1} \times (\Unif(\cA_i) \times \pi'_{-i,h}), \sets{i} } }_{i=1}^m.
\end{align*}
It is straightforward to let all players know and abide by the schedule of the $P$ (just round-robin over $\sets{i}$ for $i\in[m]$ in lexicographic order). Further the marginal policy $\wt{\pi}_i$ of each $\wt{\pi}$ in this list is fully determined by $\wb\pi_i$ and one of $\sets{\Unif(\cA_i),\pi'_i}$ (depending on whether $i\in P$), which verifies the marginal executability assumption.

\paragraph{VLPR}
Observe that for the \VLPR{} algorithm described in Algorithm~\ref{alg:vlpr}-\ref{alg:abs-v}, most of the steps (such as $\noreg$ and $\optreg$) are by nature decentralized and can be executed by each player independently. The only coordinations involved are executing either the replay policy $\wb\pi^t=\Unif(\set{\pi^\tau}_{\tau\in[T]})$ (Line~\ref{line:collect-dinit} in Algorithm~\ref{alg:abs-cce}), or the exploration policies $(\wt{\pi},P)\in\Piexp(\wb\pi^t, \mu^{k}_h)$ within Algorithm~\ref{alg:abs-cce} and $(\wt{\pi},P)\in\Piexp(\wb\pi^t, \pi^{t+1}_h)$ within Algorithm~\ref{alg:abs-v}. Executing $\wb\pi^t$ can be done by using the shared randomness described above. Further, as both $\mu^k_{i,h}$ and $\pi^{t+1}_{i,h}$ are known to the $i$-th player, and by the marginal executability assumption, all the exploration policies can be executed in a decentralized fashion. This verified the claim for \VLPR{}.

\paragraph{AVLPR}
The only difference in \AVLPR{} over \VLPR{} is to check the triggering condition in Line~\ref{line:avlpr-trigger} of Algorithm~\ref{alg:AVLPR}, which in each iteration $t\in[T]$ requires one communication of $m$ bits, one for each player (indicator of whether the condition holds for player $i\in[m]$). The players will enter the replay part if the triggering condition holds for at least one player, and start the next episode otherwise. Since all players know whether they have entered the replay part in each iteration, the replay index $I_t$ is a common knowledge that can be maintained by all players simultaneously. Further, we can let this communication can be triggered only when the triggering condition holds, which by Condition~\ref{cond:switch} happens for at most $\tO(\dreplay \log T)$ times within $T$ iterations of play.

%% file: Sections/linear_properties.tex
\section{All-policy completeness implies ``essentially tabular'' games}
\label{app:essentially-tabular}

Here we argue that the restriction to linear argmax policies in Assumption~\ref{assumption:q-v-completeness} (or some other kind of restriction) is necessary, by showing that the unrestricted \emph{all-policy completeness} assumption places a strong implicit requirement on the game.

Consider the following all-policy completeness assumption for decentralized linear function approximation, which strengthens Assumption~\ref{assumption:q-v-completeness} by removing the $\Pi^{\rm lin}$ restriction.
\begin{assumption}[All-policy completeness]
\label{assumption:all-policy-completeness}
For any $(i,h)\in[m]\times[H]$, any function $\oV=\oV_{i,h+1}:\cS\to [0,H]$ and \emph{any} policy $\pi$, there exists $\theta^{h,\pi_{-i},\oV}\in\R^{d}$ with $\ltwos{\theta^{h,\pi_{-i},\oV}}\le B_\theta$ such that 
\begin{align}
\label{eqn:all-policy-completeness}
\DD_{\delta_{a_i}\times \pi_{-i}}\brac{r_{i,h} + \P_h\oV_{i,h+1}}(s) = \phi_i(s,a_i)^\top \theta^{h,\pi_{-i},\oV}~~~\textrm{for all}~(s,a_i)\in\cS\times\cA_i.
\end{align}
\end{assumption}

Fix any $(h,s^\star)\in [H]\times\cS$, fix any player $i\in[m]$ and $s'\in\cS$. Let $\pi^1$ and $\pi^2$ be any two joint policies that are different only at $(h,s^\star)$. By applying Assumption~\ref{assumption:all-policy-completeness} with zero reward (\emph{i.e.} $r_{i,h}=0$) and function $\oV(\cdot)=\indics{s'=\cdot}$, there exists $\theta^{\pi^{\sets{1,2}}}$ such that for all $(s,a_i)\in\cS\times\cA_i$,
\begin{align*}
\phi_i(s,a_i)^\top \theta^{\pi^1} &= \E_{a_{-i}\sim\pi^1_{-i}(\cdot|a)}\Pr[s'|s,a_i,a_{-i}], \\
\phi_i(s,a_i)^\top \theta^{\pi^2} &= \E_{a_{-i}\sim\pi^2_{-i}(\cdot|a)}\Pr[s'|s,a_i,a_{-i}].
\end{align*}
As $\pi^1=\pi^2$ at $(h,s)$ with any $s\neq s^\star$, we have for every $s\neq s^\star\in \cS$ and $a_i\in \cA_i$ that
\begin{align*}
    \phi_i(s,a_i)^\top \left(\theta^{\pi^1}-\theta^{\pi^2}\right)=0,
\end{align*}
and for $s=s^\star$ that
\begin{align*}
    \phi_i(s^\star,a_i)^\top \left(\theta^{\pi^1}-\theta^{\pi^2}\right)= \E_{a_{-i}\sim\pi^1_{-i}(\cdot|a)}\Pr[s'|s^\star,a_i,a_{-i}]-\E_{a_{-i}\sim\pi^2_{-i}(\cdot|a)}\Pr[s'|s^\star,a_i,a_{-i}].
\end{align*}
We say that a state $(s,h)\in \cS\times [H]$ is \emph{irrelevant} if the transition of this state can be affected by the action of some players. If a state is $s^\star$ \emph{relevant}, by definition $\exists i\in [m], s'\in\cS, a_i\in \cA_i,$ and $\pi^{1}, \pi^{2}$ such that
\[
\phi_i(s^\star,a_i)^\top \left(\theta^{\pi^1}-\theta^{\pi^2}\right)= \E_{a_{-i}\sim\pi^1_{-i}(\cdot|a)}\Pr[s'|s^\star,a_i,a_{-i}]-\E_{a_{-i}\sim\pi^2_{-i}(\cdot|a)}\Pr[s'|s^\star,a_i,a_{-i}]\neq 0.
\]
It follows that  (1) $v\defeq \theta^{\pi^1}- \theta^{\pi^2}\neq 0$; (2) $v$ is orthogonal to $\phi_i(s,a'_i)$ for all other $s\neq s^\star$ and $a'_i\in\cA_i$; (3) $\phi_i(s^\star, a_i)$ is not orthogonal to $v$, and thus linearly independent from $\sets{\phi_i(s, a_i')}_{s\neq s^\star,a_i'\in\cA_i}$.
Since the features $\phi_i(s,a_i)\in\R^d$, there could be at most $d$ such feature vectors that are linearly independent from everyone else, and therefore there are at most $d$ relevant states for player $i$. 

It follows that except for at most $dm$ states, all other states are irrelevant: the transition probabilities at such states are not a function of the players' joint action. If we simply omit such states (and play an arbitrary policy when visiting such states) from the trajectory, the resulting dynamics would be a Markov game dynamics over a small (at most $dm$) number of states. In this sense such a Markov game would be ``essentially tabular''.

\subsection{Explicit forms of the policy class in \citet{cui2023breaking}}
\label{app:cui}

If the no-regret-learning oracle in  \citet{cui2023breaking} is chosen as the Exponential Weights algorithm, then it will induce a policy class of the following form: 
$\Pi^{\rm estimate}=\Pi^{\rm estimate}_1\times\cdots\times\Pi^{\rm estimate}_m$ with
\begin{align*}
\Pi^{\rm estimate}_i:=\left\{
\pi_i(\cdot \mid s)\propto \exp\left(\eta\sum_{i=1}^K \left[\phi_{i}(s,\cdot)\trans\theta^k + \beta\|\phi_{i}(s,\cdot)\|_{\Sigma^{-1}}\right]_{[0,H]}  \right):~ \theta^k\in\R^d,~\Sigma\in\R^{d\times d},~\Sigma \succeq \lambda I 
\right\}, 
\end{align*}
where $[\cdot]_{[0,H]}$ denotes a truncation operator s.t. $[x]_{[0,H]}=\min\{\max\{x,0\},H\}$ and $\eta,\lambda,\beta$ are some tunable parameters in their algorithm. Note that linear argmax policies can be parameterized by a single $d$-dimension vector, while policies in above class are specified by a much larger number of parameters ($K$ different $d$-dimension vectors, a $d\times d$ matrix, and a few additional scalars) and involve $K$ truncations that make the exponents potentially highly nonlinear. In this sense, the above policy class is more complex than the linear argmax policy class $\Pi^{\rm lin}$ considered in this paper. We further note that $\Pi^{\rm estimate}_i$ reduces to $\Pi^{\rm lin}_i$ if we remove the truncation operator, choose $\beta=0$ and let $\eta$ go to infinity in the above definition.

If the no-regret-learning oracle is instead chosen as Expected Follow the Perturbed Leader, then we will have 
$\Pi^{\rm estimate}=\Pi^{\rm estimate}_1\times\cdots\times\Pi^{\rm estimate}_m$ with
\begin{align*}
\Pi^{\rm estimate}_i:=\bigg\{
\pi_i(a_i\mid s) &= \P_{v\sim D_i} \left[a_i\in\argmax_{\hat a_i} \left(\sum_{i=1}^K \left[\phi_{i}(s,\hat a_i)\trans\theta^k + \beta\|\phi_{i}(s,\hat a_i)\|_{\Sigma^{-1}}\right]_{[0,H]} +\eta^{-1}v_{\hat a_i}\right) \right]\\
 &:~ \theta^k\in\R^d,~\Sigma\in\R^{d\times d},~\Sigma \succeq \lambda I 
\bigg\}, 
\end{align*}
where vector $v\in\R^{A_i}$ is sampled from some distribution $D_i$ over $\R^{A_i}$ and $v_{\hat a_i}$ denotes the  
$\hat a_i$-th coordinate of $v$. 
Similar to the argument above, this $\Pi^{\rm estimate}_i$ is still more involved than $\Pi^{\rm lin}_i$. It can again be reduced to $\Pi^{\rm lin}_i$ by removing the truncation operator, choosing $\beta=0$ and picking $D=\delta_{\overset{\rightarrow}{0}}$: the Dirac distribution at point $\overset{\rightarrow}{0}$.

%% file: Sections/proof_linearVLPR.tex
\section{Proofs for Section \ref{sec:linear-vlpr}}
\label{app:proof-linear}

\subsection{Details of the linear \avlpr~Algorithm}
\label{app:details-linear-avlpr}

\paragraph{Understanding Assumption~\ref{assumption:q-v-completeness}} Assumption~\ref{assumption:q-v-completeness} has the following implication, which is used throughout the design and analysis of the linear function approximation case.

\begin{remark}
Assumption~\ref{assumption:q-v-completeness} implies the following statement. For any $(i,h)\in[m]\times[H]$, any function $\oV=\oV_{i,h+1}:\cS\to [0,H]$ and any policy $\pi\in \Pi^{\lin}$, there exists $\theta^{h,\pi_{-i},\oV}\in\R^{d}$ such that 
\begin{align}
\DD_{\delta_{a_i}\times \pi_{-i,h}}\brac{r_{i,h} + \P_h\oV_{i,h+1}}(s) = \phi_i(s,a_i)^\top \theta^{h,\pi_{-i},\oV}~~~\textrm{for all}~(s,a_i)\in\cS\times\cA_i.
\end{align}
\end{remark}
This can be seen by picking $f_{i,h+1}(s,a_i)=\oV(s)$ and applying Assumption~\ref{assumption:q-v-completeness}.

\paragraph{Choice of $\Psi_{i,h}$} The switching condition in Algorithm~\ref{alg:AVLPR} is chosen as $$\Psi_{i,h}(\cB):=\log\det\left(I+\frac{1}{A_i}\sum_{s\in\cB}\sum_{a_i\in\cA_i}\phi_i(s,a_i)\phi_i(s,a_i)^\top\right).$$

\paragraph{Processing $\cD_{\rm init}$ and $\cD_{\rm sample}$} For linear function approximation, the dataset $\cD_{\rm init}^i$ will then be used to compute $H$ feature covariance matrices $\{\hat\Sigma^{\bar\pi}_{i,h}\}_{h\in [H]}$ that measures the coverage of the exploration policy $\bar{\pi}$, defined as
\[\hat\Sigma^{\bar \pi}_{i,h}:=\frac{1}{|\cD_{\rm init}^i|\cdot A_i}\sum_{s_h\in \cD_{\rm init}^i}\sum_{a_i\in\cA_i} \phi_i(s_h,a_i)\phi_i(s_h,a_i)^\top.\]
Additionally we define the population version
\begin{equation*}
\Sigma^{\bar\pi}_{i,h} = \E_{s_h\sim \bar\pi}~\E_{a\sim \Unif(\cA_i)}\left[\phi_i(s_h,a_i)\phi_i(s_h,a_i)^\top\right].
\end{equation*}

For linear function approximation we choose the exploration scheme $\Gamma_{\rm explore}(\bar\pi, \mu_h)$ in Algorithm~\ref{alg:abs-cce} and~\ref{alg:abs-v} as the ordered set
\begin{align}
\label{eqn:pi-explore-linear}
\brac{ \left(\bar\pi_{1:h-1}\circ (\Unif(\cA_1)\times \mu_{-1,h}), \{1\}\right),\cdots,  \left(\bar\pi_{1:h-1}\circ (\Unif(\cA_m)\times \mu_{-m,h}), \{m\}\right)  }.
\end{align}
As a result, $\cD^{k,i}_{\rm sample}$ will contain exactly one element, denoted as $(s_h^k, a_{i,h}^k, y_{i,h}^k)$.

\paragraph{\noreg} Condition~\ref{cond:cce-regret} can be understood as a state-wise regret bound with respect to the loss function $\ell^k_{i,h}(s,a_i):=\D_{\delta_{a_i}\times \mu^k_{-i,h}}[r_{i,h}+\P_{h+1}\oV_{i,h+1}](s)$. As per Assumption~\ref{assumption:q-v-completeness}, $\ell^k_{i,h}$ can be written as a linear function
$\ell^k_{i,h}(s,a_i) = \langle \theta^k_{i,h}, \phi_i(s,a_i)\rangle$. In order to guarantee a state-wise regret, we first construct a linear estimator of $\ell^k_{i,h}$ for all $(s,a_i)\in \cS\times \cA_i$:
\[\hat \ell^k_{i,h}(s,a_i)= \langle\hat\theta^k_{i,h},\phi_i(s,a_i)\rangle,\]
where
\[
\hat\theta^k_{i,h}:=\left(\hat\Sigma_{i,h}^{\bar\pi}+\lambda I\right)^{-1}\phi_i(s_{i,h}^k,a_{i,h}^k)y_{i,h}^k.
\]
This estimator is also used in adversarial linear bandits~\citep{neu2020efficient}. However, directly running an exponential weights algorithm with this estimator would not work in our setting because Assumption~\ref{assumption:q-v-completeness} requires $\mu^k_h$ to lie in (the convex hull of) $\Pi^{\rm lin}$; otherwise under $\mu^{k}_{-i}$ the resulting action-value function cannot be approximated with a linear function. To that end, we first make the observation that the per-state bandit regret (with the comparator in $\Delta(\cA_i)$) can be equivalently viewed as the regret of an online linear optimization problem (with the comparator in the convex hull of the action feature vectors)
\begin{align*}
    & \max_{\mu_{i,h}\in\Delta_{\cA_i}} \sum_{k=1}^K\left(\D_{\mu_{i,h}\times\mu^k_{-i,h}}-\D_{\mu^k_h}\right)(r_{i,h}+\P_h \overline{V}_{i,h+1})(s)\\
= & \max_{\mu_{i,h}\in\Delta_{\cA_i}} 
 \sum_{k=1}^K \langle \mu_{i,h} -\mu_{i,h}^k(\cdot \mid s), \ell_{i,h}^k(s,\cdot)\rangle= \max_{\phi\in CH(\Phi_i(s))} \sum_{k=1}^K \langle \phi - \Phi_i(s)^\top \mu^k_{i,h}, \theta^k_{i,h}\rangle.
\end{align*}
Here $\Phi_i(s)\in\R^{A_i\times d}$ is a matrix that stacks all feature vectors $\{\phi_i(s,\cdot)\}$, while we slightly abuse notation to use $CH(\cdot)$ to denote the convex hull of the rows of the matrix.

We will then apply the Expected Follow-the-Perturbed-Leader algorithm (\citet[Algorithm 3]{hazan2020faster}; see also~\citet{hazan2016introduction}) to the online linear optimization problem, namely choosing
\begin{equation*}
\Phi_i(s)^\top \mu^k_{i,h} = \E_{v\sim \cV}\left[\argmax_{\phi\in CH(\Phi_i(s))}\langle \phi, \sum_{k'\le k} \hat\theta^{k'}_{i,h} + v/\eta \rangle\right],
\end{equation*}
where $\cV$ is chosen as the uniform distribution over the ellipse $\{u\mid u^\top (\Sighat+\lambda I)u\le 1\}$, and $\eta$ is a parameter that plays a role similar to learning rate. This induces the following policy
\begin{equation}
\label{equ:efpl-policy}
    \mu^{k+1}_{i,h}(a_i|s):=\Pr_{v\sim \cV}\left[a_i=\argmax_{a_i'\in\cA_i}\langle \phi_i(s,a_i'), \sum_{k'\le k}\hat\theta^{k'}_{i,h} + v/\eta\rangle\right],
\end{equation}
which lies in the convex hull of $\Pi^{\rm lin}$ and therefore satisfies the requirement of Assumption~\ref{assumption:q-v-completeness}.

\paragraph{\optreg} The optimistic regression is implemented using ridge regression on the dataset $\cD_{\rm reg}^i$, which contains samples of $(s_h,a_{i,h},y_{i,h})$ where $y_{i,h}=r_{i,h}+\oV_{i,h+1}(s_{h+1})$. More specifically,
\begin{align*}
\hat\theta_{i,h}&\gets \arg\min_{\theta} \frac{1}{K}\sum_{(s_h,a_{i,h},y_{i,h})\in\cD_{\rm reg}^i}\left[\phi_i(s_h, a_{i,h})^\top\theta- y_{i,h}\right]^2 + \lambda\|\theta\|_2^2,
\\
\oQ_{i, h}(s,a_i)&\gets
\paren{\phi_i(s,a_i)^\top\hat\theta_{i,h}+\frac{3}{2}G_{i,h}(s,\bar\pi,K,\delta)}\land(H-h+1),\\
   \oV_{i,h}(s)&\gets \langle \pi_{i,h}(\cdot\mid s), \oQ_{i,h}(s,\cdot) \rangle.
\end{align*}

\paragraph{Computational efficiency}
We remark here that  $\oV_{i,h}(s)$ does not need to be computed for every $s$ but only for states in the dataset, which can be done in polynomial time. Also, the policy in (\ref{equ:efpl-policy}) does not need to be fully computed either, because executing the algorithm only requires an efficient sampling from the policy $\mu^{k+1}_{i,h}$, which can in turn easily achieved by sampling $v\sim\cV$.

\subsection{Proof of Condition \ref{cond:cce-regret}}\label{app:linear-cond-1}

As outlined in Section~\ref{sec:linear-vlpr}, we will first decompose the 
 the per-state regret in Condition~\ref{cond:cce-regret} as the per-state regret measured on the loss estimator and statistical error terms:
\begin{align*}
& \max_{\mu_{i,h}\in\Delta_{\cA_i}} \sum_{k=1}^K\left(\D_{\mu_{i,h}\times\mu^k_{-i,h}}-\D_{\mu^k_h}\right)(r_{i,h}+\P_h \overline{V}_{i,h+1})(s)\\
 = & \underbrace{\max_{\phi\in CH(\Phi_i(s))} 
\sum_{k=1}^K \langle \phi -\Phi_i(s)\mu_{i,h}^k(\cdot \mid s), \hat \theta^k_{i,h} \rangle}_{(A)} + \underbrace{\max_{\phi\in CH(\Phi_i(s))} \sum_{k=1}^K \langle \phi ,  \theta^k_{i,h}-\hat\theta^k_{i,h}\rangle}_{(B)}\\&
+\underbrace{ \sum_{k=1}^K \langle   \Phi_i(s)\mu_{i,h}^k(\cdot\mid s),  \hat \theta^k_{i,h}-\theta^k_{i,h}\rangle}_{(C)}.
 \end{align*}

In Appendix \ref{app:proof-linear-cond-1}, we prove that under the choice of $\eta=1/(dH\sqrt{K\log\delta^{-1}})$ and $\lambda=\tilde{\Theta}(d\max_i A_i/K)$, the above three terms can be respectively controlled as following: with probability at least $1-\delta$, for all $s\in\cS$
\begin{align*} 
&\text{ Term (A)} \le  \sup_{a_i\in\cA_i}\Vert \phi_i(s,a_i) \Vert_{(\Sig+\lambda I)^{-1}}\cdot \tO\left(dH\sqrt{K(\max_i A_i)}\right),\\
&\text{Term (B)} \le \sup_{a_i\in\cA_i}\|\phi_i(s,a_i)\|_{(\Sigma^{\bar\pi}_{i,h}+\lambda I)^{-1}}\times \tO\left(dH\sqrt{K(\max_i A_i)} \right),\\
&\text{Term (C)} \le \sup_{a_i\in\cA_i}\|\phi_i(s,a_i)\|_{(\Sigma^{\bar\pi}_{i,h}+\lambda I)^{-1}}\times \tO\left(dH\sqrt{K(\max_i A_i)^3} \right)+\cO(1).
\end{align*}
As a result, we can pick
\[
G_{i,h}(s,\bar\pi,K,\delta) =\sup_{a_i\in\cA_i}\|\phi_i(s,a_i)\|_{(\Sigma^{\bar\pi}_{i,h}+\lambda I)^{-1}}\times \tilde{\Theta}\left(\frac{dHA_i^{1.5}}{\sqrt{K}} \right)+\Theta\paren{\frac{1}{K}}.
\]

\subsection{Proof of Condition \ref{cond:optv}}
\label{app:linear-cond-3}

Consider a fixed $(i,h)\in [m]\times [H]$. Denote $\cD_{\regress}^i$ as
\[
\left\{ (s_h^j,a_{i,h}^j,y_{i,h}^j)\right\}_{j\in[K]}.
\]
By Assumption~\ref{assumption:q-v-completeness}, there exists $\theta^*_{i,h}$ such that for all $j$,
\begin{align*}
    \E[y_{i,h}^j|s_h^j,a_{i,h}^j] 
    =\D_{\delta_{a_{i,h}^j}\times \pi_{-i,h}}[r_{i,h}+\P_{h+1}\oV_{i,h+1}](s_h^j)
    = \langle \phi_i(s_h^j,a_{i,h}^j), \theta^*_{i,h}\rangle.
\end{align*}
Define 
$\hat\Sigma_{\regress,h}=\frac{1}{K}\sum_{j=1}^{K}\phi_i(s_h^j,a_{i,h}^j)\phi_i(s_h^j,a_{i,h}^j)^\top+\lambda I$ and $\zeta_{j}=y_{i,h}^j-\P_h\left[(\oV_{i, h+1}+r_{i,h})\right]\left(s_h^j,a_{i,h}^j\right)$. Here $\zeta_j$ is mean-zero and $H$-bounded. It follows that $\forall (s,a_{i})\in \cS\times \cA_i$
\begin{align*}
   &  \left|    \phi_i(s,a_i)^\top\hat\theta_{i,h}
- \D_{\delta_{a_{i}}\times \pi_{-i,h}}[r_{i,h}+\P_{h+1}\oV_{i,h+1}](s) \right| \\
= &   \left|    \phi_i(s,a_i)^\top\hat\theta_{i,h}
- \phi_i(s,a_i)^\top\thetas_{i,h}\right| \\ 
= & \left|    \phi_i(s,a_i)^\top\hat\Sigma_{\regress,h}^{-1}
\frac{1}{K}\sum_{j=1}^{K} \phi_i(s_h^j,a_{i,h}^j)\left(\phi_i(s_h^j,a_{i,h}^j)^\top\thetas_{i,h}+\zeta_j\right)
-  \phi_i(s,a_{i,h})^\top\thetas_{i,h}\right| \\
\le  & \|\phi_i(s,a_{i,h})\|_{\hat\Sigma_{\regress,h}^{-1}} \times \left( \left\|\frac{1}{K}\sum_{j=1}^{K} \phi_i(s_h^j,a_{i,h}^j) \zeta_j\right\|_{\hat\Sigma_{\regress,h}^{-1}}+\sqrt{\lambda}B_\theta  \right).
 \end{align*}
 
\begin{lemma}
Suppose we pick $\lambda=\Theta(d\log(dK/\delta)/K)$,  then
with probability $1-\delta$ 
$$\left\|\sum_{j=1}^{K} \phi_i(s_h^j,a_{i,h}^j) \zeta_j\right\|_{\Sigma_{\regress,h}^{-1}}\le 
\cO\paren{ \sqrt{KdH^2}\log(KdH/\delta)}.
$$
\end{lemma}
The proof of this lemma is identical to that of Lemma~\ref{lem:term-B-1}. Finally note that by Lemma~\ref{lem:relative-concentration}, with probability $1-\delta$,
\begin{align*}
    \hat\Sigma_{\regress,h} \succcurlyeq \frac{1}{2}\Sig + \lambda I - \cO\left(\frac{d\log(dK/\delta)}{K}\right) I \succcurlyeq \frac{1}{2}\left(\Sig + \lambda I\right).
\end{align*}
Therefore
\begin{align*}
\MoveEqLeft \left| \phi_i(s,a_i)^\top\hat\theta_{i,h}
- \left[\P^{\pi^t}_h(\oV_{i, h+1}+r_{i,h})\right]\left(s,a_i\right)\right| \\
&\le \|\phi_i(s,a_i)\|_{\hat\Sigma_{\regress,h}^{-1}} \cdot \cO\left(dH\sqrt{1/K}\log(dK/\delta)\right)\\
&\le \|\phi_i(s,a_i)\|_{(\Sigma^{\bar \pi}_{i,h}+\lambda I)^{-1}} \cdot \cO\left(dH\sqrt{1/K}\log(dK/\delta)\right) \\
&\le \frac{1}{2}G_{i,h}(s,\bar\pi,K,\delta).
\end{align*}

We conclude that $\forall (s,a_i)$
\begin{align*}
&\left(\left[\P^{\pi^t}_h(\oV_{i, h+1}+r_{i,h})\right]\left(s,a\right)+G_{i,h}(s,\bar\pi,K,\delta)\right)\land (H-h+1)\le \oQ_{i,h}(s,a)\\
\le  & \left(\P^{\pi^t}_h(\oV_{i, h+1}+r_{i,h})\left(s,a\right)+2G_{i,h}(s,\bar\pi,K,\delta)\right)\land (H-h+1).
\end{align*}
It follows that
\begin{align*}
    \min\left\{\D_{ \pi_{h}}\left[ r_{i,h}+\P_{h+1}\oV_{i,h+1} \right](s)  + G_{i,h}(s,\bar\pi,K,\delta),H-h+1\right\} \le \oV_{i,h}(s), \\
\oV_{i,h}(s) \le \D_{ \pi_{h}}\left[ r_{i,h}+\P_{h+1}\oV_{i,h+1} \right](s)+2G_{i,h}(s,\bar\pi,K,\delta).
\end{align*}

\subsection{Proof of Condition \ref{cond:pigeon}}

Denote 
$$
X_t := \E\left[\phi_i(s_h,a_{i,h})\phi_i(s_h,a_{i,h})\trans\mid s_h\sim \pi^t_{1:h-1},~a_{i,h}\sim\Unif(\cA_i)\right], \qquad S_t :=\sum_{\tau=1}^{t} X_\tau +  \lambda_0  I_{d\times d},
$$
where $\lambda_0 =\tilde\cO(d)$. 
Then using the definition of $G_{i,h}$ in Equation \eqref{equ:g-linear},
\begin{align*}
    &\sum_{t=1}^T \E_{\pi^{t+1}}\left[G_{i,h}(s,\bar\pi^t,t,\delta)\right]\\
\le &\tilde{\cO}\left(\frac{d (\max_i A_i)^{1.5} H}{\sqrt{t}}\cdot \sum_{t=1}^T  
 \E\left[ \sqrt{t}\max_{a_{i,h}\in\cA_i} \| \phi(s_h,a_{i,h})\|_{S_t^{-1}}      \mid s_h\sim \pi^{t+1}_{1:h-1} \right] \right) + \tO\left(1\right)\\
 \le &\tilde{\cO}\left(d (\max_i A_i)^{2.5} H\cdot \sum_{t=1}^T  
 \E\left[ \| \phi(s_h,a_{i,h})\|_{S_t^{-1}}      \mid s_h\sim \pi^{t+1}_{1:h-1}, a_{i,h}\sim\Unif(\cA_i) \right] \right)+ \tO\left(1\right)\\
 \le & \tilde{\cO}\left(d (\max_i A_i)^{2.5} H\cdot 
 \sqrt{T\cdot \sum_{t=1}^T  
 \E\left[ \| \phi(s_h,a_{i,h})\|^2_{S_t^{-1}}      \mid s_h\sim \pi^{t+1}_{1:h-1}, a_{i,h}\sim\Unif(\cA_i) \right]}
\right)+ \tO\left(1\right)\\
\le &  \tilde{\cO}\left(d (\max_i A_i)^{2.5} H\cdot 
 \sqrt{T\cdot \sum_{t=1}^T  
 \E\left[ {\rm tr}(X_{t+1}S_t^{-1})\right]}
\right) + \tO\left(1\right)= \Tilde{\cO}\left(\sqrt{d^3 (\max_i A_i)^5 H^2 T}\right).
\end{align*}

\subsection{Proof of Condition \ref{cond:switch}}
\label{app:linear-cond-4}

Let us fix $(i,h,t)\in [m]\times [H]\times[T]$.
Define $\hat S_{t} := I+\frac{1}{A_i}\sum_{s\in \cB_h^t}\sum_{a_i\in\cA_i} \phi_i(s,a_i)\phi_i(s,a_i)^\top$. Then 
\[
\Psi_{i,h}(\cB_h^t)- \Psi_{i,h}(\cB_h^{I_t}) = \log\det\left(\hat S_t \hat S_{I_t}^{-1}\right).
\]
Therefore that $\Psi_{i,h}(\cB_h^t)- \Psi_{i,h}(\cB_h^{I_t})\le 1$ implies
\[
\Vert \hat S_t ^{\frac{1}{2}}\hat S_{I_t}^{-1}\hat S_t^{\frac{1}{2}}\Vert_2 \le 2,
\]
which further implies
\[
\hat S_t \preccurlyeq 2 \hat S_{I_t}.
\]
In other words, to prove Condition~\ref{cond:switch} it suffices to show that $\hat S_t \preccurlyeq 2 \hat S_{I_t}$ implies 
\[
t(\Sigma^{\bar \pi^t}_{i,h} + \lambda_t I)\le 8I_t\left(\Sigma^{\bar\pi^{I_t}}_{i,h}+\lambda_{I_t} I\right),
\]
where $\lambda_t= \tilde\Theta(d\max_i A_i/t)$. 
This is equivalent to showing that 
\[
t(\Sigma^{\bar \pi^t}_{i,h} + \lambda_t I)\ge 8I_t\left(\Sigma^{\bar\pi^{I_t}}_{i,h}+\lambda_{I_t} I\right),
\]
implies $\hat S_t \succcurlyeq 2 \hat S_{I_t}$. By Lemma~\ref{lem:relative-concentration}, with probability $1-\delta$,
\begin{align*}
    \hat S_t \succcurlyeq \frac{t}{2}\Sigma^{\bar\pi^t}_{i,h}  - \tilde\Theta(d) I \succcurlyeq 
    4I_t\Sigma^{\bar\pi^{I_t}}_{i,h} + \tilde\Theta(d) I  \succcurlyeq 
    2\hat S_{I_t}.
\end{align*}
Finally taking a union bound w.r.t. $i$, $h$ and $t$ proves part (a) of the condition.

As for the second part, we make the observation that $\Psi_{i,h}(\emptyset)=\log\det I =1$, and
\begin{align*}
\Psi_{i,h}(\cD_h^t)\le \log\det\hat S_T \le d\log\left(\Vert S_T\Vert_2\right)\le d\log T.
\end{align*}
Therefore the total number of switches is at most $dmH\log T$, \emph{i.e.} part (b) is satisfied with $\dreplay=dmH$.

\subsection{Sample complexity for linear function approximation}
Sections~\ref{app:linear-cond-1} through~\ref{app:linear-cond-3} show that Conditions~\ref{cond:cce-regret} through~\ref{cond:pigeon} are satisfied with
\[
G_{i,h}(s,\bar\pi,K,\delta) =\sup_{a_i\in\cA_i}\|\phi_i(s,a_i)\|_{(\Sigma^{\bar\pi}_{i,h}+\lambda I)^{-1}}\times \tilde{\Theta}\left(\frac{dH(\max_i A_i)^{1.5}}{\sqrt{K}} \right)+\Theta\paren{\frac{1}{K}}.
\]
and
\[
L=\tO\left(\max_{i\in[m]} d^3 (\max_i A_i)^5 H^2 \right).
\]
Finally Section~\ref{app:linear-cond-4} verified that Condition~\ref{cond:switch} is satisfied with $\dreplay = dmH$. By (\ref{eqn:pi-explore-linear}), $\bar\Gamma = m$. Therefore by applying Theorem~\ref{thm:avlpr}, we obtain following the sample complexity bound for finding an $\epsilon$-CCE
\begin{align*}
\tO\left(\frac{H^3L\bar\Gamma \dreplay}{\epsilon^2}\right) = \tO\left(\frac{d^4 m^2 H^6 \max_{i\in[m]}A_i^5}{\epsilon^2}\right).
\end{align*}

%% file: Sections/lems_linearVLPR.tex
\section{Proofs for Appendix \ref{app:linear-cond-1}}
\label{app:proof-linear-cond-1}

\subsection{Relative concentration}

Consider the following random process: at time step $t$, we (randomly) picks a distribution $D_t$  over the $d$-dimensional unit ball based on $\{x_\tau\}_{\tau\in[t-1]}$, and then sample $x_t\sim D_t$. Denote by $\Sigma_t$ the covariance matrix of $D_t$. We have the following relative concentration lemma regarding the closeness between the empirical temporal-average covariance and the population one in the multiplicative sense.
\begin{lemma}\label{lem:relative-concentration}
With probability at least $1-\delta$, for all $t\in[T]$
$$
 \frac12 \sum_{\tau\in[t]} \Sigma_\tau - \beta I 
  \preceq   \sum_{\tau\in[t]} x_{\tau} x_{\tau}\trans  \preceq
  2\sum_{\tau\in[t]}\Sigma_\tau + \beta I
$$
where $\beta=\Theta(d\log(dT/\delta))$.
\end{lemma}
\begin{proof}
Let us first fix $t\in [T]$. Fix any $w\in\R^d$ with $\Vert w\Vert_2=1$. Define $W_\tau = \langle x_\tau, w\rangle^2$. It follows that
$\E[W_\tau]=w^\top \Sigma_\tau w $, and $0\le W_\tau \le 1$. By Bernstein's inequality, with probability $1-\delta'$
\begin{align}
\label{equ:vertex-concentration}
\left|\sum_{\tau\in[t]}(W_\tau-\E[W_\tau])\right| &\le \sqrt{4\log(1/\delta')\sum_{\tau\in[t]}\E[W_\tau^2]}+\cO(\log1/\delta')\\
&\le \sqrt{4\log(1/\delta')\sum_{\tau\in[t]}\E[W_\tau]}+\cO(\log1/\delta')\\
&\le \frac{1}{2}\sum_{\tau\in[t]}\E[W_\tau] + \cO(\log 1/\delta').
\end{align}
Therefore with probability $1-\delta'$
\begin{align*}
\sum_{\tau\le t}W_\tau &\le 2\E\left[\sum_{\tau\le t}W_\tau\right] + O\left(\log\left(\frac{1}{\delta'}\right)\right),\\
\sum_{\tau\le t}W_\tau &\ge \frac{1}{2}\E\left[\sum_{\tau\le t}W_\tau\right] - O\left(\log\left(\frac{1}{\delta'}\right)\right)
\end{align*}
It remains to construct an $\epsilon'$-cover $W_{\epsilon'}$ of the $d$-sphere, where we choose $\epsilon'=0.01/T$. It follows that with probability $1-\delta$, for all $w$ in the unit sphere,
\begin{align*}
    \sum_{\tau\le t}\langle w,x_\tau\rangle^2 \le 2\E[\sum_{\tau\le t} \langle w,x_\tau\rangle^2] + \cO(\log(|W_{\epsilon'}|/\delta),\\
    \sum_{\tau\le t}\langle w,x_\tau\rangle^2 \ge \frac{1}{2}\E[\sum_{\tau\le t} \langle w,x_\tau\rangle^2] -\cO(\log(|W_{\epsilon'}|/\delta).
\end{align*}
This implies
\begin{align*}
     \frac12 \sum_{\tau\in[t]} \Sigma_\tau - \cO(\log(|W_{\epsilon'}|/\delta) I 
  \preceq   \sum_{\tau\in[t]} x_{\tau} x_{\tau}\trans  \preceq
  2\sum_{\tau\in[t]}\Sigma_\tau + \cO(\log(|W_{\epsilon'}|/\delta)) I
\end{align*}
Replacing $\delta$ by $\delta/T$ and plugging in $|W_{\epsilon'}|\le\left(\frac{3}{\epsilon'}\right)^d$~\citep[Corollary 4.2.13]{vershynin2018high} proves the lemma.
\end{proof}

\subsection{Controlling Term (A) in Condition \ref{cond:cce-regret}}
\label{app:lemma-terma}
In order to evoke the analysis of Expected FPL, we make the observation that
\begin{align*}
\text{Term (A)} &= \max_{\mu_i\in\Delta_{\cA_i}}
 \sum_{k=1}^K \langle \mu_i -\mu_{i,h}^k(\cdot \mid s), \hat \ell^k_{i,h}(s,\cdot)\rangle\\
 &= \max_{\mu_i\in\Delta_{\cA_i}}
 \sum_{k=1}^K \langle \mu_i -\mu_{i,h}^k(\cdot \mid s), \Phi_i(s,\cdot)^\top\hat\theta^k_{i,h}\rangle\\
 &= \max_{x\in CH(\{\phi_i(s,\cdot)\})} \sum_{k=1}^K \langle x - \Phi_i(s,\cdot)\mu^k_{i,h}(\cdot|s), \hat \theta^k_{i,h}\rangle.
\end{align*}
Note that in our algorithm, 
\[
\mu^k_{i,h}(a_i|s):=\Pr_{v\sim \cV}\left[a_i=\argmax\left\langle \phi_i(s,\cdot), \sum_{k'<k}\hat\theta^{k'}_{i,h} + \frac{1}{\eta} (\Sighat+\lambda I)^{-1/2}v\right\rangle\right], 
\]
which implies
\[
\Phi_i(s,\cdot)\mu^k_{i,h}(\cdot|s) = \E_{v\sim \cV}\argmax_{x\in CH(\{\phi_i(s,\cdot)\})}\left\langle (\Sighat+\lambda I)^{-1/2}x, (\Sighat+\lambda I)^{1/2}\sum_{k'<k}\hat\theta^{k'}_{i,h} + \frac{1}{\eta} v\right\rangle.
\]
This is identical to the Expected Follow-the-Perturbed-Leader algorithm (see \emph{e.g.} ~\cite[Algorithm 17]{hazan2016introduction}) on a sequence of linear loss vectors 
$$
(\Sighat+\lambda I)^{1/2}\hat\theta^1_{i,h} ,\cdots, (\Sighat+\lambda I)^{1/2}\hat\theta^K_{i,h}.
$$
Therefore it follows from the regret of Expected FPL~\citep[Theorem 10]{hazan2020faster} that, by choosing $\cV$ to be the uniform distribution over the $d$-dimensional unit ball,
\begin{align*}
{\rm Term (A)} \le \sup_{a_i\in\cA_i}\Vert \phi_i(s,a_i) \Vert_{(\Sighat+\lambda I)^{-1}}\cdot \left[ \frac{1}{\eta} + \eta d \sum_{k=1}^K \Vert \hat\theta^k_{i,h}\Vert_{\paren{\Sighat+\lambda I}}^2\right].
\end{align*}
By Lemma~\ref{lem:loss-linf}, with probability at least $1-\delta$
\begin{align*}
\sum_{k=1}^K \Vert \hat\theta^k_{i,h}\Vert_{\Sighat+\lambda I}^2 =\cO\left(dK+\frac{\log\delta^{-1}}{\lambda}\right).
\end{align*}
Therefore, by plugging in $\eta=1/(d H\sqrt{(\max_i A_i)K\log\delta^{-1}})$ and $\lambda=\tilde{\Theta}(d(\max_i A_i)/K)$, we have
\begin{align*} 
{\rm Term (A)} \le & \sup_{a\in\cA_i}\Vert \phi_i(s,a_i) \Vert_{(\Sighat+\lambda I)^{-1}}\cdot \cO\left(dH\sqrt{K(\max_i A_i)\log\delta^{-1}}\right) \\
= & \sup_{a\in\cA_i}\Vert \phi_i(s,a_i) \Vert_{(\Sig+\lambda I)^{-1}}\cdot \cO\left(dH\sqrt{K(\max_i A_i)\log\delta^{-1}}\right),
\end{align*}
where the equality follows from Lemma \ref{lem:relative-concentration}.

\begin{lemma}
    \label{lem:loss-linf}
With probability $1-\delta$,
\[ \sum_{k=1}^K \Vert \hat\theta^k_{i,h}\Vert^2_{(\Sighat+\lambda I)^{-1}} = \cO\left(dKH^2+\frac{H^2\log\delta^{-1}}{\lambda}\right).\]
\end{lemma}
\begin{proof}
Define 
$$
x_{k} :=\phi_i(s_{i,h}^{k},a_{i,h}^k),  
\quad z_k:=x_k^\top (\Sighat+\lambda I)^{-1}x_k. 
$$By definition 
\[
\sum_{k=1}^K \Vert \hat\theta^k_{i,h}\Vert^2_{(\Sighat+\lambda I)^{-1}} \le H^2 \sum_{k=1}^K z_k.
\]
Moreover, $\{z_k\}_{k\in[K]}$ are i.i.d. samples satisfying that 
\[
\E [z_k] = \trace\paren{\Sig \left(\Sighat+\lambda I\right)^{-1}}\le \cO(d),
\]
where the inequality follows from Lemma \ref{lem:relative-concentration} and the choice of $\lambda$,  and 
\begin{align*}
    \E [z_k^2] &\le \frac{1}{\lambda} \cdot \E\left[z_k \right] \le \cO\paren{\frac{d}{\lambda}},
\end{align*}
and 
\[|z_k| \le \frac{1}{\lambda}.\]
Therefore by Bernstein's inequality, with high probability
\begin{align*}
\sum_k z_k \le \cO\left(dK + \sqrt{\frac{dK\log\delta^{-1}}{\lambda}} + \frac{\log\delta^{-1}}{\lambda}\right)= \cO\left(dK+\frac{\log\delta^{-1}}{\lambda}\right).
\end{align*}
\end{proof}

\subsection{Controlling Term (B) in Condition \ref{cond:cce-regret}}

Consider a fixed player $i\in[m]$ and step $h\in[H]$. To simplify notations, denote 
$$x_{k} :=\phi_i(s_{i,h}^{k},a_{i,h}^k), \quad 
y_k:=r_{i,h}^k+V_{i,h+1}(s_{h+1}^k),\quad 
\zeta_k:=y_k-x_k\trans \theta_{i,h}^k.
$$
For any $(s,a_i)\in\cS\times\cA_i$:
\begin{align*}
    &\sum_{k=1}^K \langle \phi_{i}(s,a_i) , 
    \hat \theta^k_{i,h}-\theta^k_{i,h}\rangle  \\
    = & \left\langle \phi_{i}(s,a_i) , \sum_{k=1}^K \left(\hat \theta^k_{i,h}-\theta^k_{i,h}\right)\right\rangle \\
    =& \left\langle \phi_{i}(s,a_i) , \sum_{k=1}^K \left(\left(\hat\Sigma_{i,h}^{\bar\pi}+\lambda I\right)^{-1}x_k y_k-\theta^k_{i,h}\right)\right\rangle \\
     =& \phi_i(s,a)^\top 
\paren{\Sighat+\lambda I}^{-1} \left[\sum_{k=1}^K x_k \left(x_k^\top\theta^{k}_{i,h} + \zeta_k\right)-\left(\Sighat+\lambda I\right)\sum_{k=1}^K\theta^{k}_{i,h} \right] \\
\le  & \|\phi_i(s,a_i)\|_{ 
\paren{\Sighat+\lambda I}^{-1} }\bigg[\sqrt{\lambda}B_\theta K+ \underbrace{\left\|\sum_{k=1}^K x_k \zeta_k \right\|_{ 
\paren{\Sighat
+\lambda I}^{-1}}}_{\text{Term (B1)}}\\
&+ \underbrace{\norm{\sum_{k=1}^K \left(x_k x_k^\top-\Sighat\right)\theta^{k}_{i,h}}_{ 
\paren{\Sighat+\lambda I}^{-1} }}_{\text{Term (B2)}} \bigg].
\end{align*}
By Lemma \ref{lem:term-B-1}, \ref{lem:term-B-2},  the choice of $\lambda=\tO(d(\max_i A_i)/K)$ and relative concentration (Lemma \ref{lem:relative-concentration}),  
\begin{align*}
\text{Term (B)} \le 
   \tO\paren{ \|\phi_i(s,a_i)\|_{ 
\paren{\Sig+\lambda I}^{-1} }\times Hd\sqrt{K(\max_i A_i)}}.
\end{align*}

\begin{lemma}[Term (B1)]\label{lem:term-B-1}
With probability at least $1-\delta$, we have
$$\left\| \sum_{k=1}^K x_k \zeta_k \right\|_{ 
\paren{\Sighat
+\lambda I}^{-1}} = \cO\paren{ \sqrt{KdH^2\log(KdH/\delta)}+ \frac{dH\log(KdH/\delta)}{\sqrt{\lambda}}}.
$$
\end{lemma}
\begin{proof}
Consider a fixed $v\in\R^d$ with $\|v\|_2=1$.
Define 
$$
z_k := v^\top 
\paren{\Sighat
+\lambda I}^{-1/2}  x_k \zeta_k.
$$
Note that $\{z_k\}_{k=1}^K$ is a martingale with conditional variance and range bounded by
$$
|z_k| \le  H \lambda^{-1/2}, 
$$
and 
\begin{align*}
\text{Var}(z_k \mid z_{1:k-1} )&= \E\left[
\zeta_k^2
v^\top 
\paren{\Sighat
+\lambda I}^{-1/2}  x_k  
x_k^\top \paren{\Sighat
+\lambda I}^{-1/2} v 
\right] \\
& \le H^2 \norm{\paren{\Sighat
+\lambda I}^{-1/2} \Sig \paren{\Sighat
+\lambda I}^{-1/2}}_2 \le \cO(H^2),
\end{align*}
where the second inequality uses Lemma \ref{lem:relative-concentration}, the definition of $\Sighat$ and the choice of $\lambda$.

By Freedman inequality, 
$$
\left| \sum_{k=1}^K  z_k\right| 
\le  \cO\paren{ \sqrt{KH^2\log\delta^{-1}}+ \frac{H\log\delta^{-1}}{\sqrt{\lambda}}}.
$$
Finally, by taking a union bound for all $v$ from a $(\sqrt{\lambda}/(HK))$-cover of the $d$-dimensional unit ball, we conclude that 
\begin{align*}
\left\|\sum_{k=1}^K x_k \zeta_k \right\|_{ 
\paren{\Sighat
+\lambda I}^{-1}} & = \max_{v:~\|v\|_2=1}  
\left|v^\top 
\paren{\Sighat
+\lambda I}^{-1/2}\sum_{k=1}^K x_k \zeta_k\right|   \\
& \le  \cO\paren{ \sqrt{KdH^2\log(KdH/\delta)}+ \frac{dH\log(KdH/\delta)}{\sqrt{\lambda}}}.
\end{align*}
\end{proof}

\begin{lemma}[Term (B2)]\label{lem:term-B-2}
With probability at least $1-\delta$, we have
    $$
\norm{ \sum_{k=1}^K \left(x_k x_k^\top-\Sighat\right)\theta^{k}_{i,h}}_{ 
\paren{\Sighat+\lambda I}^{-1} } 
=\cO\paren{ \sqrt{KdB_\theta^2\log(KdB_\theta/\delta)}+ \frac{dB_\theta\log(KdB_\theta/\delta)}{\sqrt{\lambda}}}.
    $$
\end{lemma}
\begin{proof}
    By triangle inequality and relative concentration (Lemma \ref{lem:relative-concentration}), we have 
    \begin{align*}
& \norm{ \sum_{k=1}^K \left(x_k x_k^\top-\Sighat\right)\theta^{k}_{i,h}}_{ 
\paren{\Sighat+\lambda I}^{-1} } \\
\le &  \norm{ \sum_{k=1}^K \left(x_k x_k^\top-\Sig\right)\theta^{k}_{i,h}}_{ 
\paren{\Sighat+\lambda I}^{-1} }+  \norm{\left(\Sig-\Sighat\right) \sum_{k=1}^K\theta^{k}_{i,h}}_{ 
\paren{\Sighat+\lambda I}^{-1} } \\
\le & \cO\paren{ \norm{ \sum_{k=1}^K \left(x_k x_k^\top-\Sig\right)\theta^{k}_{i,h}}_{ 
\paren{\Sig+\lambda I}^{-1} } +  \norm{\left(\Sig-\Sighat\right) \sum_{k=1}^K\theta^{k}_{i,h}}_{ 
\paren{\Sig+\lambda I}^{-1} }}.
    \end{align*}
Consider an arbitrary $v\in\R^d$ with $\|v\|_2=1$.   Define 
$$
z_k:= v^\top\paren{\Sig+\lambda I}^{-1/2}\left(x_k x_k^\top-\Sig\right)\theta^{k}_{i,h}.
$$
Notice that $\{z_k\}_{k=1}^K$ is a martingale with conditional variance and range bounded by
$$
|z_k| \le  B_\theta \lambda^{-1/2}, 
$$
and 
\begin{align*}
\text{Var}(z_k\mid z_{1:k-1})&\le  \E\left[\left(
v^\top\paren{\Sig+\lambda I}^{-1/2}x_k x_k^\top\theta^{k}_{i,h} 
\right)^2
\right] \\
& \le B_\theta^2 \E\left[\left(
v^\top\paren{\Sig+\lambda I}^{-1/2}x_k  
\right)^2
\right]\\
 & = B_\theta^2 \E\left[
v^\top\paren{\Sig+\lambda I}^{-1/2}x_k  
x_k^\top\paren{\Sig+\lambda I}^{-1/2} v
\right] \\
& = B_\theta^2 
v^\top\paren{\Sig+\lambda I}^{-1/2}(\Sig)\paren{\Sig+\lambda I}^{-1/2} v \le \cO(B_\theta^2),
\end{align*}
where the second equality uses the fact that $\E[x_k x_k^\top]= \Sig$ and the last inequality uses Lemma \ref{lem:relative-concentration}.

By Freedman inequality, 
$$
\left| \sum_{k=1}^K  z_k\right| 
\le  \cO\paren{ \sqrt{KB_\theta^2\log\delta^{-1}}+ \frac{B_\theta\log\delta^{-1}}{\sqrt{\lambda}}}.
$$
Finally, by taking a union bound for all $v$ from a $(\sqrt{\lambda}/(B_\theta K))$-cover of the $d$-dimensional unit ball, we conclude that 
\begin{align*}
& \norm{ \sum_{k=1}^K \left(x_k x_k^\top-\Sig\right)\theta^{k}_{i,h}}_{ 
\paren{\Sig+\lambda I}^{-1} } \\
 \le  & \cO\paren{ \sqrt{KdB_\theta^2\log(KdB_\theta/\delta)}+ \frac{dB_\theta\log(KdB_\theta/\delta)}{\sqrt{\lambda}}}.
\end{align*}
Now recall that $\Sighat$ is estimated by using $K$ samples  i.i.d. sampled from $\bar\pi^t$, so we can simply repeat the above concentration arguments for controlling 
$\norm{ \sum_{k=1}^K \left(x_k x_k^\top-\Sig\right)\theta^{k}_{i,h}}_{ 
\paren{\Sig+\lambda I}^{-1} }$ to upper bound $\norm{\left(\Sig-\Sighat\right) \sum_{k=1}^K\theta^{k}_{i,h}}_{
\paren{\Sig+\lambda I}^{-1} }$, which results in the same bound as above.
 \end{proof}

\subsection{Controlling Term (C) in Condition \ref{cond:cce-regret}}

Consider a fixed player $i\in[m]$ and step $h\in[H]$. To simplify notations, denote 
$$x_{k} :=\phi_i(s_{i,h}^{k},a_{i,h}^k), \quad 
y_k:=r_{i,h}^k+V_{i,h+1}(s_{h+1}^k),\quad 
\zeta_k:=y_k-x_k\trans \theta_{i,h}^k.
$$
We have the following error decomposition similar to the one in controlling Term (B): for any $s\in\cS$,
\begin{equation}\label{eq:termC-1}
\allowdisplaybreaks
\begin{aligned}
 \text{Term (C)}  = & \sum_{k=1}^K \langle   \Phi_s\mu_{i,h}^k(\cdot\mid s), \theta^k_{i,h} -\hat \theta^k_{i,h}\rangle
    \\
    =& \sum_{a_i\in\cA_i}\phi_i(s,a_i)^\top \left( \sum_{k=1}^K
\mu_{i,h}^k(a_i\mid s)
    \theta^{k}_{i,h}  -  \sum_{k=1}^K\mu_{i,h}^k(a_i\mid s) \hat\theta^{k}_{i,h}\right)\\
     =& \sum_{a_i\in\cA_i}\phi_i(s,a_i)^\top 
\paren{\Sighat+\lambda I}^{-1} \bigg[\left(\Sighat+\lambda I\right) \sum_{k=1}^K\mu_{i,h}^k(a_i\mid s)\theta^{k}_{i,h} \\
&~\qquad-  \sum_{k=1}^K\mu_{i,h}^k(a_i\mid s) x_k \left(x_k^\top\theta^{k}_{i,h} + \zeta_k\right)\bigg] \\
\le  & \sum_{a_i\in\cA_i} \|\phi_i(s,a_i)\|_{ 
\paren{\Sighat+\lambda I}^{-1} }\bigg[\sqrt{\lambda}B_\theta K + \underbrace{\left\| \sum_{k=1}^K \mu_{i,h}^k(a_i\mid s)x_k \zeta_k\right\|_{
\paren{\Sighat
+\lambda I}^{-1}}}_{(C1)}\\
&+ \underbrace{\norm{ \sum_{k=1}^K \mu_{i,h}^k(a_i\mid s)\left(x_k x_k^\top-\Sighat\right)\theta^{k}_{i,h}}_{ 
\paren{\Sighat+\lambda I}^{-1} }}_{(C2)} \bigg].
\end{aligned}
\end{equation}

It is easy to verify that the same arguments for bounding Term (B1) and (B2) can be used to bound Term (C1) and (C2), respectively.  Formally, we have the following counterparts of  Lemma \ref{lem:term-B-1} and \ref{lem:term-B-2} for bounding Term (C1) and (C2).

\begin{lemma}[Term (C1)]\label{lem:term-C-1}
Consider a fixed pair of state and action $(s,a_i)\in\cS\times\cA_i$ and a unit vector $v\in\R^d$. 
With probability at least $1-\delta$, we have
\begin{align*}
\MoveEqLeft \left| v\trans\paren{\Sighat
+\lambda I}^{-1/2}\sum_{k=1}^K \mu_{i,h}^k(a_i\mid s)x_k \zeta_k\right| 
\\
= &\cO\paren{ \sqrt{KH^2\log(1/\delta)}+ \frac{H\log(1/\delta)}{\sqrt{\lambda}}}.
\end{align*}
\end{lemma}

\begin{lemma}[Term (C2)]\label{lem:term-C-2}Consider a fixed pair of state and action $(s,a_i)\in\cS\times\cA_i$ and a unit vector $v\in\R^d$. 
With probability at least $1-\delta$, we have
\begin{align*}
    \MoveEqLeft \left| v\trans
\paren{\Sighat+\lambda I}^{-1/2} \sum_{k=1}^K \mu_{i,h}^k(a_i\mid s)\left(x_k x_k^\top-\Sighat\right)\theta^{k}_{i,h} \right| \\
= & \cO\paren{ \sqrt{KB_\theta^2\log(1/\delta)}+ \frac{B_\theta\log(1/\delta)}{\sqrt{\lambda}}}.
    \end{align*}
\end{lemma}
The proofs of Lemma  \ref{lem:term-C-1} and \ref{lem:term-C-2} follow almost the same as the first half of Lemma \ref{lem:term-B-1} and \ref{lem:term-B-2} (before taking the union bound) respectively, so we omit them here.

To control Term (C) with Lemma  \ref{lem:term-C-1} and \ref{lem:term-C-2}, we needs to take a union bound for all state and action $(s,a_i)\in\cS\times\cA_i$ and  unit vector $v\in\R^d$. The following lemma essentially says that such union bound will only incur an additional factor of $\tO(d A_i)$ in the upper bound.

\begin{lemma}\label{lem:policy-lipschitz}
Consider a policy $\pi$ defined as
\[
\pi_{i,h}(a_i|s):=\Pr_{v\sim \cV}\left[a_i=\argmax\langle \phi_i(s,\cdot), w +  W v\rangle\right], 
\]
where $w \in\R^d$ s.t. $\|w\|_2 \le \gamma$ , $ \alpha I_{d\times d} \preceq W\preceq \beta I_{d\times d}$, $\cV$ denotes the uniform distribution over the $d$-dimensional unit ball. Then for any states $s,s'\in\cS$ satisfying 
$
\max_{a_i\in\cA_i} \| \phi_{i}(s,a_i) - \phi_{i}(s',a_i)\|_2 \le \epsilon,
$
we have
    $$
\left\| \E_{a\sim\pi_{i,h}(\cdot \mid s)}[ \phi_{i}(s,a_i)]
-\E_{a\sim\pi_{i,h}(\cdot \mid s')}[ \phi_{i}(s',a_i)] \right\|_2 =\tO\paren{\frac{d\beta \gamma \sqrt{\epsilon}}{\alpha^2}}. 
$$
\end{lemma}
We defer the proof of Lemma \ref{lem:policy-lipschitz} to the end of this subsection. 

By standard discretization argument, there exists a subset $\cS_\epsilon$ of $\cS$ (i.e., a discrete cover of $\cS$ w.r.t. metric $d(s,s')=\max_{a_i\in\cA_i} \| \phi_{i}(s,a_i) - \phi_{i}(s',a_i)\|_2$)  such that 
\begin{itemize}
    \item for any $s\in\cS$, there exists $s'\in\cS_\epsilon$ satisfying 
    $$
    \max_{a_i\in\cA_i} \| \phi_{i}(s,a_i) - \phi_{i}(s',a_i)\|_2 \le \frac{1}{\poly(B_\theta,K,d,H,A_i,\lambda^{-1},\eta^{-1},\delta^{-1})},
    $$

    \item and 
    $$
\log |\cS_\epsilon| \le \tO\paren{d A_i}.
    $$
\end{itemize}
For all $s\in\cS$: denote by $s'$ the closest neighbour of $s$ in  $\cS_\epsilon$ w.r.t. metric $d(s,s')=\max_{a_i\in\cA_i} \| \phi_{i}(s,a_i) - \phi_{i}(s',a_i)\|_2$,
\begin{align*}
\allowdisplaybreaks
\MoveEqLeft \sum_{k=1}^K \langle   \Phi_s\mu_{i,h}^k(\cdot\mid s),  \hat \theta^k_{i,h}-\theta^k_{i,h}\rangle\\
\overset{(i)}{\le} &   \sum_{k=1}^K \langle   \Phi_{s'}\mu_{i,h}^k(\cdot\mid {s'}),  \hat \theta^k_{i,h}-\theta^k_{i,h}\rangle  + 1\\
\overset{(ii)}{\le} & \sum_{a_i\in\cA_i} \|\phi_i(s',a_i)\|_{ 
\paren{\Sighat+\lambda I}^{-1} }\bigg[\sqrt{\lambda}B_\theta K \\
&+ \max_{v:~\|v\|_2=1}\left| v\trans
\paren{\Sighat+\lambda I}^{-1/2} \sum_{k=1}^K \mu_{i,h}^k(a_i\mid s')\left(x_k x_k^\top-\Sighat\right)\theta^{k}_{i,h} \right|\\
&+ \max_{v:~\|v\|_2=1}\left| v\trans
\paren{\Sighat+\lambda I}^{-1/2} \sum_{k=1}^K \mu_{i,h}^k(a_i\mid s')\left(x_k x_k^\top-\Sighat\right)\theta^{k}_{i,h} \right|  \bigg] +1\\
 \overset{(iii)}{\le}  & \max_{a_i\in\cA_i}\|\phi_i(s',a_i)\|_{ 
\paren{\Sig+\lambda I}^{-1} }\times \tO\paren{ dH\sqrt{K(\max_i A_i)^3} }+1\\
 \overset{(iv)}{\le}  & \max_{a_i\in\cA_i}\|\phi_i(s,a_i)\|_{ 
\paren{\Sig+\lambda I}^{-1} }\times \tO\paren{ dH\sqrt{K(\max_i A_i)^3}  }+1,
\end{align*}
where $(i)$ uses the definition of $\mu_{i,h}^k$, $\cS_\epsilon$ and Lemma \ref{lem:policy-lipschitz}, $(ii)$ uses Equation \eqref{eq:termC-1}, 
$(iii)$ uses  Lemma  \ref{lem:term-C-1} and \ref{lem:term-C-2} along with a union bound for all $s'\in\cS_\epsilon$ and all  $v$  from a 
$$
{1}/{\poly(B_\theta,K,d,H,A_i,\lambda^{-1},\eta^{-1},\delta^{-1})}\text{-cover}$$
 of the $d$-dimensional unit ball, and (iv) uses the fact that $s'$ is the closest neighbour of $s$ in  $\cS_\epsilon$.

As a result,
$$
\text{Term (C)}\le \max_{a_i\in\cA_i}\|\phi_i(s,a_i)\|_{ 
\paren{\Sig+\lambda I}^{-1} }\times \tO\paren{ dH\sqrt{K(\max_i A_i)^3}  }+1.
$$
\begin{proof}
[Proof of Lemma \ref{lem:policy-lipschitz}]
To simplify notations, denote 
$x_{a_i} = \phi_{i}(s,a_i)$ and $\bar x_{a_i} = \phi_{i}(s',a_i)$, $a_i\in\cA_i$. We cluster the actions in $\cA_i$ into $\{\cC_v\}_{v=1}^n$ according to the following rule: action $a_i$ and $a_i'$ are in the same cluster if and only if $\|x_{a_i}-x_{a_i'}\|_2 \le \Delta$, where $\Delta>10\epsilon$ is a parameter to be specified later.  Denote 
$y_v:= \frac{1}{|\cC_v|}\sum_{a_i\in\cC_v}x_{a_i}$. We further denote by $c(a_i)$ the cluster that $a_i\in\cA_i$ belongs to.

It is simple to verify that
 $$
\left\| \E_{a\sim\pi_{i,h}(\cdot \mid s)}[ \phi_{i}(s,a_i)]
-\E_{a\sim\pi_{i,h}(\cdot \mid s)}[ y_{c(a_i)}] \right\|_2 \le \Delta,
$$
and 
 $$
\left\| \E_{a\sim\pi_{i,h}(\cdot \mid s')}[ \phi_{i}(s',a_i)]
-\E_{a\sim\pi_{i,h}(\cdot \mid s')}[  y_{c(a_i)}] \right\|_2 \le 2\Delta.
$$
 As a result, to prove Lemma \ref{lem:policy-lipschitz}, it suffices to upper bound 
$$
\max_{v\in[n]} \left| 
\pi_{i,h}(\cC_v \mid s) - 
\pi_{i,h}(\cC_v \mid s') 
\right|.
$$
For any $(a_i,a_i',\theta)\in\cA_i^2\times\R^d$, define event 
$$
E_{a_i,a_i'}(v) = \{ (x_{a_i} -x_{a_i'})\trans (w +  W v) >0 \},\qquad
E_{a_i}(v) = \bigcap_{a_i'\in(\cA_i/\cC_{c(a_i)})} E_{a_i,a_i'}(v).
$$
Similarly, we define $\bar E_{a_i,a_i'}(v)$ and $\bar E_{a_i}(v)$ by replacing $x$ with $\bar x$ in the above definition. 
We have
\begin{align*}
& \left| 
\pi_{i,h}(\cC_v \mid s) - 
\pi_{i,h}(\cC_v \mid s') 
\right| \\
= & \left|  \P\left(\bigcup_{a_i\in\cC_v}E_{a_i}(v)\right) -\P\left(\bigcup_{a_i\in\cC_v}\bar E_{a_i}(v)\right)   \right| \\
= &  \left|  \P\left(\bigcup_{a_i\in\cC_v}\bigcap_{a_i'\in(\cA_i/\cC_{v})} E_{a_i,a_i'}(v)\right) -\P\left(\bigcup_{a_i\in\cC_v}\bigcap_{a_i'\in(\cA_i/\cC_{v})} \bar E_{a_i,a_i'}(v)\right)   \right| \\
\le & \sum_{a_i\in\cC_v} \sum_{a_i'\in (\cA_i/\cC_v)}
\left|  \P(E_{a_i,a_i'}(v) )- \P(\bar E_{a_i,a_i'}(v))\right|.
\end{align*}
By the definition of $E_{a_i,a_i'}(v)$,
\begin{align*}
     \P(E_{a_i,a_i'}(v) ) &= \P\paren{(x_{a_i} -x_{a_i'})\trans (w +  W v) >0} \\
     &= \P\paren{(x_{a_i} -x_{a_i'})\trans W v >(x_{a_i'} -x_{a_i})\trans w }\\
     & = \P\paren{v_1  >\frac{(x_{a_i'} -x_{a_i})\trans w}{\|(x_{a_i} -x_{a_i'})\trans W \|_2} },
\end{align*}
where the last equality uses the symmetry of distribution $\cV$. 
By simple algebra, one can show the density function of $v_1$ is upper bounded by $\tO(d)$. As a result, we have 
\begin{align*}
    &  \left|  \P(E_{a_i,a_i'}(v) )- \P(\bar E_{a_i,a_i'}(v))\right|\\
     &\le \tO(d) \times \left|\frac{(x_{a_i'} -x_{a_i})\trans w}{\|(x_{a_i} -x_{a_i'})\trans W  \|_2}- \frac{(\bar x_{a_i'} -\bar x_{a_i})\trans w}{\|(\bar x_{a_i} -\bar x_{a_i'})\trans W  \|_2}\right| \\
     &\le \tO(d) \times \left|\frac{(x_{a_i'} -x_{a_i})\trans w \times  \|(\bar x_{a_i} -\bar x_{a_i'})\trans W  \|_2
     - (\bar x_{a_i'} -\bar x_{a_i})\trans w \times \|(x_{a_i} -x_{a_i'})\trans W  \|_2
     }{\|(x_{a_i} -x_{a_i'})\trans W  \|_2 \times \|(\bar x_{a_i} -\bar x_{a_i'})\trans W  \|_2} \right| \\
     &\le \tO\paren{\frac{d}{\alpha^2 \Delta^2}} \times \left|(x_{a_i'} -x_{a_i})\trans w \times  \|(\bar x_{a_i} -\bar x_{a_i'})\trans W  \|_2
     - (\bar x_{a_i'} -\bar x_{a_i})\trans w \times \|(x_{a_i} -x_{a_i'})\trans W  \|_2
     \right| \\
     &\le \tO\paren{\frac{d}{\alpha^2 \Delta^2}} \times \cO\paren{ \beta \gamma \epsilon} =  \tO\paren{\frac{d\beta \gamma \epsilon}{\alpha^2 \Delta^2}}
\end{align*}
where: (i)  the third inequality uses 
the fact that $a_i$ and $a_{i'}$ are from different clusters, $W \succeq \alpha I$, and $\|\bar x_{a_i}- x_{a_i}\|_2\le \epsilon \le 0.1 \Delta$; (ii) the last inequality uses triangle inequality, $\|w\|_2\le \gamma$, $W \preceq \beta I$ and $\|\bar x_{a_i}- x_{a_i}\|_2\le \epsilon$. We complete the proof by choosing $\Delta=20\epsilon^{1/4}$.
\end{proof}

%% file: Sections/proof_tabularVLPR.tex
\section{Proofs for Section \ref{sec:tabular-avlpr}}
\label{sec:proof-tabular-appendix}

\subsection{Details of the tabular \avlpr~algorithm}
\label{app:details-tabular-avlpr}

The tabular MG case is a special case of the linear function approximation setting with finite number of states, \emph{i.e.} $|\cS|\le S$. For the tabular setting, we choose the switching criterion function $\Psi_h$ as 
$$
\Psi_{h}(\cB_h):=\ln\prod_{s\in\cS}\max\left\{\sum_{s_h\in\cB_h}\mathbbm{1}(s_h=s),1\right\},
$$
while the exploration scheme is chosen as $\Gamma_{\rm explore}(\bar\pi,\mu_h):=\{(\bar\pi_{1:h-1}\odot \mu_h,[m])\}$. In other words, in Line 4 of Algorithm~\ref{alg:abs-cce} and Line 3 of Algorithm~\ref{alg:abs-v}, all players jointly play $\mu_h^k$ (or $\pi_h$) once.

\paragraph{\noreg} Notice that  $\cD^{k,i}_{\rm sample}=\{(s_h^k,a_{i,h}^k,r_{i,h}^k+V_{i,h+1}(s_{h+1}^k))\}$ always consists of a single sample. We will use it to perform an EXP3-IX style update~\citep{neu2015explore}, that is
\begin{align*}
    \hat{\ell}^k_{i,h}(s,a_{i})&= \frac{H-r_{i,h}^k-V_{i,h+1}(s_{h+1}^k)}{\mu_{i,h}^k(a_i\mid s) + \gamma_i}\times \mathbbm{1}((s,a_i)=(s_h^k,a_{i,h}^k)),\\
    \mu^{k+1}_{i,h}(\cdot \mid s)&
\propto \exp\left(-\eta_i \sum_{k'\le k}\hat \ell^{k'}_{i,h}(s,\cdot)\right),
\end{align*}
where
$\eta_i = \sqrt{\frac{S\log T}{H^2A_i T}}$ and $\gamma_i=\frac{\eta}{2}$.

\paragraph{\optreg}
Denote the data tuple in $\cD_{\rm reg}^i$ by $\{(s_h^k,a_{i,h}^k,r_{i,h}^k+V_{i,h+1}(s_{h+1}^k))\}_{k\in[K]}$. 
Define $N_h(s):=\sum_{k=1}^K  \mathbbm{1}(s_h^k=s)$ and
$$
\beta_{i}(n):=\Theta\left(\frac{\iota}{\eta_i (n+\iota)}+\eta_i H^2A_i\right),
$$
where $\iota=\log(KSA_i H m/\delta)$. The optimistic regression is performed by an empirical averaging step with bonus: if $N_h(s)>0$, set $V_{i,h}(s)= H-h+1$, otherwise, 
$$
V_{i,h}(s) =
\min\left\{\frac{1}{N_h(s)}\sum_{k=1}^K (r_{i,h}^k+V_{i,h+1}(s_{h+1}^k))  \times \mathbbm{1}(s_h^k=s) + \beta_i(N_h(s)), H-h+1\right\}.
$$

\paragraph{Computational efficiency} It is straightforward to see that, as our instantiation only involves standard EXP3 algorithm with exponential weights updates, bouns computations, and simple averaging, the entire algorithm runs in polynomial time in $(T,H,S,\set{A_i}_{i\in[m]})$.

The rest of this section is devoted to proving Theorem~\ref{thm:tabular-rate}, by checking Conditions~\ref{cond:cce-regret} through~\ref{cond:switch} and then applying Theorem~\ref{thm:avlpr}.

\subsection{Proof of Condition \ref{cond:cce-regret}}
\label{app:tabular-cond-1}
Denote by $N_h(s)$ the number of times state $s$ is visited at step $h$ during the $K$ episodes of executing $\bar\pi$ in \cceaprx.  Let $\iota=\log(mSK\max_i A_i/\delta)$ and $p_s:=\P^{\bar\pi}(s_h=s)$. 
By invoking the theoretical guarantee of Exp3-IX (e.g., Theorem 12.1 in \citet{lattimore2020bandit}) and taking a union bound for all $(i,s)\in[m]\times\cS$, we have that 
with probability at least $1-\delta$: for all $(i,s)\in[m]\times\cS$:
\begin{align*}
\max_{\mu_{i,h}\in\Delta_{A_i}} \sum_{k=1}^K
\left(\D_{\mu_{i,h}\times \mu_{-i,h}^k}-\D_{ \mu_{h}^k}\right)\left[r_{i,h}+\P_{h+1}V_{i,h+1} \right](s) \times \mathbbm{1}(s_h^k=s)\\
\le \cO\left(\frac{\iota}{\eta_i}+\eta H^2A_iN_h(s)\right).
\end{align*}
By Freedman's inequality and taking a union bound for all $(i,s,\mu_{i,h})\in[m]\times\cS\times\{e_i\}_{i\in[A_i]}$, we have that with probability at least $1-\delta$: for all $(i,s)\in[m]\times\cS$: 
\begin{align*}
& \max_{\mu_{i,h}\in\Delta_{A_i}} \sum_{k=1}^K
\left(\D_{\mu_{i,h}\times \mu_{-i,h}^k}-\D_{ \mu_{h}^k}\right)\left[r_{i,h}+\P_{h+1}V_{i,h+1} \right](s) \times \mathbbm{1}(s_h^k=s)\\
\ge & p_s \times \max_{\mu_{i,h}\in\Delta_{A_i}} \sum_{k=1}^K
\left(\D_{\mu_{i,h}\times \mu_{-i,h}^k}-\D_{ \mu_{h}^k}\right)\left[r_{i,h}+\P_{h+1}V_{i,h+1} \right](s) - \cO(H\sqrt{p_sK\iota}+H\iota ),
\end{align*}
and 
\begin{align*}
N_h(s)  \le \cO(p_s K +\iota).
\end{align*}
Combining all above relations gives that 
\begin{align*}
\MoveEqLeft \frac{1}{K}\max_{\mu_{i,h}\in\Delta_{A_i}} \sum_{k=1}^K
\left(\D_{\mu_{i,h}\times \mu_{-i,h}^k}-\D_{ \mu_{h}^k}\right)\left[r_{i,h}+\P_{h+1}V_{i,h+1} \right](s) \\
\le  &\min\left\{\cO\left(\frac{\iota}{\eta_i p_s K}+\eta_i H^2A_i\left(1+ \frac{\iota}{p_s K}\right)\right),H\right\}\\
\le  &\cO\left(\frac{\iota}{\eta_i (p_s K+\iota)}+\eta_i H^2A_i\right),
\end{align*}
where the last inequality uses the fact that $\eta^{-2}_i\ge  A_i$. As a result, we can pick
\[
G_{i,h}(s,\bar\pi,K,\delta) =\cO\paren{\frac{\iota}{\eta_i (K\P^{\bar\pi}(s_h=s)+\iota)}+\eta_i H^2A_i}.
\]

\subsection{Proof of Condition \ref{cond:optv}}
\label{app:tabular-cond-2}
Denote by $N_h(s)$ the number of times state $s$ is visited at step $h$ during the $K$ episodes of executing $\bar\pi$ in \vapprox.  
Let $\iota=\log(mSK\max_i A_i/\delta)$ and $p_s:=\P^{\bar\pi}(s_h=s)$. 
Since the case of $N_h(s)=0$ is trivial, below we only consider those state $s$ such that  
$N_h(s)>0$.

By Azuma-Hoeffding inequality and taking a union bound for all $(i,s,\mu_{i,h})\in[m]\times\cS\times\{e_i\}_{i\in[A_i]}$, we have that with probability at least $1-\delta$: for all $(i,s)\in[m]\times\cS$: 
\begin{align*}
 \left|\frac{1}{N_h(s)} \sum_{k=1}^K
\left(r_{i,h}^k+V_{i,h+1} (s_h^{k+1}) \right)\times \mathbbm{1}(s_h^k=s) - \D_{ \pi_{h}}\left[ r_{i,h}+\P_{h+1}V_{i,h+1} \right](s) \right|\\
\le \cO\left(H\sqrt{\frac{\iota}{N_h(s)}}\right)\le \cO\left(\frac{\iota}{\eta_i (N_h(s)+\iota)}+\eta_i H^2A_i\right),
\end{align*}
where the second inequality uses the fact  that $\eta^{-1}_i\ge \iota$.
As a result, to prove both relations in Condition \ref{cond:optv}, it suffices to show for all $(i,s)\in[m]\times\cS$: 
$$
G_{i,h}(s,\bar\pi,K,\delta) =  \Theta\left(\frac{\iota}{\eta_i (N_h(s)+\iota)}+\eta_i H^2A_i\right).
$$
By Bernstein inequality and taking a union bound for all $s\in\cS$, we have that with probability at least $1-\delta$: for all $s\in\cS$: 
\begin{align*}
\frac12 p_s K -\frac12 \iota \le N_h(s)  \le 2 p_s K +\frac12\iota.
\end{align*}
We complete the proof by plugging the above sandwich relation back into the definition of 
$G_{i,h}(s,\bar\pi,K,\delta) $.

\subsection{Proof of Condition \ref{cond:pigeon}}
\label{app:tabular-cond-3}
Let $\iota=\log(mSK\max_i A_i/\delta)$, $w_s^t:=\P^{\pi^t}(s_h=s)$ and $W_s^t:=\sum_{\tau\le t}\P^{\pi^\tau}(s_h=s)$. 
By plugging in the definition of $G_{i,h}$, we have 
\begin{align*}
\allowdisplaybreaks
     &\sum_{t=1}^T  \E_{\pi^{t+1}}\left[G_{i,h}(s,\bar\pi^t,t,\delta)  \right]  \\
=  & \cO\left(\sum_{t=1}^T  \E_{\pi^{t+1}} \left[\frac{\iota}{\eta_i (W_s^t+\iota)}+\eta_i H^2A_i \right]\right)\\
= & \cO\left(\sum_{s\in\cS}\sum_{t=1}^T  w_s^{t+1}\left[\frac{\iota}{\eta_i (W_s^t +\iota)}+\eta_i H^2A_i \right]\right) = \cO\left( \frac{S\iota\log(T)}{\eta_i}+\eta_i H^2A_i T\right).
\end{align*}

\subsection{Proof of Condition \ref{cond:switch}}
\label{app:tabular-cond-4}
Denote by $N_h(s,\cB_h)$ the number of times state $s$ occurs in dataset $\cB_h$. By Bernstein inequality and taking a union bound for all $(t,s)\in[T]\times\cS$, we have that with probability at least $1-\delta$: for all $(t,s)\in[T]\times\cS$: 
\begin{align*}
\frac12 \sum_{\tau=1}^t \P^{\pi^\tau}(s_h=s) -\frac12 \iota \le N_h(s,\cB_h^t)  \le 2 \sum_{\tau=1}^t \P^{\pi^\tau}(s_h=s)  +\frac12\iota.
\end{align*}
Since $\Psi_{i,h}(\cB^t_h) \le   \Psi_{i,h}(\cB_{h}^{I_t})+1$, we have 
$N_h(s,\cB_h^t)\le 2N_h(s,\cB_h^{I_t})$. Using the above relative concentration result, we obtain
\begin{align*}
\sum_{\tau=1}^t \P^{\pi^\tau}(s_h=s)
   + \iota    
  \le 2(N_h(s,\cB_h^t)+ \iota) \le 
  4 N_h(s,\cB_h^{I_t})+ 2\iota
  \le 8\sum_{\tau=1}^{I_t} \P^{\pi^\tau}(s_h=s) + 4\iota.
\end{align*}
Finally, we complete the proof of Condition~\ref{cond:switch}(a) by recalling 
$$
G_{i,h}(s,\bar\pi^t,t,\delta) =\frac{\iota}{\eta_i (\sum_{\tau=1}^t \P^{\pi^\tau}(s_h=s)+\iota)}+\eta_i H^2A_i.
$$
As for Condition~\ref{cond:switch}(b), simply observe that: (1) $\Psi_{i,h}$ does not depend on $i$; (2) $\Psi_{i,h}(\cD^t_h)\le S\log t$. Therefore the total number of switches up to iteration $T$ is bounded by $SH\log T$. In other words Condition~\ref{cond:switch}(b) is satisfied with $\dreplay=SH$.
\subsection{Sample complexity for tabular MG}
Sections~\ref{app:tabular-cond-1} and~\ref{app:tabular-cond-2} shows that Condition~\ref{cond:cce-regret} and~\ref{cond:optv} are satisfied with
\[
G_{i,h}(s,\bar\pi,t,\delta) =\cO\paren{\frac{\iota}{\eta_i (K\P^{\bar\pi}(s_h=s)+\iota)}+\eta_i  H^2A_i}.
\]
Meanwhile Section~\ref{app:tabular-cond-3} shows that this choice of $G$ satisfies Condition~\ref{cond:pigeon} with
\[
L=\tO\left(SH^2\max_iA_i\right).
\]
Finally Section~\ref{app:tabular-cond-4} shows that Condition~\ref{cond:switch} is satisfied with $\dreplay=SH$. It remains to apply Theorem~\ref{thm:avlpr}, which gives the sample complexity bound of
\[
\tilde\cO\left(\frac{H^3L\dreplay}{\epsilon^2}\right)=\tilde\cO\left(\frac{S^2H^6\max_iA_i}{\epsilon^2}\right).
\]
Note that in the tabular algorithm, $\Piexp$ contains a single element, so $\bar\Gamma=1$.

%% file: Sections/discussion_pi_cce.tex
\section{Difference between $\Pi^{\Mar}$-CCE and CCE}
\label{app:separation_pi_cce}

Here we provide an example of a toy Markov Game in which there exists a correlated policy $\Lambda\in\Delta(\Pi^{\Mar})$, where $\Pi^{\Mar}$ is the set of all Markov product policies, such that $\CCEGap^{\Pi^{\Mar}}(\Lambda)=0$ but $\CCEGap(\Lambda)\ge H/4$ for any $H\ge 2$.

Consider the following ``sequential rock-paper-scissors'' game with horizon $H\ge 2$. The game is two-player zero-sum (with $m=2$ and $r_2\equiv 1-r_1$). The state space is a singleton ($\cS=\sets{s_0}$ and $S=1$), and each player has three actions corresponding to rock, paper, and scissors ($A_1=A_2=3$). The instantaneous reward $r_1(a_1,a_2)\in\set{0,1/2,1}$ for player 1 is determined by the standard rock-paper-scissors rule (for example, $r_1({\rm rock}, {\rm scissors})=1$ and $r_1({\rm rock}, {\rm rock})=1/2$). Let $\Pi^{\Mar}_1$, $\Pi^{\Mar}_2$ denote the set of all Markov policies for each player, and $\Pi^{\Mar}=\Pi^{\Mar}_1\times\Pi^{\Mar}_2$. A Markov policy in this game corresponds to running a memoryless (non history-dependent) policy at each stage $h\in[H]$. 

Let $\Lambda=\Unif(\sets{\pi^{\rm rock},\pi^{\rm paper},\pi^{\rm scissors}})$, where for each $\act\in\sets{{\rm rock}, {\rm paper}, {\rm scissors}}$,
\begin{align*}
    \pi^{\act} \defeq \pi_1^{\act}\times \pi_2^{\act},~~~{\rm where}~~~\pi_{i,h}^{\act}(\cdot|s_0)=\delta_{\act}~~\textrm{for all}~(i,h)\in[2]\times[H]
\end{align*}
specifies the policy where both players play action $\act$ deterministically within all $H$ steps. Note that $\pi^{\act}\in\Pi^{\Mar}$ and thus $\Lambda\in\Delta(\Pi^{\Mar})$. 

By definition of $\Lambda$, we have $V_1^{\Lambda}=H/2$. Further, it is straightforward to see that $\max_{\pi_1^\dagger\in\Pi_1^{\Mar}} V_1^{\pi_1^\dagger, \Lambda_{-1}}=H/2$, as this is achievable by picking $\pi_1^\dagger=\pi_1^{\rm rock}$, and no other Markov policy $\pi_1^\dagger\in\Pi_1^{\Mar}$ (which is memoryless) can achieve a reward greater than $1/2$ at any step against $\Lambda_{-1}$, which plays uniformly within $\sets{{\rm rock}, {\rm paper}, {\rm scissors}}$ at every step. This shows that $\CCEGap^{\Pi^{\Mar}}(\Lambda)=0$. 

However, consider the non-Markov policy $\wt{\pi}_1$ that plays uniformly at random at $h=1$, observes the action played by the opponent (or infers the opponent's played action from the received reward), and henceforth plays the winning action against that action at step $h\in\set{2,\dots,H}$. By definition of $\Lambda$, such a non-Markov policy will deterministically achieve reward $1$ at all steps $h\ge 2$, and thus
\begin{align*}
    V_1^{\wt{\pi_1}, \Lambda_{-1}} = \frac{1}{2} + H-1 = H-\frac{1}{2},
\end{align*}
which gives
\begin{align*}
    \CCEGap(\Lambda) \ge V_1^{\wt{\pi_1}, \Lambda_{-1}} - V_1^{\Lambda} = H-\frac{1}{2} - \frac{H}{2} \ge \frac{H}{4}
\end{align*}
for any $H\ge 2$.

%% file: Sections/proof_pi_cce.tex
\section{Proofs for Section~\ref{sec:restricted}}

\subsection{Explorative All-Policy Evaluation (APE)}

We provide the full description of the \APE{} algorithm in Algorithm~\ref{alg:ftu}.

\begin{algorithm}[h]
\small
\caption{$\APE_i(\cF_i,\Pi_i,\pi_{-i},K,\beta)$: Explorative All-Policy Evaluation ($i$-th player)}
\label{alg:ftu}
\begin{algorithmic}[1]
\STATE \textbf{Initialize} confidence set $\cB^1 \setto \cF_i\times\Pi_i$, $\cD_h\setto \sets{}$.
    \FOR{$k=1,\dots,K$}
    \STATE Compute upper and lower value estimates for all $\pi_i\in\Pi_i$:
    \begin{align*}
        (\up{V}^{k,\pi_i\times \pi_{-i}}, \low{V}^{k,\pi_i\times \pi_{-i}}) \setto 
        \paren{ \max_{f:(f,\pi_i)\in\cB^k} f_1(s_1, \pi_{i,1}(s_1)),
        \min_{f:(f,\pi_i)\in\cB^k} f_1(s_1, \pi_{i,1}(s_1)) }.
    \end{align*}
    \STATE Choose $\pi_i^k \setto \argmax_{\pi_i\in\Pi_i} (\up{V}^{k,\pi_i\times \pi_{-i}}  -\low{V}^{k,\pi_i\times \pi_{-i}})$. \label{line:ftu}
     \STATE Execute $\pi_i^k\times\pi_{-i}$, and collect the trajectory $(s_1^{k}, a_{i,1}^k, r_{i,1}^k, \dots, s_H^k, a_{i,H}^k, r_{i,H}^k)$ for the $i$-th player.
     \STATE Update $\cD_h\setto \cD_h\cup\sets{(s_{h}^{k},a_{i,h}^{k},r_{i,h}^{k},s_{h+1}^{k})}$ for all $h\in[H]$.
     \STATE Update confidence set 
     \begin{align*}
     & \cB^{k+1}= \set{
     (f,\pi_i)\in\cF_i\times \Pi_i:~ \cL_h^{\cD_h}(f_h,f_{h+1},\pi_i) \le \min_{f_h'\in\cF_{i,h}}\cL_h^{\cD_h}(f_h',f_{h+1},\pi_i)+\beta,~\forall~h\in[H]
     } \bigcap
     \cB^{k}, \\
     & {\rm where}~~
     \cL_h^{\cD_h}(f_h,f_{h+1},\pi_i) :=\sum_{(s,a_i,r,s')\in\cD_h} \brac{ f_h(s,a_i)-r-f_{h+1}(s',\pi_{i,h}(s')) }^2.
     \end{align*}
     \label{line:ape-conf-set}
\ENDFOR
\ENSURE Optimistic value estimates $\sets{\up{V}^{K,\pi_i\times\pi_{-i}}}_{\pi_i\in\Pi_i}$.
\end{algorithmic}    
\end{algorithm}

\subsection{Proof of Theorem~\ref{thm:restricted-cce}}
\label{app:proof-restricted-cce}

In this section we prove Theorem~\ref{thm:restricted-cce}. We first present the guarantee for the $\FTU$ subroutine in the following proposition, whose proof can be found in Appendix~\ref{app:proof-ftu}.

\begin{proposition}[Learning accurate Q-functions for all policies by \APE{}]
\label{prop:ftu}
    Under Assumption~\ref{asp:completeQ} \&~\ref{asp:bounded_eluder}, there exists an absolute constant $c>0$ so that for any player $i\in[m]$, if we choose $\beta=cH^2\log(|\Pi_i||\cF_i|KH/\delta)$ in Algorithm \ref{alg:ftu}, then with probability at least $1-\delta$ we have
    \begin{enumerate}[label=(\alph*)]
    \item 
    $\low{V}^{K,\pi_i\times\pi_{-i}} \le V_{i,1}^{\pi_i, \pi_{-i}}(s_1) \le \up{V}^{K,\pi_i\times\pi_{-i}}$ for all $\pi_i\in\Pi_i$.
    \item
    $\max_{\pi_i\in\Pi_i} \paren{ \up{V}^{K,\pi_i\times\pi_{-i}} - \low{V}^{K,\pi_i\times\pi_{-i}} }
    \le\cO\left(H\sqrt{\frac{d_i\log K\cdot \beta}{K} }\right)
    $.
    \end{enumerate}
\end{proposition}

Since Algorithm~\ref{alg:opmd} calls the $\FTU$ subroutine for $T$ round with $m$ players per round with parameters $(\beta,K)\setto (\beta_i,K_i)$, applying Proposition~\ref{prop:ftu} with a union bound yields that, with probability at least $1-\delta/2$, the optimistic value estimates $\sets{\Vb_{i}^{(t),\pi_i\times \pi^{t}_{-i}}}_{\pi_i\in\Pi_i}$ satisfy that
\begin{align}
\label{eqn:optimistic_accurate_v}
\begin{aligned}
    V_{i,1}^{\pi_i, \pi_{-i}}(s_1) 
    & \stackrel{(i)}{\le} V_{i,1}^{\pi_i, \pi_{-i}}(s_1) + \cO\left(H^2\sqrt{\frac{d_i\log(K_i)\cdot \log(\sum_i \abs{\Pi_i}\abs{\cF_i} TK_iH/\delta)}{K_i} }\right) \\
    & \stackrel{(ii)}{\le} V_{i,1}^{\pi_i, \pi_{-i}}(s_1) + \eps/2
\end{aligned}
\end{align}
for all $i\in[m]$, $\pi_i\in\Pi_i$, and $t\in[T]$ simultaneously. Above, (i) used our choice of $\beta_i$, and (ii) can be satisfied by choosing
\begin{align}
\label{eqn:ki_choice}
    K_i = \tO\paren{\frac{H^4d_i\cdot \log(\sum_i \abs{\Pi_i}\abs{\cF_i})}{\eps^2}}.
\end{align}

We next show that \DOPMD{} achieves small regret for any optimistic value estimate satisfying~\eqref{eqn:optimistic_accurate_v}. The proof can be found in Appendix~\ref{app:proof-regret-opmd}.
\begin{proposition}[Regret guarantee of \DOPMD{}]
\label{prop:regret-opmd}
Suppose the optimistic value estimates in Algorithm~\ref{alg:opmd} achieve valid optimism and uniformly small error, i.e.
\begin{align}
\label{eqn:opmd-v-assumption}
V_{i,1}^{\pi_i\times \pi^{t}_{-i}}(s_1)\le \Vb_{i}^{(t),\pi_i\times \pi^{t}_{-i}} \le V_{i,1}^{\pi_i\times \pi^t_{-i}}(s_1) + \epsilon
\end{align}
for all $t\in[T]$, $i\in[m]$, and $\pi_i\in\Pi_i$. Then, Algorithm \ref{alg:opmd} with $\eta_i=\sqrt{\log\abs{\Pi_i}/(H^2T)}$ achieves with probability at least $1-\delta$ that
\begin{equation}
   \max_{i\in[m]}\max_{\pi_i\in \Pi_i}\sum_{t=1}^T \left[V^{\pi_i \times \Lambda^t_{-i}}_{i,1}(s_1) - V^{\Lambda^{t}}_{i,1}(s_1)\right] 
    \le \eps T + \cO\left(H\sqrt{T\log\Big(\sum_{i\in[m]} |\Pi_i|/\delta }\Big)\right).
\end{equation}
\end{proposition}
By~\eqref{eqn:optimistic_accurate_v} and Proposition~\ref{prop:regret-opmd}, we have that with probability at least $1-\delta$, the output policy $\wb{\Lambda}$ of Algorithm~\ref{alg:opmd} achieves
\begin{align*}
& \quad \CCEGap^\Pi(\wb{\Lambda}) = \max_{i\in[m]}\max_{\pi_i\in \Pi_i} \paren{V_{i,1}^{\pi_i\times\wb{\Lambda}_{-i}} - V_{i,1}^{\wb{\Lambda}}} = \frac{1}{T} \max_{i\in[m]}\max_{\pi_i\in \Pi_i}\sum_{t=1}^T \left[V^{\pi_i \times \Lambda^t_{-i}}_{i,1}(s_1) - V^{\Lambda^{t}}_{i,1}(s_1)\right]  \\
& \le \eps/2 + \cO\left(H\sqrt{\log\Big(\sum_{i\in[m]} |\Pi_i|/\delta \Big)  / T} \right) \le \eps,
\end{align*}
where the last inequality requires choosing
\begin{align}
\label{eqn:t_choice}
    T = \tO\paren{ \frac{H^2\log(\sum_{i\in[m]} |\Pi_i| )}{\eps^2} }.
\end{align}
Combining~\eqref{eqn:ki_choice} with~\eqref{eqn:t_choice}, the total number of episodes played is at most
\begin{align*}
    T \times \paren{\sum_{i\in[m]} K_i} = \tO\paren{ \frac{H^6\paren{\sum_{i\in[m]} d_i}\cdot \log^2(\sum_i \abs{\Pi_i}\abs{\cF_i}) }{\eps^4} }.
\end{align*}
This completes the proof of Theorem~\ref{thm:restricted-cce}.
\qed

\subsection{Proof of Proposition~\ref{prop:regret-opmd}}
\label{app:proof-regret-opmd}

Fix any player $i\in[m]$. We have
\begin{align*}
\Reg_T^i &\defeq \max_{\pi_i\in \Pi_i} \sum_{t=1}^T \left[V^{\pi_i \times \Lambda^t_{-i}}_{i,1}(s_1) - 
        V^{\Lambda^{t}}_{i,1}(s_1)\right] \\
&\le \underbrace{\max_{\pi_i\in \Pi_i}\sum_{t=1}^T \left[V^{\pi_i \times \pi^t_{-i}}_{i,1}(s_1) - 
        V^{\Lambda^{t}_{i} \times \pi^t_{-i}}_{i,1}(s_1)\right]}_{\rm I} + \cO\left(H\sqrt{T\log\Big(\sum_{i\in[m]} |\Pi_i|/\delta }\Big)\right)
\end{align*}
with probability at least $1-\delta$, where the inequality uses the fact that
\begin{align*}
    & \quad \max_{\Lambda_i\in \Delta(\Pi_i)} \abs{\sum_{t=1}^T \left[V^{\Lambda_i \times \pi^t_{-i}}_{i,1}(s_1) - V^{\Lambda_i \times \Lambda^t_{-i}}_{i,1}(s_1)\right]} \\
    & = \max_{\Lambda_i\in \Delta(\Pi_i)} \abs{\sum_{\pi_i\in\Pi_i} \Lambda_i(\pi_i) \sum_{t=1}^T \left[V^{\pi_i \times \pi^t_{-i}}_{i,1}(s_1) - V^{\pi_i \times \Lambda^t_{-i}}_{i,1}(s_1)\right]} \\
    & = \max_{\pi_i\in\Pi_i} \abs{\sum_{t=1}^T \left[V^{\pi_i \times \pi^t_{-i}}_{i,1}(s_1) - V^{\pi_i \times \Lambda^t_{-i}}_{i,1}(s_1)\right]} \le \cO(H\sqrt{T\log\Big(\sum_{i\in[m]} |\Pi_i|/\delta\Big)}),
\end{align*}
following by applying Azuma-Hoeffding's inequality for all $i\in[m]$ and all $\pi_i\in\Pi_i$ simultaneously.



Next, to bound term ${\rm I}$, we have
\begin{align*}
{\rm I} &= \max_{\pi_i\in \Pi_i}\left(\sum_{t=1}^T \left[V^{\pi_i \times \pi^t_{-i}}_{i,1}(s_1) - 
        V^{\Lambda^t_i\times \pi^t_{-i}}_{i,1}(s_1)\right]\right)\\
&= \underbrace{\max_{\pi_i\in\Pi_i} \sum_{t=1}^T \left[ \Vb_{i}^{(t),\pi_i\times \pi^{t}_{-i}} - \Vb_{i}^{(t),\Lambda^t_i\times \pi^{t}_{-i}} \right]}_{(a)} +
\underbrace{\max_{\pi_i\in\Pi_i}\sum_{t=1}^T \left[V^{\pi_i \times \pi^t_{-i}}_{i,1}(s_1) - \Vb_{i}^{(t),\pi_i\times \pi^{t}_{-i}} \right]  }_{(b)} \\
&\quad +\underbrace{\sum_{t=1}^T \left[\Vb_{i}^{(t),\Lambda^t_i\times \pi^{t}_{-i}} - V_{i,1}^{\Lambda^t_i \times \pi^t_{-i}}(s_1)\right]}_{(c)}.
\end{align*}
By~\eqref{eqn:opmd-v-assumption}, we have $(b)\le 0$ and $(c)\le \epsilon_1\cdot T$.
To bound $(a)$, note that by Algorithm~\ref{alg:opmd}, $\Lambda^t_i$ has the following equivalent Follow-The-Regularized-Leader (FTRL) form: $\Lambda^t_i(\pi_i)\propto_{\pi_i} \exp\left(\eta_i \sum_{\tau=1}^{t-1} \Vb_{i}^{(t),\pi_i\times \pi^{t}_{-i}}\right)$, where each $\Vb_{i}^{(t),\pi_i\times \pi^{t}_{-i}}\in[0,H]$. 
Therefore, by standard FTRL analysis~\citep[Section 6.6]{orabona2019modern},
\begin{align*}
\max_{\pi_i\in\Pi_i} \sum_{t=1}^T \left[ \Vb_{i}^{(t),\pi_i\times \pi^{t}_{-i}} - \Vb_{i}^{(t),\Lambda^t_i\times \pi^{t}_{-i}} \right] 
&\le \frac{\log|\Pi_i|}{\eta_i} + \frac{\eta_i}{2} H^2 T \le \cO\paren{H\sqrt{\log\abs{\Pi_i}\cdot T}},
\end{align*}
where in the last inequality we have picked $\eta_i=\sqrt{\log\abs{\Pi_i}/(H^2T)}$. This gives that ${\rm I}\le \eps T + \cO\paren{H\sqrt{\log\abs{\Pi_i}\cdot T}}$, which when plugged back into the regret bound yields that, with probability at least $1-\delta$, we have for all $i\in[m]$ simultaneously
\begin{align*}
    & \quad \Reg_T^i \le \eps T + \cO\paren{H\sqrt{\log\abs{\Pi_i}\cdot T}} + \cO\left(H\sqrt{T\log\Big(\sum_{i\in[m]} |\Pi_i|/\delta }\Big)\right) \\
    & \le \eps T + \cO\left(H\sqrt{T\log\Big(\sum_{i\in[m]} |\Pi_i|/\delta }\Big)\right).
\end{align*}
This proves the desired result.
\qed


\subsection{Proof of Proposition~\ref{prop:ftu}}
\label{app:proof-ftu}

We begin by providing the following lemma, which shows that the confidence sets at every iteration contain the true value function of any policy $\pi_i$, and achieves small estimation errors with respect to the visited state-actions. The proof relies on the $\Pi$-completeness assumption (Assumption~\ref{asp:completeQ}) and standard fast-rate concentration arguments for the square loss, and can be found in Appendix~\ref{app:proof-concentrate-Q}.
\begin{lemma}[Properties of $\cB^k$]
\label{lem:concentrate-Q}
    Under Assumption \ref{asp:completeQ}, there exists an absolute constant $c>0$ so that if we choose $\beta=cH^2\log(|\Pi_i||\cF_i|KH/\delta)$ in Algorithm \ref{alg:ftu}, then with probability at least $1-\delta$,
    \begin{enumerate}[label=(\alph*)]
        \item $(Q_i^{\pi_i,\pi_{-i}},\pi_i)\in\cB^k$ for all $(\pi_i,k)\in\Pi_i\times[K]$,
        \item $\sum_{t=1}^{k-1} \left[(f_h-\cT_{i,h}^{\pi_i\times\pi_{-i}}f_{h+1})(s_h^{t},a_{i,h}^{t})\right]^2 \le \cO(\beta)$ for all $(k,h)\in[K]\times[H]$ and $(f,\pi_i)\in\cB^k$,
        \item $\sum_{t=1}^{k-1} \E_{(s_h,a_{i,h})\sim\pi_i^t\times\pi_{-i}}\brac{ (f_h-\cT_{i,h}^{\pi_i\times\pi_{-i}}f_{h+1})(s_h,a_{i,h})^2 } \le \cO(\beta)$ for all $(k,h)\in[K]\times[H]$ and $(f,\pi_i)\in\cB^k$.
    \end{enumerate}
\end{lemma}

By Lemma~\ref{lem:concentrate-Q}(a), on the good event it ensures (with probability at least $1-\delta/2$) and by the definition of $\up{V}^{K,\pi_i\times\pi_{-i}}$ and $\low{V}^{K,\pi_i\times\pi_{-i}}$ in Algorithm~\ref{alg:ftu}, we immediately have $\low{V}^{K,\pi_i\times\pi_{-i}} \le V_{i,1}^{\pi_i, \pi_{-i}}(s_1) \le \up{V}^{K,\pi_i\times\pi_{-i}}$ for all $(\pi_i,k)\in\Pi_i\times[K]$, which proves part (a).

To prove part (b), for any $k\in[K]$, denote the optimistic and pessimistic Q estimates of the ``exploration policy'' $\pi_i^k$ by
$$
\up{f}^k = \arg\max_{f:(f,\pi_i^k)\in\cB^k} f_1(s_1,\pi_{i,1}^k(s_1))~~~\textrm{and}~~~\low{f}^k = \arg\min_{f:(f,\pi_i^k)\in\cB^k} f_1(s_1,\pi_{i,1}^k(s_1)),
$$
where we recall that $\pi_i^k$ is chosen to maximize the difference between the above two values over all $\pi_i\in\Pi_i$. This combined with the monotonicity of $\cB^k$ gives that, for any fixed $\pi_i\in\Pi_i$,
\begin{align*}
    & K\times \paren{ \up{V}^{K,\pi_i\times\pi_{-i}} - \low{V}^{K,\pi_i\times\pi_{-i}} } \\
    \le & \sum_{k=1}^K \paren{ \max_{f:(f, \pi_i)\in\cB^k} f_1(s_1, \pi_{i,1}(s_1)) - \min_{f:(f, \pi_i)\in\cB^k} f_1(s_1, \pi_{i,1}(s_1)) } \\
    \le & \sum_{k=1}^K \left(\fu^k_1(s_1,\pi_{i,1}^k(s_1)) - \fl^k_1(s_1,\pi_{i,1}^k(s_1))\right) \\
    = & \sum_{k=1}^K \left(\fu^k_1(s_1,\pi_{i,1}^k(s_1)) - V^{\pi_i^k\times\pi_{-i}}_{i,1}(s_1,\pi_{i,1}^k(s_1))\right) + \sum_{k=1}^K \left(  V^{\pi_i^k\times\pi_{-i}}_{i,1}(s_1,\pi_{i,1}^k(s_1))-\fl^k_1(s_1,\pi_{i,1}^k(s_1))\right).
\end{align*}
The above two terms can be bounded by the same arguments. WLOG, below we focus on the first term. 

Recall that the BE dimension assumption (Assumption~\ref{asp:bounded_eluder}) asserts that either the $\cD_{\Pi_i\times\pi_{-i}}$-type or the $\cD_\Delta$-type distributional Eluder dimension is bounded (cf. Definition~\ref{def:bellman-eluder}). We first consider the case for the $\cD_\Delta$-type distributional Eluder dimension, where we have for any $\eps>0$,
\begin{align*}
    d_i(\eps) \defeq \max_{h\in[H]} \dE\paren{\set{f_h - \cT_{i,h}^{\pi_{i}\times\pi_{-i}}f_{h+1} : (f, \pi_i)\in\cF\times\Pi_i}, \cD, \eps} \le d_i\log(1/\eps).
\end{align*}
In this case, we have
\begin{align}
    \label{eqn:perf_diff_be}
    \begin{aligned}
     & \quad \sum_{k=1}^K \left(\fu^k_1(s_1,\pi_{i,1}^k(s_1)) - V^{\pi_i^k\times\pi_{-i}} _1(s_1,\pi_{i,1}^k(s_1))\right)  \\
     & = \sum_{h=1}^H \sum_{k=1}^K \E_{\pi_i^k\times\pi_{-i}}\left[\fu^k_h(s_h,\pi_{i,h}^k(s_h)) - r_h - \fu^k_{h+1}(s_{h+1},\pi_{i,h+1}^k(s_{h+1})) \right] \\ 
     & \stackrel{(i)}{\le} \sum_{h=1}^H \sum_{k=1}^K \left[\left(\fu^k_h - \cT_{i,h}^{\pi_i^k\times\pi_{-i}}\fu^k_{h+1})(s_{h}^{k},a_{i,h}^{k}) \right)\right] + \cO\left( H\sqrt{K\log(H/\delta)}\right)\\
     & \stackrel{(ii)}{\le} \cO\left(H\sqrt{d_i(K^{-1/2})K \beta}\right)+ \cO\left( H\sqrt{K\log(H/\delta)}\right) \le \cO\left(H\sqrt{d_iK\log K\cdot  \beta}\right).
\end{aligned}
\end{align}
Above, (i) follows by Azuma-Hoeffding's inequality; (ii) follows by combining Lemma~\ref{lem:concentrate-Q}(b) applied on $(\fu^k, \pi_i^k)$ with an Eluder dimension argument~\citep[Lemma 41]{jin2021bellman}, which gives that for all $h\in[H]$,
\begin{align*}
    & \sum_{t=1}^{k-1} \left[(\fu^k_h-\cT_{i,h}^{\pi_i^k\times\pi_{-i}}\fu^k_{h+1})(s_h^{t},a_{i,h}^{t})\right]^2 \le \cO(\beta)~~~\textrm{for all}~k\in[K] \\
    \implies & \sum_{k=1}^K \left[\left(\fu^k_h - \cT_{i,h}^{\pi_i^k\times\pi_{-i}}\fu^k_{h+1})(s_{h}^{k},a_{i,h}^{k}) \right)\right] \le \cO\paren{ \sqrt{d_i(K^{-1/2}) K \beta} } \le \cO\paren{ \sqrt{d_i\log K\cdot K\beta} }.
\end{align*}
For the other case of the $\cD_{\Pi_i\times\pi_{-i}}$-type distributional-Eluder dimension, we conduct the same arguments up to the point before inequality (i) in~\eqref{eqn:perf_diff_be}, and apply the same Eluder dimension argument with respect to roll-in distributions $\sets{d^{\pi_i^k\times\pi_{-i}}_h}_{k\ge 1}$ combined with Lemma~\ref{lem:concentrate-Q}(c) to obtain the same bound as the $\cD_\Delta$ case.

Together with the same bound for the second term, we obtain
\begin{align*}
    K\times \paren{ \up{V}^{K,\pi_i\times\pi_{-i}} - \low{V}^{K,\pi_i\times\pi_{-i}} } \le \cO\left(H\sqrt{d_iK\log K\cdot \beta}\right).
\end{align*}
Dividing by $K$ on both sides proves the desired result.
\qed

\subsubsection{Proof of Lemma~\ref{lem:concentrate-Q}}
\label{app:proof-concentrate-Q}

The proof is similar to that of~\citet[Lemma 39(b) \& 40]{jin2021bellman}. Recall that we consider a fixed $\pi_{-i}$, and let us use $\pi=\pi_i\times\pi_{-i}$ for shorthand. Define random variable
\begin{align*}
    X_h^t(f, \pi_i) \defeq 2(f_h - \cT_{i,h}^{\pi}f_{h+1})(s_h^t, a_{i,h}^t) \times \brac{ r_{i,h}^t + f_{h+1}(s_{h+1}^t, \pi_{i,h+1}(s_{h+1}^{t})) - (\cT_{i,h}^{\pi}f_{h+1})(s_h^t, a_{i,h}^t) }
\end{align*}
for all $(f,\pi_i,t,h)\in\cF_i\times\Pi_i\times[K]\times[H]$. 

Consider the filtration $\sets{\cG_h^t}_{t\ge 1}$ that includes all historical observations up to $(s_h^t,a_{i,h}^t)$ within iteration $t$, but not $(r_{i,h}^t,s_{h+1}^t)$. Note that $X_h^t(f,\pi_i)$ is a martingale difference sequence with respect to $\sets{\cG_h^t}_{t\in[K]}$ (as the second term is mean-zero on $\cG_h^t$).
Further, we have $X_h^t(f, \pi_i)\le 2H^2$ almost surely as $f_h(\cdot,\cdot)\in[0,H-h+1]$ for all $h\in[H]$. Therefore, by Freedman's inequality (Lemma~\ref{lemma:freedman}) and a union bound, for any fixed $\lambda\le 1/(2H^2)$, we have with probability at least $1-\delta$ that
\begin{align}
\label{eqn:fast_square_freedman}
\begin{aligned}
    & \quad \sum_{t=1}^k X_h^t(f, \pi_i) \le 4\lambda H^2\sum_{t=1}^k \brac{(f_h - \cT_{i,h}^{\pi}f_{h+1})(s_h^t, a_{i,h}^t)}^2 + \frac{\log(\abs{\cF_i}\abs{\Pi_i}KH/\delta)}{\lambda} \\
    & = \frac{1}{2}\sum_{t=1}^k \brac{(f_h - \cT_{i,h}^{\pi}f_{h+1})(s_h^t, a_{i,h}^t)}^2 + 8H^2\log(\abs{\cF_i}\abs{\Pi_i}KH/\delta).
\end{aligned}
\end{align}
for all $(f,\pi_i,k,h)$ simultaneously, where in the second line we have picked $\lambda=1/(8H^2)$.

Let $\cD_h^k$ denote the dataset $\cD_h$ maintained in Algorithm~\ref{alg:ftu} before the start of the $k$-th iteration (i.e. used in forming $\cB^k$). To prove part (b), take any $(k,h)\in[K]\times[H]$ and $(f, \pi_i)\in\cB^k$. We have by definition of $\cB^k$ that
\begin{align*}
    & \quad \beta \ge \cL_h^{\cD_h^k}(f_h, f_{h+1}, \pi_i) - \min_{f_h'\in\cF_{i,h}} \cL_h^{\cD_h^k}(f_h', f_{h+1}, \pi_i) \\
    & \stackrel{(i)}{\ge} \cL_h^{\cD_h^k}(f_h, f_{h+1}, \pi_i) - \cL_h^{\cD_h^k}(\cT_{i,h}^\pi f_{h+1}, f_{h+1}, \pi_i) \\
    & = \sum_{t=1}^{k-1} \brac{ f_h(s_h^t, a_{i,h}^t) - r_{i,h}^t - f_{h+1}(s_{h+1}^t, \pi_{i,h+1}(s_{h+1}^t))}^2 \\
    & \qquad - \sum_{t=1}^{k-1} \brac{ (\cT_{i,h}^\pi f_{h+1})(s_h^t, a_{i,h}^t) - r_{i,h}^t - f_{h+1}(s_{h+1}^t, \pi_{i,h+1}(s_{h+1}^t))}^2 \\
    & = -\sum_{t=1}^{k-1} X_h^t(f, \pi_i) + \sum_{t=1}^{k-1} \brac{(f_h - \cT_{i,h}^{\pi}f_{h+1})(s_h^t, a_{i,h}^t)}^2 \\
    & \stackrel{(ii)}{\ge} -8H^2\log(\abs{\cF_i}\abs{\Pi_i}KH/\delta) + \frac{1}{2}\sum_{t=1}^{k-1} \brac{(f_h - \cT_{i,h}^{\pi}f_{h+1})(s_h^t, a_{i,h}^t)}^2.
\end{align*}
Above, (i) follows by $\Pi$-completeness (Assumption~\ref{asp:completeQ}), and (ii) follows by~\eqref{eqn:fast_square_freedman}. Therefore, choosing $\beta=8H^2\log(\abs{\cF_i}\abs{\Pi_i}KH/\delta)$ ensures that
\begin{align*}
    \sum_{t=1}^{k-1} \brac{(f_h - \cT_{i,h}^{\pi}f_{h+1})(s_h^t, a_{i,h}^t)}^2 \le 4\beta,
\end{align*}
which proves part (b).

To prove part (a), first note that $Q_i^\pi=Q_i^{\pi_i\times\pi_{-i}}\in\cF_i$, as we have $Q_{i,h}^\pi\in\cF_{i,h}$ for $h=H,\dots,1$ by Assumption~\ref{asp:completeQ} repeatedly. Therefore, fix any $(k,h)\in[K]\times[H]$ and $f_h'\in\cF_{i,h}$, and let $\wt{Q}\in\cF$ be defined as $\wt{Q}_h=f_h'$ and $\wt{Q}_{h'}=Q_{i,h'}^\pi$ for all $h'\ne h$. Similar as above, we have
\begin{align*}
    & \quad \cL_h^{\cD_h^k}(Q_{i,h}^\pi, Q_{i,h+1}^\pi, \pi_i) - \cL_h^{\cD_h^k}(f_h', Q_{i,h+1}^\pi, \pi_i) \\
    & = \cL_h^{\cD_h^k}(\cT_{i,h}^\pi \wt{Q}_{h}, \wt{Q}_{h+1}, \pi_i) - \cL_h^{\cD_h^k}(\wt{Q}_h, \wt{Q}_{h+1}, \pi_i) \\
    & = \sum_{t=1}^{k-1} X_h^t(\wt{Q}, \pi_i) - \sum_{t=1}^{k-1}\brac{ (f_h' - Q_{i,h}^\pi)(s_h^t, a_{i,h}^t)}^2 \\
    & \stackrel{(i)}{\le} 8H^2\log(\abs{\cF_i}\abs{\Pi_i}KH/\delta) - \frac{1}{2}\sum_{t=1}^{k-1}\brac{ (f_h' - Q_{i,h}^\pi)(s_h^t, a_{i,h}^t)}^2 \stackrel{(ii)}{\le} \beta,
\end{align*}
where (i) follows by~\eqref{eqn:fast_square_freedman} and (ii) follows by our choice of $\beta=8H^2\log(\abs{\cF_i}\abs{\Pi_i}KH/\delta)$. As this holds for any $f_h'\in\cF_{i,h}$, taking supremum over the left-hand side above gives that
\begin{align*}
    \cL_h^{\cD_h^k}(Q_{i,h}^\pi, Q_{i,h+1}^\pi, \pi_i) - \inf_{f_h'\in\cF_{i,h}} \cL_h^{\cD_h^k}(f_h', Q_{i,h+1}^\pi, \pi_i)  \le \beta.
\end{align*}
As this holds for all $h\in[H]$, by definition we have $(Q_{i,h}^\pi,\pi)\in\cB^k$ for all $k\in[K]$. This proves part (a).

Finally, part (c) can be proved by exactly the same arguments as part (b), except for redefining the filtration $\sets{\cG^t_h}_{t\ge 1}$ to include all historical observations \emph{before} episode $t$ starts, so that $(s_h^t, a_{i,h}^t)\sim d^{\pi_i^k\times\pi_{-i}}_h$ conditioned on $\cG^t_h$, and rescaling the tail probability $\delta\to\delta/2$ in both~\eqref{eqn:fast_square_freedman} and its analog with respect to the new filtration here.
\qed

\subsection{Details for Linear Quadratic Games}
\label{app:lqg}



Here we provide the details for the LQG example (Example~\ref{example:lqg}). 
Define the following feature map for all $i\in[m]$ (with $d_{\phi,i}\defeq d_S+d_{A,i}+1$):
\begin{align*}
\phi_i(s,a_i) = \begin{bmatrix}
s \\ a_i \\1
\end{bmatrix}\begin{bmatrix}
s^\top & a_i^\top &1
\end{bmatrix} \in \R^{d_{\phi,i}\times d_{\phi,i}}.
\end{align*}
We consider the following linear value class and linear policy class for all $i\in[m]$:
\begin{itemize}[topsep=0pt, itemsep=0pt]
    \item $\cF_{i,h}\defeq \sets{ f_{i,h}(s,a_i)=\<\phi_i(s,a_i),\theta_h\>:\theta_h\in\R^{d_{\phi,i}\times d_{\phi,i}}, \lfro{\theta_h}\le B_{\theta,h} }$.
    \item $\Pi_i\defeq \sets{ \pi_i=\sets{\pi_{i,h}(s) = M_{i,h}s}_{h\in[H]}: M_{i,h}\in\R^{d_{A,i}\times d_S}, \lfro{M_{i,h}}\le B_{M,h}}$.
\end{itemize}
Fixing any linear policy $\pi_{-i}\in\Pi_{-i}$ for the opponents, by the structure of the transition~\eqref{eqn:lqg_transition} and the reward, the MDP faced by player $i$ reduces to a Linear Quadratic Regulator (LQR), which we denote for simplicity of notation as
$$
\begin{cases}
    s_{h+1} = C_h s_h + D_h a_{i,h}+ z_h,\\
    r_{i,h}(s,a_i) = \langle J_{i,h}, \phi_i(s,a_i) \rangle.
\end{cases}
$$
The above $C_h,D_h,J_{i,h}$ can be computed from $A_h,\sets{B_{i,h}}_{i},\{K^i_{h}\}_{i}$, $\sets{K^i_{j,h}}_{i,j}$, and $\pi_{-i}$. It is straightforward to see that, with proper choice of $B_{\theta,h}=\cO({\rm poly}(d_{\phi,i},B_M)^{H-h+1})$ (the final sample complexity will only depend on its logarithm, by covering arguments), we have $\cT^{\pi_i\times\pi_{-i}}_{i,h} f_{h+1}\in \cF_{i,h}$ for any $f_{h+1}\in\cF_{i,h+1}$. This verifies Assumption~\ref{asp:completeQ}.

Further, observe that the function class
$$
\left\{f_h-\cT_h^{\pi_i\times\pi^{-i}} f_{h+1} \mid (f,\pi_i)\in\cF_i \times\Pi_i\right\}
$$
is a linear function class with a $d_{\phi,i}^2$-dimensional feature map $\phi_i(\cdot,\cdot)$. By standard Eluder dimension bounds for linear function classes, the $\cD_\Delta$-type BE dimension (Definition~\ref{def:bellman-eluder}) is bounded by $\tO(d_{\phi,i}^2)$, thus verifying Assumption~\ref{asp:bounded_eluder} with $d_i\defeq \cO(d_{\phi,i}^2)=\cO((d_S+d_{A,i})^2)$. Further by standard covering arguments, we can construct finite coverings of $\cF_i,\Pi_i$ both with log-cardinality $\tO({\rm poly}(H)\cdot d_{\phi,i}^2)$. Plugging these into Theorem~\ref{thm:restricted-cce}, we obtain that \DOPMD{} learns a $\Pi$-CCE for LQGs within
\begin{talign*}
\tO\paren{{\rm poly}(H, \sum_{i\in[m]} d_{\phi,i}) / \eps^4}
\end{talign*}
episodes of play.